\newif\ifpnas
\pnasfalse

\ifpnas
  \documentclass[9pt,twocolumn,twoside,lineno]{pnas-new}
  \templatetype{pnasresearcharticle} 
  \usepackage{stefan_tex}
  \usepackage{mathtools}
  \usepackage{amsmath}%
  \usepackage{amssymb}%
  \usepackage{amsthm}%
  \usepackage{ifthen}%
\else  
  \documentclass{article}%
  \usepackage{natbib}%
  \usepackage{caption}%
  \usepackage{geometry}
  \usepackage[colorlinks=true,linkcolor=blue,citecolor=blue]{hyperref}%
  \usepackage{textcomp}%
  \usepackage{lastpage}
  \usepackage{float}%
  \usepackage[dvipsnames]{xcolor}%
  \usepackage{mathtools}
  \usepackage{graphicx}%
  \usepackage{natbib}%
  \usepackage{amsmath}%
  \usepackage{amssymb}%
  \usepackage{mathtools}%
  \usepackage{booktabs}%
  \usepackage{epstopdf}%
  \usepackage{adjustbox}%
  \usepackage{array}%
  \usepackage{titling}%
  \usepackage{stefan_tex}
  \usepackage{amsthm}
  \usepackage{xargs}
  \usepackage{xcolor}
  \usepackage{ifthen}%
\fi

\ifpnas
  \setboolean{pnas}{true}
\else
  \setboolean{pnas}{false}
\fi
\newcommand{\appendixref}[2][]{\ifthenelse{\boolean{pnas}}{SI Appendix {#1}}{Appendix \ref{#2}}}

\ifpnas
  \author[a, 1]{Vitor Hadad}
  \author[a]{David Hirshberg} 
  \author[b]{Ruohan Zhan} 
  \author[a]{Stefan Wager} 
  \author[a]{Susan Athey} 
  \affil[a]{
  Stanford Graduate School of Business, Stanford University, Stanford, CA 94305}
  \affil[b]{Institute for Computational and Mathematical Engineering, Stanford University, Stanford, CA 94305}
  \leadauthor{Hadad} 
  \authorcontributions{V.H., D.A.H., R.Z., S.W. and S.A. designed research, performed research, contributed analytic tools, analyzed simulated data, and wrote the paper.}
  \authordeclaration{The authors declare no competing interests.}
  \correspondingauthor{\textsuperscript{1}To whom correspondence should be addressed. E-mail: vitorh@stanford.edu}
  \significancestatement{
    Randomized control trials are central to the scientific process, but they can be costly. For example, a clinical trial may assign patients to treatments that are detrimental to them. Adaptive experimental designs, such as multi-armed bandit algorithms, reduce costs by increasing the probability of assigning promising treatments over the course of the experiment. However, because observations collected by these methods are dependent and their distribution is nonstationary, statistical inference can be challenging. 
    We propose a treatment effect estimator that has an asymptotically unbiased and normal test statistic
    under straightforward, relatively weak conditions on the adaptive design. This estimator generalizes for a variety of parameters of interest.
  }
  \keywords{Adaptive experimentation $|$ Multi-armed bandits $|$ Policy Evaluation $|$ Central limit theorem $|$ Frequentist inference } 
\else
  \author{
  Vitor Hadad \\ \texttt{vitorh@stanford.edu}
  \and David A. Hirshberg \\ \texttt{dhirshbe@stanford.edu}
  \and Ruohan Zhan \\ \texttt{rhzhan@stanford.edu}
  \and Stefan Wager \\ \texttt{swager@stanford.edu}
  \and Susan Athey \\ \texttt{athey@stanford.edu}}
  \date{Stanford University}
\fi

\graphicspath{{./figures/}}

\theoremstyle{plain}
\newtheorem{prop}{Proposition}

\newtheorem{lemm}[prop]{Lemma}
\newtheorem{theo}[prop]{Theorem}
\theoremstyle{definition}

\newtheorem{assu}{Assumption}
\theoremstyle{remark}

\newcommand{\Wef}[2]{ \frac{ \ind{ W_{#1} = #2 } }{ e_t(#2)} }
\newcommand{\todist}{\xrightarrow[]{d}}
\newcommand{\toprob}{\xrightarrow[]{p}}
\newcommand{\plim}{\operatorname{plim}}

\newcommand{\hm}{\hat{m}}

\newcommand{\E}{\mathbb{E}}
\DeclarePairedDelimiter\norm{\lVert}{\rVert}
\DeclareMathOperator{\var}{Var}
\title{Confidence Intervals for Policy Evaluation in Adaptive Experiments}

\ifpnas
  \renewcommand{\citet}{\cite}
\fi

\begin{document}
\maketitle

\begin{abstract}
  Adaptive experimental designs can dramatically improve efficiency in randomized trials. But with adaptively collected data, common estimators based on sample means and inverse propensity-weighted means can be biased or heavy-tailed. This poses statistical challenges, in particular when the experimenter would like to test hypotheses about parameters that were not targeted by the data-collection mechanism.
In this paper, we present a class of test statistics that can handle these challenges. Our approach is to adaptively reweight the terms of an augmented inverse propensity weighting estimator to control the contribution of each term to the estimator's variance. This scheme reduces overall variance and yields an asymptotically normal test statistic.
We validate the accuracy of the resulting estimates and their confidence intervals in numerical experiments and show our methods compare favorably to existing alternatives in terms of mean squared error, coverage, and confidence interval size.
\end{abstract}

\ifpnas
  \dates{This manuscript was compiled on \today}
  \doi{\url{www.pnas.org/cgi/doi/10.1073/pnas.XXXXXXXXXX}}
  \thispagestyle{firststyle}
  \ifthenelse{\boolean{shortarticle}}{\ifthenelse{\boolean{singlecolumn}}{\abscontentformatted}{\abscontent}}{}
\fi

\ifpnas
  \dropcap{A}adaptive
\else
  \section{Introduction}
  Adaptive
\fi
experimental designs can dramatically improve efficiency for particular objectives: maximizing welfare during the experiment \citep{lai1985asymptotically,auer2002nonstochastic} or after it \citep{bubeck2011pure, kasy2019adaptive}; quickly identifying the best treatment arm \citep{mannor2004sample,russo2020simple}; maximizing the power of a particular hypothesis \citep{robbins1952some, hu2006theory, van2008construction}; and so on. To achieve these efficiency gains, we adaptively choose assignments to 
resolve uncertainty about some aspects of the data generating process at the expense of learning little about others.
For example, welfare-maximizing designs tend to focus on differentiating optimal and near-optimal treatments, 
collecting relatively little data about suboptimal ones.

But once the experiment is over, we are often interested in using the adaptively-collected data to answer a variety of questions, not all of them necessarily targeted by the design. For example, a company experimenting with many types of web ads may use a bandit algorithm to maximize click-through rates during an experiment, but still want to quantify the effectiveness of each ad. 
At this stage, fundamental tensions between the experiment objective and statistical inference become apparent: extreme undersampling or non-convergence of the assignment probabilities make re-using this data challenging.

In this paper, we propose a method for constructing frequentist confidence intervals based on approximate normality,
even when challenges of adaptivity are severe, provided that the treatment assignment probabilities are known and satisfy certain conditions. To get a better sense of the challenges we face, we'll first consider 
an example in which traditional approaches to statistical testing fail. Suppose we run a two-stage, two-arm trial as follows. For the first $T/2$ time periods, we randomize assignments with probability 50\% for each arm. After $T/2$ time periods, we identify the arm with the higher sample mean, and for the next $T/2$ time periods we allocate treatment to the seemingly better arm 90\% of the time. Then, one estimator of the expected value $Q(w)$ for each arm $w \in \{1, \, 2\}$ is the sample mean at the end of the experiment,
\begin{equation}
\label{eq:avg}
\hQ^{\text{AVG}}(w) = \frac{1}{T_w}\sum_{\substack{t \le T \\ W_t = w}} Y_t, \quad  T_w := \sum_{\substack{t \le T \\ W_t = w}} 1,
\end{equation}
where $W_t$ denotes the arm pulled in the $t$-th time period and $Y_t$ denotes the observed outcome. Both arms have the same outcome distribution: $Y_t \cond W_t = w \sim \nn\p{0, \, 1}$ for all values of $t$ and $w$.

\begin{figure*}[t]
\begin{center}
\includegraphics[width=\textwidth]{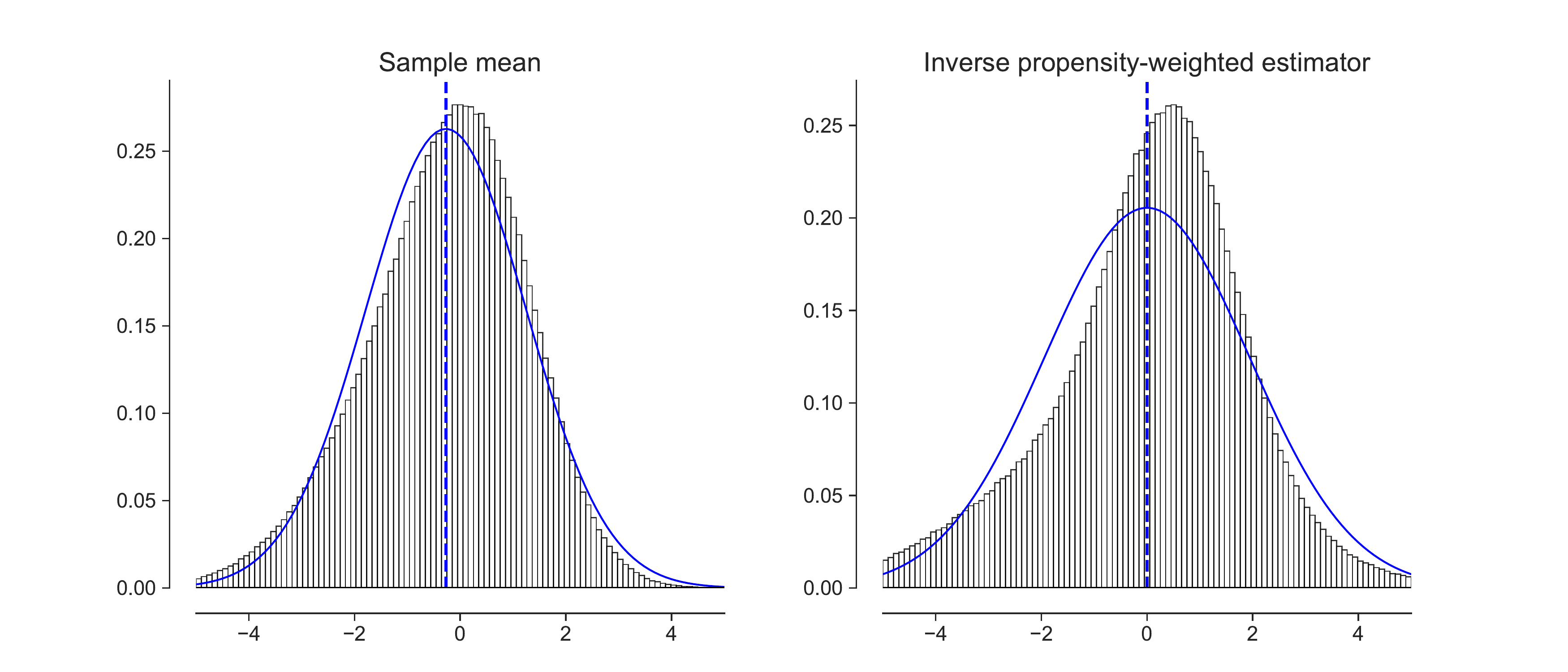} 
\caption{Distribution of the estimates \smash{$\hQ^{\text{AVG}}(1)$} and \smash{$\hQ^{\text{IPW}}(1)$} described in the introduction.
The plots depict the distribution of the estimators for $T = 10^6$, scaled by a factor $\sqrt{T}$ for visualization.
The distributions are overlaid with the normal curve that matches the first two moments of the distribution, along with a dashed
line that denotes the mean. All numbers are aggregated over 1,000,000 replications.}
\label{fig:example}
\end{center}
\end{figure*}

This example is relatively benign in that adaptivity is minimal. Yet, as the left panel of Figure~\ref{fig:example} shows, the estimate of the value of the first arm \smash{$\hQ^{\text{AVG}}(1)$} is biased downward. This is a well-known phenomenon; see e.g. \cite{nie2018adaptively, xu2013estimation, bowden2017unbiased, shin2019bias, shin2020conditional, deshpande2017accurate, deshpande2019online}. The downward bias occurs because arms in which we observe random upward fluctuations initially will be sampled more, while arms in which we observe random downward fluctuations initially will be sampled less. The upward fluctuations are corrected as estimates of arms that are sampled more regress to their mean, while the downward ones may not be corrected because of the reduction in sampling. Here we only show estimates for the first arm so there are no selection bias effects; the bias is a direct consequence of the adaptive data collection.

One often-discussed fix to this particular bias problem is to use the inverse-probability weighting estimator,
$\hQ^{\text{IPW}}(w) =T^{-1} \sum_{t = 1}^T \ind{W_t = w} Y_t / e_t(w)$ where $e_t(w)$ is the probability with
which our adaptive experiment drew arm $w$ in step $t$. This compensates for the outsize influence
of early downward fluctuations that reduce the probability of an arm being assigned by up-weighting later observations
within that arm when we see them. As seen in the right panel of Figure \ref{fig:example},
inverse-probability weighting fixes the bias problem but results in a non-normal asymptotic distribution. In fact, it can exacerbate the problem of inference: when the probability of assignment to the arm of interest tends to zero, the inverse probability weights increase, which in turn causes the tails of the distribution to become heavier. 

If adaptivity essentially vanishes over the experiment, then we are justified in using naive estimators that ignore adaptivity. 
In particular, if assignment probabilities quickly converge to 
constants, then in long-running experiments we can often treat the data as if treatments had always been assigned according to these limiting probabilities; see e.g., \cite{melfi2000estimation} or \cite{van2008construction}. It can be argued that most adaptive designs would eventually converge in this sense if run forever.
However, how quickly they converge, and therefore the number of samples we'd need for adaptivity to be ignorable, 
can vary considerably with the data-generating process. For example, we know that so long as some arm is best,
an $\epsilon$-greedy $K$-armed bandit algorithm will eventually assign to the best arm with probability $1-\epsilon + \epsilon/K$ and the others with probability $\epsilon/K$;
however, this happens only after the best arm is identified, which depends strongly on the unknown spacing between the arm values $Q(1), \ldots, Q(K)$. Moreover, in practice adaptive experiments are often used precisely when there is limited budget for experimentation, therefore substantial data collection after convergence is rare. As a result, if we try to exploit this convergence by using estimators that are only valid in convergent designs,
we get brittle estimators. If we want an estimator that is reliable, we must use one that is valid regardless of whether the assignment probabilities converge.

In this work, we propose a test statistic that is asymptotically unbiased and normal even when assignment probabilities converge to zero or do not converge at all.\footnote{In \appendixref[A.5]{sec:intro_example_revisited}, we revisit the example shown in the introduction and demonstrate how our method leads to an asymptotically normal test statistic for the arm value.} 
We believe this approach to be of practical interest because normal
confidence intervals are widely used in several fields including, e.g., medicine and economics. 
Moreover, though we focus on estimating the value of a pre-specified policy,
our estimates can also be used as input to procedures for testing 
adaptive hypotheses, which have as their starting point a vector of normal estimates~\citep[e.g.,][]{andrews2019inference}.

Other approaches to inference with adaptively collected data are available. One line of research eschews asymptotic normality in favor of developing finite-sample bounds using martingale concentration inequalities \citep[e.g.,][and references therein]{howard2021uniform, robbins1970statistical}. \cite{nie2018adaptively} considers approaches to debiasing value estimates using ideas from conditional inference \citep{fithian2014optimal}. 
And some avoid frequentist arguments altogether, preferring a purely Bayesian approach, although this can produce estimates that have poor frequentist properties \citep{dawid1994selection}. Section~\ref{sec:lit} further reviews papers on policy evaluation.

\section{Policy Evaluation with Adaptively Collected Data}
\label{sec:eval}

We start by establishing some definitions. Each observation in our data is represented by a tuple $(W_t, Y_t)$.
The random variables $W_t \in \mathcal{W}$ are called the arms, treatments or interventions. Arms are categorical. The reward or outcome $Y_t$ represents the individual's response to the treatment. The set of observations up to a certain time $H^T := \{(W_s, Y_s) \}_{s=1}^T$ is called a history.
The treatment assignment probabilities $e_t(w) := \mathbb{P}[W_t = w \cond H^{t-1}]$, also called propensity scores, are
time-varying and decided via some known algorithm, as it is the case with many popular bandit algorithms such as Thompson sampling \citep{thompson1933likelihood, russo2018tutorial}. 

We define causal effects using potential outcome notation \citep{imbens2015causal}. We denote by $Y_t(w)$ the random variable representing the outcome that would be observed if individual $t$ were assigned to a treatment $w$. In any given experiment, this individual can be assigned only one treatment $W_t$
from a set of options $\mathcal{W}$, so we observe only one realized outcome $Y_t = Y_t(W_t)$. We focus on the ``stationary'' setting where individuals, represented by a vector of potential outcomes $(Y_t(w))_{w \in \mathcal{W}}$, are independent and identically distributed. However, the observed outcomes $Y_t$
are in general neither independent nor identically distributed, because the distribution of the treatment assignment $W_t$ depends on the history of outcomes up to time $t$.

Given this setup, we are concerned with the problem of estimating and testing pre-specified hypotheses about the value  of an arm, denoted by $Q(w) := \EE{Y_t(w)}$, as well as differences between two such values, denoted by $\Delta(w, \, w') := \EE{Y_t(w)} - \EE{Y_t(w')}$. We would like to do that even in data-poor situations in which the data-collection mechanism did not target these estimands. 

We will provide consistent and asymptotically normal test statistics for $Q(w)$ and $\Delta(w, \,w')$. This is done in three steps. First, we start with a class of \emph{scoring rules}, which are transformations of the observed outcomes that can be used for unbiased arm evaluation, but whose sampling distribution can be non-normal and heavy-tailed due to adaptivity. Second, we average these objects with carefully chosen data-adaptive weights, obtaining a new estimator with controlled variance at the cost of some finite-sample small bias. Finally, by dividing these estimators by their standard error we obtain a test statistic that has a centered and standard normal limiting distribution.

\subsection{Unbiased scoring rules}
\label{sec:scoring}

A first step in developing methods for inference with adaptive data is to account for sampling bias.
The following construction provides a generic way of doing so. We say that \smash{$\hGamma_t(w)$}
is an unbiased scoring rule for $Q(w)$ if for all $w \in \mathcal{W}$ and $t = 1, \, ..., \, T$,
\begin{equation}
\label{eq:unb}
\EE{\hGamma_t(w) \cond H^{t-1}} = Q(w).
\end{equation}
Given this definition, we can readily verify that
a simple average of such a scoring rule,
\begin{equation}
\label{eq:unb_est}
\hQ_T(w) = \frac{1}{T} \sum_{i = 1}^T \hGamma_t(w),
\end{equation}
is unbiased for $Q(w)$ even though the $\hGamma_t(w)$ are correlated over time, as the next proposition shows.

\begin{prop}
\label{prop:unb}
Let $\cb{Y_t(w)}_{w \in \mathcal{W}}$ be an independent and identically distributed sequence of
potential outcomes for $t = 1, \, \ldots, \, T$, and let $H^t$  denote the observation
history up to time $t$ as described above. Then, any estimator of the form \eqref{eq:unb_est}
based on an unbiased scoring rule \eqref{eq:unb} satisfies \smash{$\mathbb{E}[\hQ_T(w)] = Q(w)$}.
\end{prop}

One can easily verify Proposition \ref{prop:unb} by applying the law of iterated expectations and \eqref{eq:unb},
\begin{align*}
\EE{\hQ_T(w)} = \EE{\frac{1}{T} \sum_{t = 1}^T \EE{\hGamma_t(w) \cond H^{t-1}}} = Q(w).
\end{align*}
The key fact underlying this result is that the normalization factor $1/T$ used in \eqref{eq:unb_est}
is deterministic and so cannot be correlated with stochastic fluctuations in the \smash{$\hGamma_t(w)$}.
In particular, we note that the basic averaging estimator \eqref{eq:avg} is not of the form \eqref{eq:unb_est} and
instead has a random denominator $T_{w}$---and is thus not covered by Proposition \ref{prop:unb}.

Given Proposition \ref{prop:unb}, we can readily construct several unbiased estimators for $Q(w)$. One
straightforward option is to use an inverse propensity score weighted (IPW) estimator:
\begin{equation}
  \label{eq:ipw}
  \hQ_T^{IPW}(w) := \frac{1}{T}\sum_{t=1}^T\hGamma_t^{IPW}(w), \ \ \ \hGamma_t^{IPW}(w) := \Wef{t}{w} Y_t.
\end{equation}
This estimator is simple to implement, and one can directly check that the condition \eqref{eq:unb} holds because, by
construction, $\PP{W_t = w \cond H^{t-1}, \, Y_t(w)} = \PP{W_t = w \cond H^{t-1}} = e(w;H^{t-1})$.

The augmented inverse propensity weighted (AIPW) estimator generalizes this by incorporating regression adjustment \citep{robins1994estimation}:
\begin{equation}
  \label{eq:aipw}
  \begin{split}
  &\hQ_T^{AIPW}(w) := \frac{1}{T}\sum_{t=1}^T \hGamma_t^{AIPW}(w), \\
  &\hGamma_t^{AIPW}(w) := \Wef{t}{w} Y_t + \left( 1 - \Wef{t}{w} \right)\hat{m}_t(w).
  \end{split}
\end{equation}
The symbol $\hat{m}_t(w)$ denotes an estimator of the conditional mean function $m(w) = \E[Y_t(w)]$ based on the history $H^{t-1}$, but it need not be a good one---it could be biased or even inconsistent. The second term of $\hGamma_t^{AIPW}(w)$ acts as a control variate:
adding it preserves unbiasedness but can reduce variance, as it has mean zero conditional on $H^{t-1}$ and, 
if $\hat{m}_t(w)$ is a reasonable estimator of $m(w)$, is negatively correlated with the first term. When 
$\hat{m}_t(w)$ is identically zero, the AIPW estimator reduces to the IPW estimator.

\subsection{Asymptotically normal test statistics}

The estimators discussed in the previous section are unbiased by construction, but in general they are not guaranteed to have an asymptotically normal sampling distribution. The reason for this failure of normality was illustrated in the example in the introduction for the IPW estimator. When, purely by chance, an arm has a higher sample mean in the first half of the experiment, it is sampled often in the second half and the IPW estimator concentrates tightly. When the opposite happens, the arm is sampled less frequently and the IPW estimator is more spread out. What we see in Figure~\ref{fig:example} is a heavy-tailed distribution corresponding to the mixture of the two behaviors. Qualitatively, what we need for normality is for the variability of the estimator to be deterministic. Formally, what is required is that the sum of conditional variances of each term in the sequence converges in ratio to the unconditional variance of the estimator \citep[see~e.g.,][Theorem 3.5]{hall2014martingale}. Simple averages of unbiased scoring rules \ref{prop:unb} fail to satisfy this because, as we'll elaborate below, the conditional variances the terms in \eqref{eq:aipw} depend primarily on the behavior of the inverse assignment probabilities $1/e_t(w)$, which may diverge to infinity or fail to converge.

To address this difficulty, we consider a generalization of the AIPW estimator \eqref{eq:aipw} that non-uniformly
averages the unbiased scores $\hGamma^{AIPW}$ using a sequence of \emph{evaluation weights}
$h_t(w)$. The resulting estimator is the \emph{adaptively-weighted AIPW estimator}:
\begin{align}
  \label{eq:aw}
  \hQ^{h}_T(w) = \frac{\sum_{t=1}^T h_t(w) \hGamma_t^{AIPW}(w)}{\sum_{t=1}^T h_t(w)}.
\end{align}

Evaluation weights $h_t(w)$ provide an additional degree of flexibility in controlling the variance and tails of the sampling distribution. When chosen appropriately, these weights compensate for erratic trajectories of the assignment probabilities $e_t(w)$, stabilizing the variance of the estimator. With such weights, the adaptively-weighted AIPW estimator \eqref{eq:aw}, when normalized by an estimate of its standard deviation, has a centered and normal asymptotic distribution. Similar ``self-normalization'' schemes are often key to martingale central limit theorems \citep[see e.g.,][]{pena2008self}.

 Throughout, we will use evaluation weights $h_t(w)$ that are a function of the history $H^{t-1}$; we will call such functions \emph{history-adapted}. Note that if we used weights with sum equal to one, we would have a generalization of the unbiased scoring property \eqref{eq:unb_est}, $\E[h_t(w) \hGamma_t(w) \mid H^{t-1}] = h_t(w)Q(w)$, so the adaptively-weighted estimator \eqref{eq:aw} would be unbiased. In Section \ref{sec:weights} we will discuss weight heuristics that do not sum to one but that empirically seem to reduce variance and mean-squared error relative to alternatives. In that case, the estimator \eqref{eq:aw} will have some bias due to the random denominator $\smash{\sum_{t=1}^T h_t(w)}$. However, for the appropriate choices of evaluation weights $h_t(w)$ this bias disappears asymptotically.
 
The main conditions required by our weighting scheme are stated below. Assumption \ref{assu:inf} requires that the effective sample size after adaptive weighting, that is, the ratio \smash{$(\sum_{t=1}^T \E[ \alpha_t | H^{t-1} ])^2 / \E[\sum_{t=1}^T \alpha_t^2]$} where $\alpha_t := h_t(w) \ind{W_t=w}/e_t(w)$, goes to infinity. This implies that the estimator converges. Assumption \ref{assu:variance_convergence} is the more subtle condition that unbiased estimators such as \eqref{eq:unb_est} (i.e., estimators with \smash{$h_t(w) \equiv 1$}) often fail to satisfy. Assumption \ref{assu:lyapunov} is a Lyapunov-type regularity condition on the weights controlling higher moments of the distribution.

\begin{assu}[Infinite sampling]
\label{assu:inf}
The weights used in \eqref{eq:aw} satisfy
\begin{equation}
\label{eq:infinite_sampling}
\p{\sum_{t=1}^T h_t(w)}^2 \,\bigg/\,
      \EE{  \sum_{t=1}^T \frac{h^2_t(w)}{e_t(w)} }
      \xrightarrow[T \to \infty]{p} \infty.
\end{equation}
\end{assu}

\begin{assu}[Variance convergence]
\label{assu:variance_convergence}
The weights used in \eqref{eq:aw} satisfy, for some $p > 1$,
\begin{equation}
\label{eq:variance_convergence}
 \sum_{t=1}^T \frac{h^2_t(w)}{e_t(w) }  \,\bigg/\,
    \EE{  \sum_{t=1}^T \frac{h^2_t(w)}{e_t(w)} }
    \xrightarrow[T \to \infty]{L_p} 1.
\end{equation}
\end{assu}

\begin{assu}[Bounded moments]
\label{assu:lyapunov}
The weights used in \eqref{eq:aw} satisfy, for some $\delta > 0$,
\begin{equation}
\label{eq:lyapunov}
 \sum_{t=1}^T \frac{h^{2 + \delta}(w)}{e^{1 + \delta}(w) }  \,\bigg/\,
    \EE{  \sum_{t=1}^T \frac{h^2_t(w)}{e_t(w)} }^{1 + \delta/2}
    \xrightarrow[T \to \infty]{p} 0.
\end{equation}
\end{assu}

\begin{theo}
\label{theo:arm_value_clt}
Suppose that we observe arms $W_t$ and rewards $Y_t=Y_t(W_t)$, and that the underlying potential
outcomes $(Y_t(w))_{w \in \mathcal{W}}$ are independent and identically distributed with nonzero variance, and satisfy $\E|Y_{t}(w)|^{2+\delta} < \infty$ for some $\delta > 0$ and all $w$.
Suppose that the assignment probabilities $e_t(w)$ are strictly positive and 
let $\hat{m}_t(w)$ be any history-adapted estimator of $Q(w)$ that is bounded and that converges almost-surely
to some constant $m_{\infty}(w)$. Let $h_t(w)$ be non-negative history-adapted weights satisfying Assumptions \ref{assu:inf}, \ref{assu:variance_convergence} and \ref{assu:lyapunov}. Suppose that either $\hat{m}_t(w)$ is consistent or $e_t(w)$ has a limit 
$e_{\infty}(w) \in [0, \, 1]$, i.e., either
\begin{equation}
     \label{eq:e_mu_alternative}
      \hat{m}_t(w) \xrightarrow[t \to \infty]{a.s.} Q(w)  \quad \text{or}
      \quad e_t(w)  \xrightarrow[t \to \infty]{a.s.} e_{\infty}(w)
\end{equation}
Then, for any arm $w \in \mathcal{W}$, the adaptively-weighted estimator \eqref{eq:aw} is consistent for the arm value $Q(w)$, and the following studentized statistic is asymptotically normal:
\begin{equation}
  \begin{aligned}
    \label{eq:clt}
    &\frac{\hQ_T^{h}(w) - Q(w)}{\hV_T^{h}(w)^{\frac{1}{2}}} \todist \mathcal{N}(0, 1),
    \ \ \ \text{where} \\
    &\hV_T^{h}(w) := \frac{\sum_{t=1}^T h^2_t(w) \left( \hGamma_t(w) - \hQ_T(w) \right)^2}{\left( \sum_{t=1}^T h_t(w) \right)^2}.
  \end{aligned}
\end{equation}
\end{theo}

As Theorem \ref{theo:arm_value_clt} suggests, the asymptotic behavior of our estimator is largely determined by the behavior of the propensity scores $e_t(w)$ and evaluation weights $h_t(w)$. If the former's behavior is problematic, the latter can correct for that. For instance, the bounded moments condition \eqref{eq:lyapunov} implies that when $e_t(w)$ decays very fast, the evaluation weights must also decay at an appropriate rate so that the variance of the estimator does not explode. However, there are limits to what this approach can correct. For example, aggressive bandit procedures may assign some arms only finitely many times, and in that case it is impossible to estimate their values consistently. This scenario is ruled out by our infinite sampling condition \eqref{eq:infinite_sampling}, which would not be satisfied.

To build more intuition for the variance convergence condition \eqref{eq:variance_convergence}, pretend for the moment that the evaluation weights $h_t(w)$ sum to one and that 
$\hm_t(w)$ is consistent. Under these conditions, the variance of each AIPW score \smash{$\hGamma_t(w)$} conditional on the past can be shown to be \smash{$\var[Y_1(w)]/e_t(w)$} plus asymptotically negligible terms, so the sum of conditional variances is asymptotically equivalent to \smash{$\var [Y_1(w)]\smash{\sum_{t=1}^T h_t^2(w)/e_t(w)}$}. The variance convergence condition \eqref{eq:variance_convergence} says that this sum of conditional variances converges to its mean, which will be the unconditional variance of the estimator. If these simplifying assumptions really held, this argument would almost suffice to establish asymptotic normality, since variance convergence conditions like this are nearly sufficient for martingale central limit theorems~\citep[see~e.g.,][Theorem 3.5]{hall2014martingale}. 

When we use history-dependent weights $h_t(w)$ that do not sum to one, the normalized weights $\tilde h_t := \smash{h_t(w)/\sum_{s=1}^{T} h_s(w)}$ that scale our terms in \eqref{eq:aw} are not a function of past data alone. 
However, the argument above provides valuable intuition in our proof, and $\var[Y_1(w)]\sum_{t=1}^{T} \tilde h_t^2 /e_t(w)$ can be thought of as a reasonable proxy for the estimator's variance. Minimizing this variance proxy suggests the use of weights $h_t(w) \propto e_t(w)$. This heuristic, in combination with appropriate constraints, motivates an empirically successful weighting scheme that we'll further discuss in Section \ref{sec:weights}. 

If the weights $h_t(w)$ are not constructed appropriately, then $h^2_t(w)/e_t(w)$ may behave erratically and the variance convergence condition will fail to hold. This can happen, for example, in a bandit experiment in which there are multiple optimal arms and uniform weights $h_t(w) \equiv 1$ are used. In this setting, the bandit algorithm may spuriously choose one arm at random early on and assign the vast majority of observations to it, so that no run of the experiment will look like an ``average run'' and the ratio in \eqref{eq:variance_convergence} will not converge. Or it may switch between arms infinitely often, in which case the ratio will converge only if it switches quickly enough that the random order is ``forgotten'' in the average. This issue persists even when assignment probabilities are guaranteed to stay above a strictly positive lower bound. In the next section, we will construct weights so that ``variance convergence'' \eqref{eq:variance_convergence} is guaranteed to be satisfied not only asymptotically, but for all sample sizes $T$.

The simplifying assumption that $\hm_t(w)$ is consistent is not necessary: as stated in \eqref{eq:e_mu_alternative},
a variant of the argument above still goes through if the propensity score $e_t(w)$ converges.
In a non-adaptive experiment, the AIPW estimator will have optimal asymptotic variance if $\hm_t(w)$ is consistent;
if it is not, the excess asymptotic variance is a function of $\lim_{t \to \infty} e_t(w)$ 
and $\lim_{t \to \infty} \hm_t(w) - Q(w)$.

\subsection{Constructing adaptive weights}
\label{sec:weights}

A natural question is how to choose evaluation weights $h_t(w)$ 
for which the adaptively-weighted AIPW estimator \eqref{eq:aw} is asymptotically unbiased and normal with low variance, i.e., 
for which we get narrow and approximately valid confidence intervals.
To this end, we'll start by focusing on the variance convergence condition \eqref{eq:variance_convergence}. Once we have a recipe
for building weights that satisfy it, we'll consider how to satisfy the other conditions of
Theorem \ref{theo:arm_value_clt} and how to optimize for power.

The variance convergence condition \eqref{eq:variance_convergence} requires the sum $\sum_{t = 1}^T h_t^2/e_t$ to concentrate
around its expectation. A direct way to ensure this is to make the sum deterministic.
To do this, we choose weights via a recursively defined ``stick-breaking'' procedure,\footnote{For notational efficiency,
whenever it does not lead to confusion we will drop the
the dependence on arm and write, e.g., \smash{$\hQ_T$} simply to mean our adaptively-weighted
estimator \eqref{eq:aw}, $h_t$ for evaluation weights, $e_t$ for assignment probabilities, and so on.}
\begin{align}
  \label{eq:variance_stabilizing}
  \frac{h^2_t}{e_t} = \left( 1 - \sum_{s=1}^{t-1} \frac{h_s^2}{e_s} \right) \lambda_t,
\end{align}
where $\lambda_t$ satisfies $0 \leq \lambda_t < 1$ for all $1 \leq t \leq T-1$, and $\lambda_T = 1$. Because $\lambda_T = 1$, the definition above for $t=T$ directly implies 
that $\sum_{t = 1}^T h_t^2/e_t = 1$. This ensures that the variance convergence condition \eqref{eq:variance_convergence} is satisfied, so we call these \emph{variance-stabilizing} weights.

We call the function $\lambda_t$ an \emph{allocation rate} because it qualitatively captures the fraction of our remaining variance that we allocate to the upcoming observation. 
This is a useful class to consider because the analyst has substantial freedom in constructing weights by 
choosing different allocation rates $\lambda_t$, while ensuring that the resulting evaluation weights automatically satisfy the variance convergence assumption, and satisfy other assumption of Theorem \ref{theo:arm_value_clt} with some generality.
\begin{theo}
\label{theo:variance_stabilizing_clt}
In the setting of Theorem \ref{theo:arm_value_clt}, suppose that the treatment propensities satisfy
\begin{equation}
\label{eq:prop_LB}
e_t(w) \geq Ct^{-\alpha}
\end{equation}
for $\alpha \in [0, 1)$ and any positive constant $C$. Then the variance-stabilizing weights \eqref{eq:variance_stabilizing} 
defined by a history-adapted allocation rate $\lambda_t(w)$ are history-adapted and satisfy Assumptions \ref{assu:inf}, \ref{assu:variance_convergence}, and \ref{assu:lyapunov} 
if $\lambda_t(w) < 1$ for $t < T$, $\lambda_T(w) = 1$ and, for a finite positive constant $C'$,
\begin{align}
\label{eq:allocation_rate_bounds}
\frac{1}{T - t + 1} \leq \lambda_t(w) \leq C' \, \frac{e_t(w)}{t^{-\alpha} + T^{1-\alpha} - t^{1-\alpha}}. 
\end{align}
\end{theo}

The main requirement of Theorem \ref{theo:variance_stabilizing_clt} is \eqref{eq:prop_LB}, a limit on the rate
at which treatment assignment propensities $e_t$ decay. In a bandit setting, this constraint requires that suboptimal arms be pulled more often than implied by rate-optimal algorithms \citep[see e.g.,][Chapter 15]{tor2018bandit}, but still allows for sublinear regret. Given this constraint, the allocation rate bounds \eqref{eq:allocation_rate_bounds} are weak enough to allow us to construct variance-reducing heuristics, like \eqref{eq:two_point_allocation} below.

Given these simple sufficient conditions for our asymptotic normality result (Theorem~\ref{theo:arm_value_clt}) 
when we use variance-stabilizing weights, it remains to choose a specific allocation rate $\lambda_t$. This next step is what will allow us to be able to provide valid estimates even when the share of relevant data vanishes asymptotically. A simple choice of allocation rate is
\begin{equation}
\label{eq:const}
\lambda_t^{\text{const}} := \frac{1}{T - t + 1}.
\end{equation}
Given this choice, we can solve \eqref{eq:variance_stabilizing} in closed form and get $h_t = \sqrt{e_t / T}$. Weights of this type were proposed by
\cite{luedtke2016statistical} for the purpose of estimating the expected value of non-unique optimal policies that possibly depend on observable covariates. We call this method the \emph{constant allocation scheme}, because the variance contribution of each observation is constant (since $h_t^2/e_t \equiv 1/T$ for these weights).

The constant allocation scheme guarantees the variance-convergence condition \eqref{eq:variance_convergence} and ensures asymptotic normality of the test statistic \eqref{eq:clt},
but it does not result in a variance-optimal estimator. 
We propose an alternative scheme in which $\lambda_t$ adapts to past data and reweights observations 
to better control the estimator's variance.
To get some intuition, recall that from the discussion following Theorem \ref{theo:arm_value_clt} that the variance of \smash{$\hQ^h_T(w)$} essentially scales like
\smash{$\sum_t(h_t^2 / e_t) \,/\, (\sum_{t=1}^T h_t)^2$}. This implies that, in the absence of any constraints on how we choose the weights,
we would minimize variance by setting $h_t \propto e_t$; this can be accomplished using the allocation rate \smash{$\lambda_t = e_t / \sum_{s = t}^T e_s$}. If we use these weights and set $\hm_t \equiv 0$ in \eqref{eq:aw}, the result is an estimator that differs from the sample average \eqref{eq:avg} only in that it replaces the normalization $1 / T_w$ with $\smash{1/\sum_{t=1}^T e_t}$. Our results do not apply to this choice of allocation rate $\lambda_t$ because it depends on future treatment assignment probabilities, and Theorem \ref{theo:variance_stabilizing_clt} requires that $\lambda_t$ depend only on the history~$H^{t-1}$.

However, this form of allocation rate suggests a natural heuristic choice of allocation rate:
\begin{equation}
  \label{eq:lambda}
 \lambda_t = \hEE[t-1]{ \frac{ e_t(w) }{ \sum_{s=t}^T e_s(w) }},
\end{equation}
where \smash{$\widehat{\mathbb{E}}_{t-1}$} denotes an estimate of the future behavior of the propensity scores using information up to the beginning of the current period. It can be estimated via Monte Carlo methods. A high-quality approximation is unnecessary for valid inference. All that is required is that the allocation rate bounds \eqref{eq:allocation_rate_bounds} be satisfied, although better approximations likely lead to better statistical efficiency.

In practice, the need to compute these estimates renders the construction \eqref{eq:lambda} unwieldy.
Furthermore, the way the resulting weights depend on our model of the assignment mechanism is fairly opaque.
As an alternative, we consider a simple heuristic that exhibits similar behavior and can be used when assignment probabilities are decided via Bayesian methods such as Thompson sampling.

To derive our scheme, we consider two scenarios: one in which the assignment probabilities $e_t$ are currently high and will continue being so in the future, as is the case when a bandit algorithm deems $w$ to be an optimal arm; and a second scenario in which the assignment probabilities will asymptotically decay towards zero as fast as the lower bound \eqref{eq:prop_LB} permits it to. If the first scenario is true, then we could approximate the behavior of  \eqref{eq:lambda} by setting $e_s = 1$ for all periods. If the second scenario is true, then we could do so by setting $e_s = Cs^{-\alpha}$ for all periods. Letting $A$ be an indicator that we are in the first scenario, we can consider a heuristic approximation to \eqref{eq:lambda}:
\begin{equation}
    \label{eq:two_point_allocation_intuition}
    \lambda_t \approx A\frac{1}{\sum_{s=t}^{T} 1} + \p{1 - A} \frac{t^{-\alpha}}{\sum_{s=t}^{T} s^{-\alpha}}.
\end{equation}

Of course, we don't know which scenario we are in. However, when assigning treatment via Thompson sampling,
$e_t$ is the posterior probability at time $t$ that arm $w$ is optimal. This suggests the heuristic of averaging the two possibilities according to this posterior probability. Substituting, in addition, an integral approximation to $\smash{\sum_{s=t+1}^{T}s^{-\alpha}}$,
we get the following allocation rate:
\begin{equation}
    \label{eq:two_point_allocation}
    \lambda_t^{\text{two-point}} :=   e_t\frac{1}{T - t + 1} + \p{1 - e_t} \frac{t^{-a}}{t^{-\alpha} + \frac{T^{1-\alpha} - t^{1-\alpha}}{1-\alpha}}.
\end{equation}
We call this the \emph{two-point allocation scheme}. Both the constant \eqref{eq:const} and two-point \eqref{eq:two_point_allocation} allocation schemes
satisfy the allocation rate bounds \eqref{eq:allocation_rate_bounds} from Theorem \ref{theo:variance_stabilizing_clt}; see \appendixref[C.5]{sec:twopoint_allocation_bounds}.

\section{Estimating Treatment Effects}%
\label{sec:contrasts}

Our discussion so far has focused on estimating the value $Q(w)$ of a single arm $w$. In many applications,
however, we may seek to provide inference for a wider variety of estimands, starting with treatment effects
of the form $\Delta(w_1, \, w_2) = \EE{Y_t(w_1) - Y_t(w_2)}$. There are two natural ways to approach this
problem in our framework. The first involves re-visiting our discussion from Section \ref{sec:scoring}, and
directly defining unbiased scoring rules for $\Delta(w_1, \, w_2)$ that can then be used as the basis for an
adaptively-weighted estimator. The second is to re-use the value estimates derived above and set
\smash{$\hDelta(w_1, \, w_2) = \hQ(w_1) - \hQ(w_2)$}; the challenge then becomes how to provide uncertainty
quantification for \smash{$\hDelta(w_1, \, w_2)$}. We discuss both approaches below.

In the first approach, we use the difference in AIPW scores as the unbiased scoring rule for~$\Delta(w_1, \, w_2)$.
\begin{equation}
  \label{eq:AIPW_DELTA}
  \begin{aligned}
    &\hGamma_t(w_1, \, w_2) = \hGamma_t^{AIPW}(w_1) - \hGamma_t^{AIPW}(w_2), \\
    &\EE{\hGamma_t(w_1, \, w_2) \cond H^{t-1}} = \Delta(w_1, \, w_2).    
  \end{aligned}
\end{equation}
One can then construct asymptotically normal estimates of $\Delta(w_1, \, w_2)$ by adaptively weighting
the scores \smash{$\hGamma_t(w_1, \, w_2)$} as in \eqref{eq:aw}. In \appendixref[B]{sec:general_clt}, our main formal result
allows for adaptively-weighted estimation of general targets, such that both Theorem \ref{theo:arm_value_clt}
and adaptively-weighted estimation with scores \eqref{eq:AIPW_DELTA} are special cases of this
result.\footnote{\label{foot:average_derivative} Our result allows for considerably more generality than either of the cases discussed above,
and applies whenever our target admits a doubly robust estimator in the sense of \cite{chernozhukov2016locally}
whose Riesz representer is a function of the treatment assignment mechanism. For example, 
consider an adaptive clinical trial setting in which patients were given random doses of a continuous treatment $W_{t}$ drawn from a time-varying dosing policy $f_t(w)$, i.e., $W_t$ is a random variable with density $f_t(w)$, and write $m(w) = \EE{Y_t(w)}$. Now suppose that,
given a specific treatment assignment policy with density $f(w)$, we are interested in estimating $\psi(m) = \int m'(w)f(w)dw$, i.e.,
how patients' outcomes would change if they received doses in slightly larger amounts than those suggested by the
baseline policy $f(w)$. An unbiased scoring rule for this estimand is
\begin{equation}
  \label{eq:avgderiv}
  \hGamma_t = \gamma_t Y_t + \p{\int \hm'(w)f(w)dw - \gamma_t \hm(W_t)} \quad \text{where}\quad \gamma_t = -\frac{f'(W_t)}{f_t(W_t)},
\end{equation}
and our results apply to inference about $\psi(m)$ by adaptively-weighted aggregation of these \smash{$\hGamma_t$}. See \appendixref[B]{sec:general_clt} for further discussion.}

The second approach is conceptually straightforward; however, we still need to check that
the estimator \smash{$\hDelta(w_1, \, w_2) = \hQ(w_1) - \hQ(w_2)$} can be used for asymptotically
normal inference about $\Delta(w_1,\, w_2)$. Theorem \ref{theo:contrast_clt} provides such a result,
under a modified version of the conditions of Theorem \ref{theo:arm_value_clt} along with an extra
assumption \eqref{eq:evaluation_weight_convergence}.

\begin{theo}
\label{theo:contrast_clt}
In the setting of Theorem \ref{theo:arm_value_clt}, let $w_{1}, \, w_{2} \in \mathcal{W}$ denote a pair of arms,
and suppose that Assumptions \ref{assu:inf}, \ref{assu:variance_convergence} and \ref{assu:lyapunov} are satisfied for both arms. In addition, suppose that the variance estimates defined in \eqref{eq:clt} satisfy
\begin{equation}
 \label{eq:evaluation_weight_convergence}
 \hV_T^h(w_1) \,\big/\, \hV^{h}_T(w_2) \xrightarrow[T \to \infty]{p} r \in [0, \infty].
\end{equation}
and that for at least one $j \in \{1, 2\}$, either 
\begin{equation}
     \label{eq:e_mu_alternative_constrast}
      \hm_t(w_j) \xrightarrow[t \to \infty]{a.s.} Q(w_j) \quad \text{or}
      \quad e_t(w_j)  \xrightarrow[t \to \infty]{a.s.} 0.
\end{equation}

Then the vector of studentized statistics \eqref{eq:clt} for $w_{1}$ and $w_{2}$ is asymptotically jointly
normal with identity covariance matrix.  Moreover, \smash{$\hDelta_T(w_{1}, w_{2})$} below is a consistent
estimator of $\Delta(w_{1}, w_{2}) = \EE{Y_t(w_{1}) - Y_t(w_{2})}$, 
  \begin{align}
    \label{eq:delta}
    \hDelta_T(w_1, w_2)
      := 
       \frac{ \sum_{t=1}^T h_t(w_{1})\hGamma_t(w_{1}) }
            { \sum_{t=1}^Th_t(w_{1}) } 
            -
      \frac{ \sum_{t=1}^T h_t(w_{2})\hGamma_t(w_{2}) }
           { \sum_{t=1}^Th_t(w_{2}) },
  \end{align}
  and the following studentized statistic is asymptotically standard normal.
  \begin{align}
    \label{eq:studentized_contrast}
    \frac{ \hDelta_T(w_{1}, w_{2}) - \Delta(w_{1}, w_{2})}
          { \p{ \hV^h_T(w_{1}) + \hV^h_T(w_{2}) }^{1/2}  }
          \xrightarrow[T \to \infty]{d} \mathcal{N}(0, 1). 
  \end{align}
\end{theo}

Both approaches to inference about $\Delta(w_{1}, \, w_{2})$ are of interest, and may be relevant
in different settings. In our experiments, we focus on the estimator
\smash{$\hDelta(w_1, \, w_2) = \hQ(w_1) - \hQ(w_2)$} studied in Theorem \ref{theo:contrast_clt}, as we found
it to have higher power---presumably because allowing separate weights $h_t(w)$ for different arms gives us more
control over the variance. However, adaptively-weighted estimators following \eqref{eq:AIPW_DELTA} that directly target
the difference $\Delta(w_{1}, \, w_{2})$ may also be of interest in some applications. In particular, they
render an assumption like \eqref{eq:evaluation_weight_convergence} unnecessary and, following
the line of argumentation in \cite{zhang2020inference}, they may be more robust to non-stationarity of the distribution
of the potential outcomes $Y_t(w)$.

\ifpnas
  \begin{figure*}[ht]
    \centering
    \includegraphics[width=\textwidth]{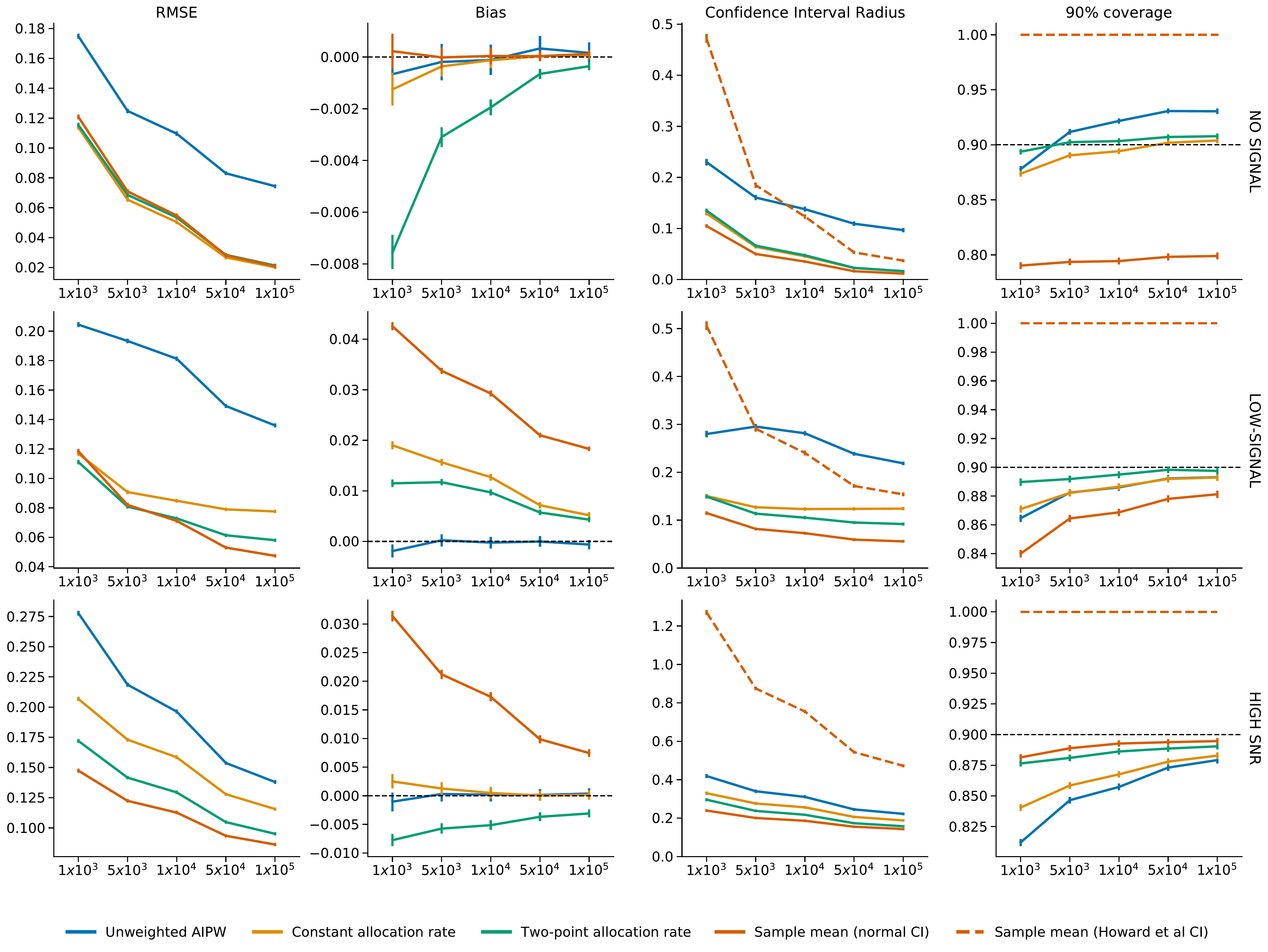}
    \caption{Evolution of estimates of $Q(3) - Q(1)$ across simulation settings for different experiment lengths. 
  	   Error bars are 95\% confidence intervals around averages across $10^5$ replications.}
    \label{fig:contrast}
  \end{figure*}

  \begin{figure*}[ht]
    \centering
    \includegraphics[width=.8\textwidth]{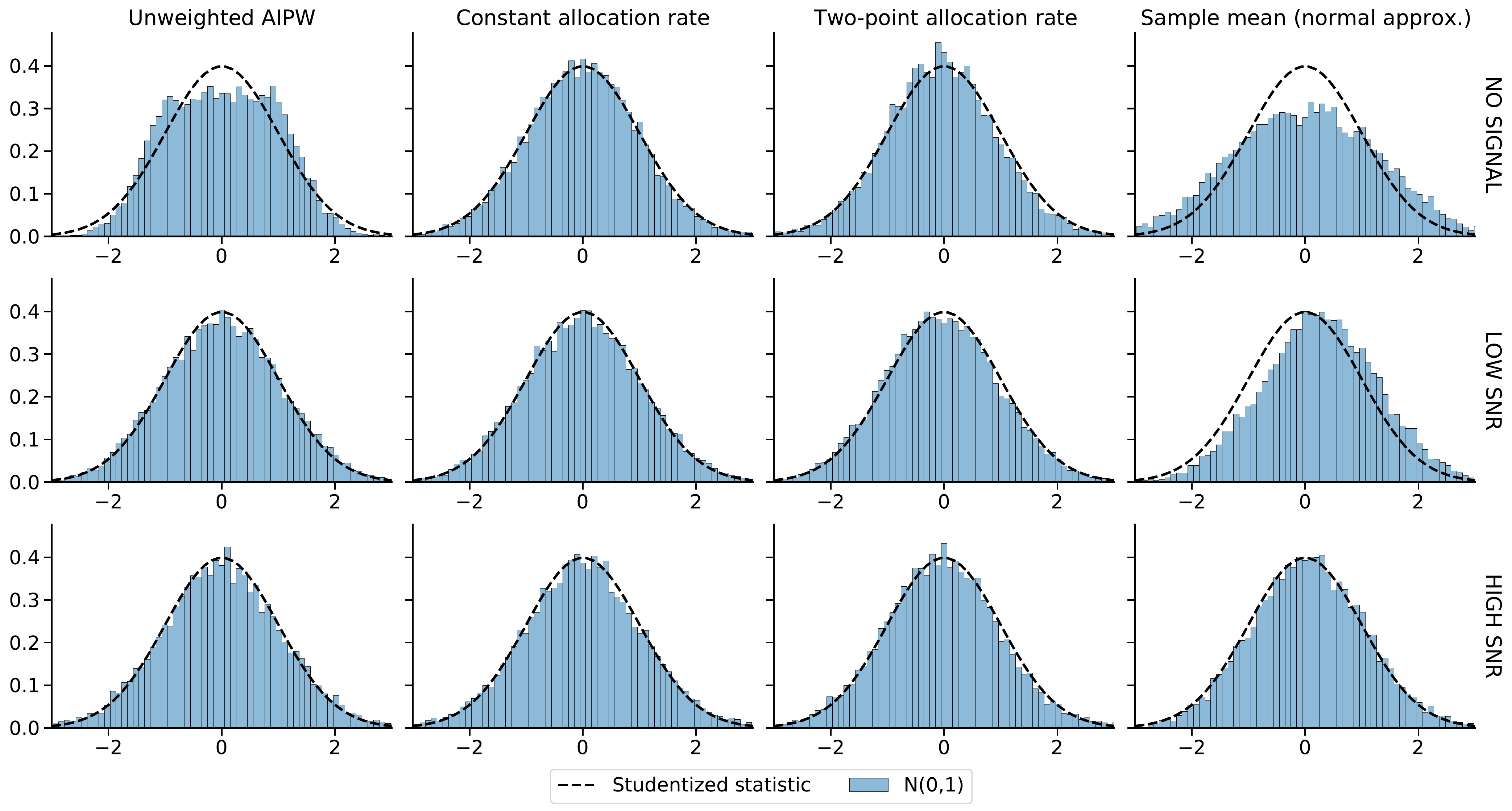}
    \caption{Histogram of studentized statistics of the form $(\hDelta_T - \Delta)/\hV_T^{1/2}$ for the difference in arm values $\Delta(3, 1) = \EE{Y_t(3) - Y_t(1)}$ at $T = 10^5$. Numbers aggregated across over $10^5$ replications.}
    \label{fig:histogram}
  \end{figure*}
\fi

\section{Numerical Experiments}
\label{sec:sims}

We compare methods for estimating of arm values $Q(w)$ and their differences $\Delta(w, w')$, as well as constructing confidence intervals around these estimates, under different data-generating processes. 

We consider four point estimators of arm values $Q(w)$: the sample mean \eqref{eq:avg}, the AIPW estimator \eqref{eq:aipw}, and the adaptively-weighted AIPW estimator \eqref{eq:aw} with constant \eqref{eq:const} and two-point \eqref{eq:two_point_allocation} allocation rates. Around each of these estimators, we construct confidence intervals $\smash{\hQ_T \pm z_{\alpha/2} \hV_T^{1/2}}$ based on the assumption of approximate normality. For the AIPW-based estimators\footnote{Recall that the AIPW estimator \eqref{eq:aipw} is an instance of the adaptively-weighted AIPW \eqref{eq:aw} with uniform weights $h_t \equiv 1$, so we may use the same formula for the variance \eqref{eq:clt}.} we use the sample mean of arm rewards up to period $t-1$ as the plug-in estimator $\hat{m}_t(w)$ and the estimate of the variance given in \eqref{eq:clt}. For the sample mean we use the usual variance estimate $\smash{\hV^{AVG}(w) := T_w^{-2} \sum_{t: W_t = w}^{T} (Y_t - \hQ^{AVG}_T(w))^2}$. Approximate normality is not theoretically justified for the unweighted AIPW estimator or for the sample mean. We also consider non-asymptotic confidence intervals for the sample mean, based on the method of time-uniform confidence sequences described in \cite{howard2021uniform}. See \appendixref[D.2]{sec:nonasymptic_ci} for details.

In addition, we consider four analogous estimators for the treatment effect $\Delta(w, w')$: the difference in sample means, the difference in AIPW estimators \eqref{eq:aipw}, and the adaptively-weighted AIPW estimator \eqref{eq:delta} with constant \eqref{eq:const} and two-point \eqref{eq:two_point_allocation} allocation rates. For the AIPW-based estimators we use $\hat{m}_t(w)$ as above for the plugin and the estimate of the variance given in \eqref{eq:studentized_contrast}. For the sample mean, we use the usual variance estimate $\smash{\hV^{AVG}(w) + \hV^{AVG}(w')}$. For the method based on \cite{howard2021uniform}, we construct confidence intervals for the treatment effect by using the union bound to combine intervals around each sample mean. 

We have three simulation settings, each with $K = 3$ arms yielding rewards that are observed with additive $\text{uniform}[-1, 1]$ noise. The settings differ in the size of the gap between the arm values. In the \emph{no-signal} case, we set arm values to $Q(w) = 1$ for all $w \in \{1, 2, 3\}$; in the \emph{low signal} case, we set $Q(w) = 0.9+ 0.1w$; and \emph{high signal} case we set $Q(w) = 0.5 + 0.5w$.  During the experiment, treatment is assigned by a modified Thompson sampling method \citep[see e.g.,][]{russo2018tutorial}: first, tentative assignment probabilities are computed via Thompson sampling with normal likelihood and normal prior; they are then adjusted to impose the lower bound $e_t(w) \geq (1/K)t^{-0.7}$. See \appendixref[D.1]{sec:thompson_sampling} for details.

As a short mnemonic, in what follows we call arms $1$ and $3$ the ``bad'' arm and ``good'' arm. 
As these labels are fixed, tests involving the value of the ``good'' arm are tests of a pre-specified hypothesis.
Figure \ref{fig:contrast} shows the evolution of estimates of the difference $\Delta(3, 1) = \EE{Y(3) - Y(1)}$ over time, Figure \ref{fig:histogram} shows the asymptotic distribution of studentized test statistics of the form $(\hDelta_T - \Delta)/\hV_T^{1/2}$ for each estimator at the end of a long ($T = 10^5$) experiment, and Figure \ref{fig:arm_values} shows arm value statistics. Additional results are shown in \appendixref[A]{sec:additional_figures}.\footnote{Reproduction code can be found at \url{https://github.com/gsbDBI/adaptive-confidence-intervals}.} 
 
Figures~\ref{fig:contrast} and~\ref{fig:arm_values} show that although the AIPW estimator with uniform weights (labeled as ``unweighted AIPW'') is unbiased, it performs very poorly in terms of RMSE and confidence interval width. In the low and high signal case, its problem is that it does not take into account the fact that the bad arm is undersampled, so its variance is high; in the no-signal case, it yields studentized statistics that are far from normal, as we see in Figure~\ref{fig:histogram}. 

Figures~\ref{fig:contrast} and~\ref{fig:arm_values} show that our adaptively weighted AIPW estimators perform relatively well and 
normal confidence intervals intervals around them have roughly correct coverage. We see that these estimators do have approximately normal studentized statistics
in Figure~\ref{fig:histogram}.
Note that even in our longest experiments, in high-signal settings the bad arm receives only around 50 observations, which suggests that normal approximation does not require an impractical number of observations.\footnote{For intuition about the number of observations we see in the bad arm in the high-signal case, consider that the assignment probabilities 
of suboptimal arms quickly hit their lower bound \smash{$e_t = (1/3)t^{-0.7}$; $(1/3)\sum_{t=1}^{10^5}t^{-0.7} \approx 35$.}}
 Of these two methods, \emph{two-point allocation} better controls the variance of bad arm estimates by more aggressively downweighting `unlikely' observations 
associated with large inverse propensity weights; this results in smaller RMSE and tighter confidence intervals. 

As mentioned in the introduction, the sample mean is downwardly-biased for arms that are undersampled. Figure \ref{fig:arm_values} shows this bias can be non-monotonic in signal strength. In the \emph{high-signal} case, the probability of pulling the bad arm decays so fast that very few observations are assigned to it. 
Most of these come from the beginning of the experiment when the algorithm is still exploring and sampling is less adaptive, resulting in smaller bias.
In the \emph{no-signal case}, the bias is small because the `good' and `bad' arms have the same value. 
In some simulations, one arm or another is discarded and its estimate is biased downward; in others, it is collected heavily and its estimate is nearly unbiased. Averaging over these scenarios results in the low bias we observe. Between these extremes, in the \emph{low-signal} case, the bad arm is usually collected for some period and then discarded, so its bias is larger in magnitude. 
For this estimator, naive confidence intervals based on the normal approximation are invalid, with severe under-coverage when there's little or no signal. 
On the other hand, the non-asymptotic confidence sequences of \cite{howard2021uniform} are conservative but often wide.

These simulations suggest that, in similar applications, the adaptively-weighted AIPW estimator with two-point allocation \eqref{eq:two_point_allocation} and 
the sample mean with confidence sequences based on \cite{howard2021uniform} should be preferred. These two methods have complementary advantages. 
In terms of mean-squared error, the sample mean often performs better, in particular in the presence of stronger signal. As for inference, normal intervals around the adaptively-weighted estimator with two-point allocation have asymptotically nominal coverage, while confidence sequences are often conservative and wider than those based on normal approximations; however, the former is valid only at a pre-specified horizon, while the latter is valid for all time periods and allows for arbitrary stopping times. Finally, in terms of assumptions, the adaptively-weighted estimator requires knowledge of the propensity scores, and its justification requires
that the propensity scores decay at a slow enough rate; the nonasymptotic confidence sequences for the sample mean require 
no restrictions on the assignment process and can be used even with deterministic algorithms such as UCB \citep{auer2002nonstochastic},\footnote{For deterministic algorithms such as UCB, the methods in \cite{dimakopoulou2017estimation} suggest inverse propensity weighting schemes whose weights are of the form $1/\max\{\hat{e}_t(w), \gamma \}$, where $\hat{e}_t(w)$ are estimates of assignment probabilities based, e.g., on the empirical distribution of past assignments, and $\gamma > 0$ is a lower bound. However, this heuristic is not guaranteed to produce asymptotically normal estimates.} 
but requires knowledge about other aspects of the distribution of potential outcomes, such as their support or an upper bound on their variance \citep[see][Section 3.2]{howard2020chernoff}.

\ifpnas\else
\begin{figure*}[ht]
  \centering
  \includegraphics[width=\textwidth]{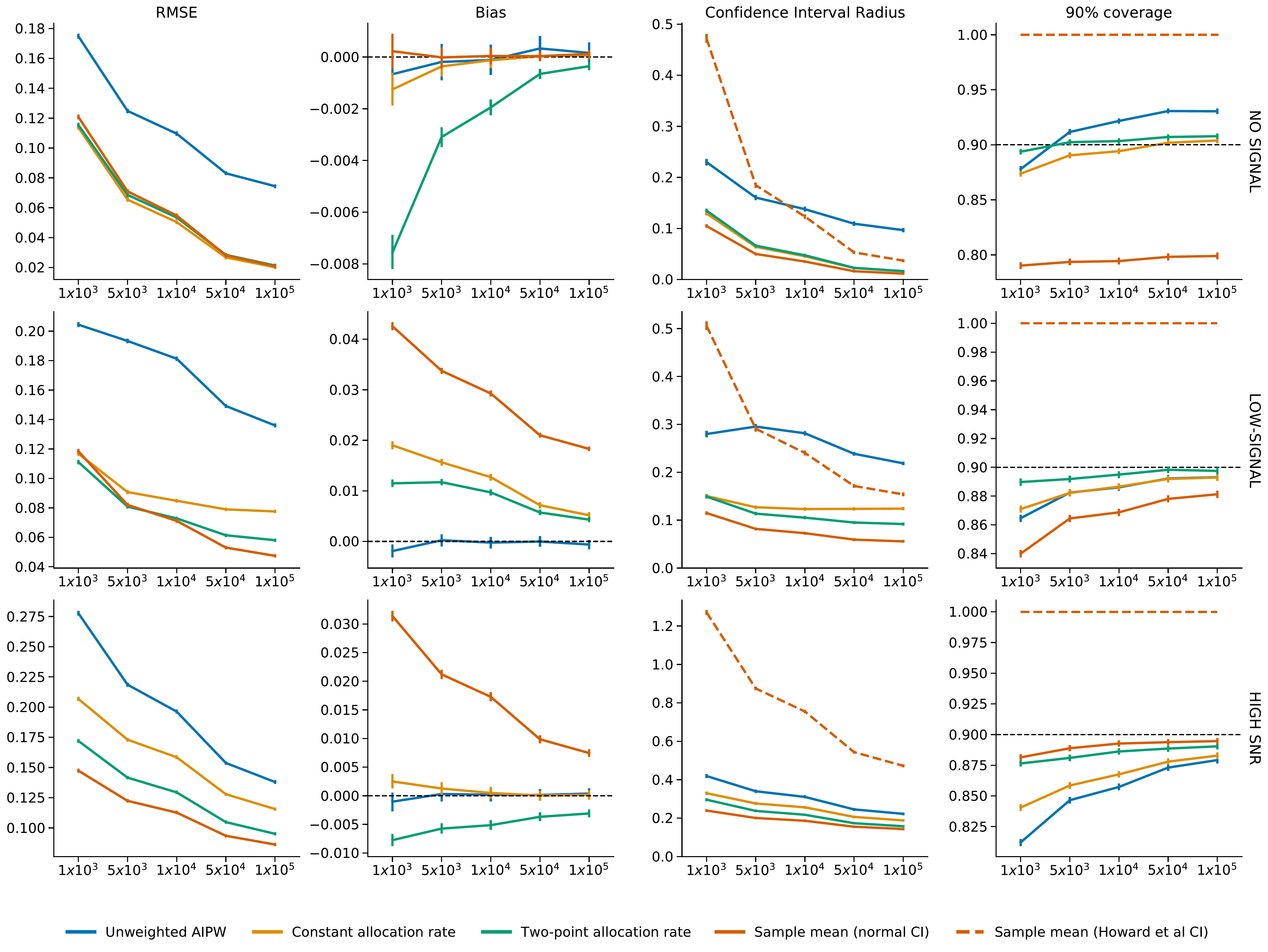}
  \caption{Evolution of estimates of $Q(3) - Q(1)$ across simulation settings for different experiment lengths. 
	   Error bars are 95\% confidence intervals around averages across $10^5$ replications.}
  \label{fig:contrast}
\end{figure*}

\begin{figure*}[ht]
  \centering
  \includegraphics[width=.7\textwidth]{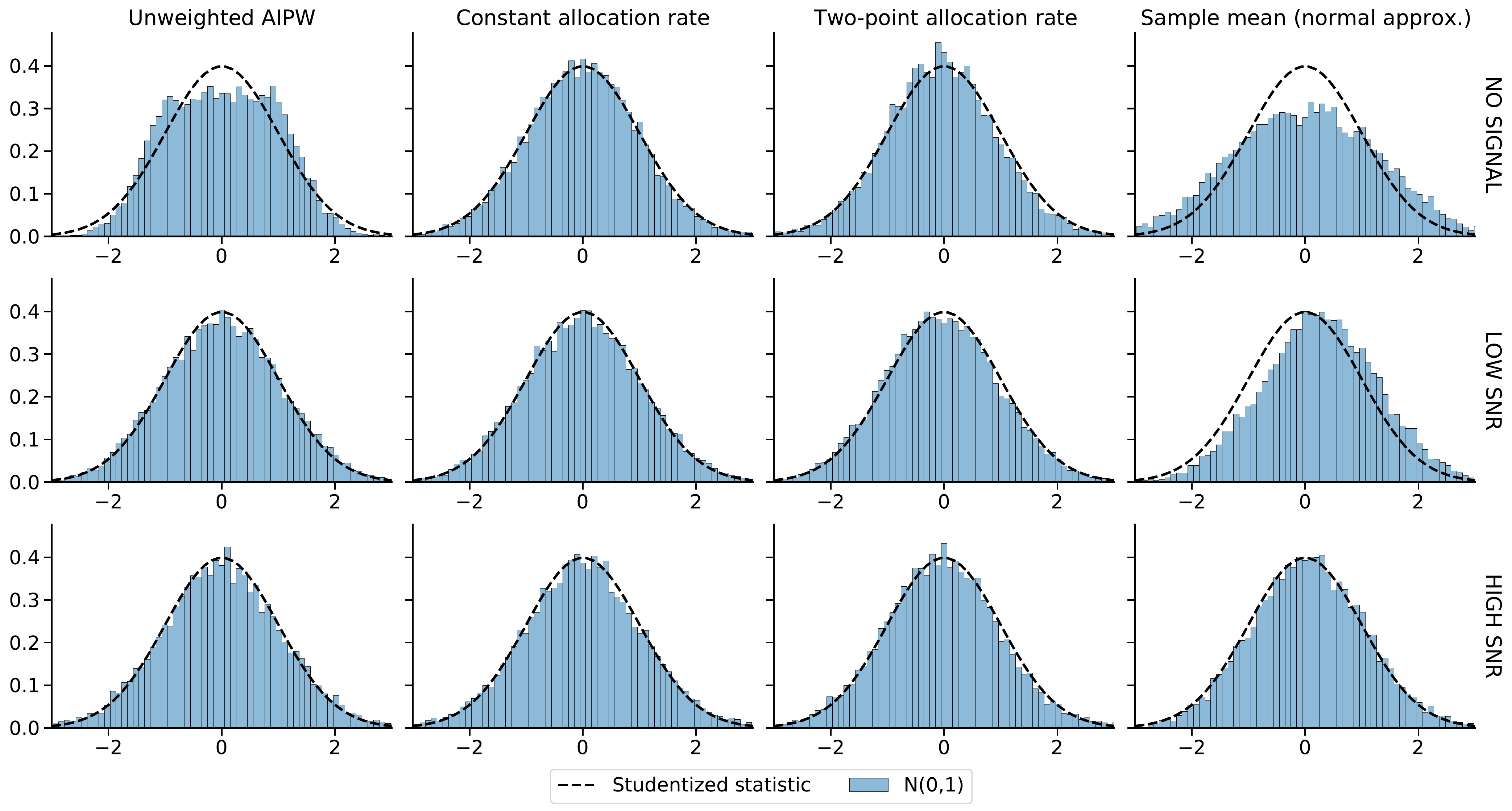}
  \caption{Histogram of studentized statistics of the form $(\hDelta_T - \Delta)/\hV_T^{1/2}$ for the difference in arm values $\Delta(3, 1) = \EE{Y_t(3) - Y_t(1)}$ at $T = 10^5$. Numbers aggregated across over $10^5$ replications.}
  \label{fig:histogram}
\end{figure*}
\fi

\begin{figure*}[ht]
  \centering
  \includegraphics[width=\textwidth]{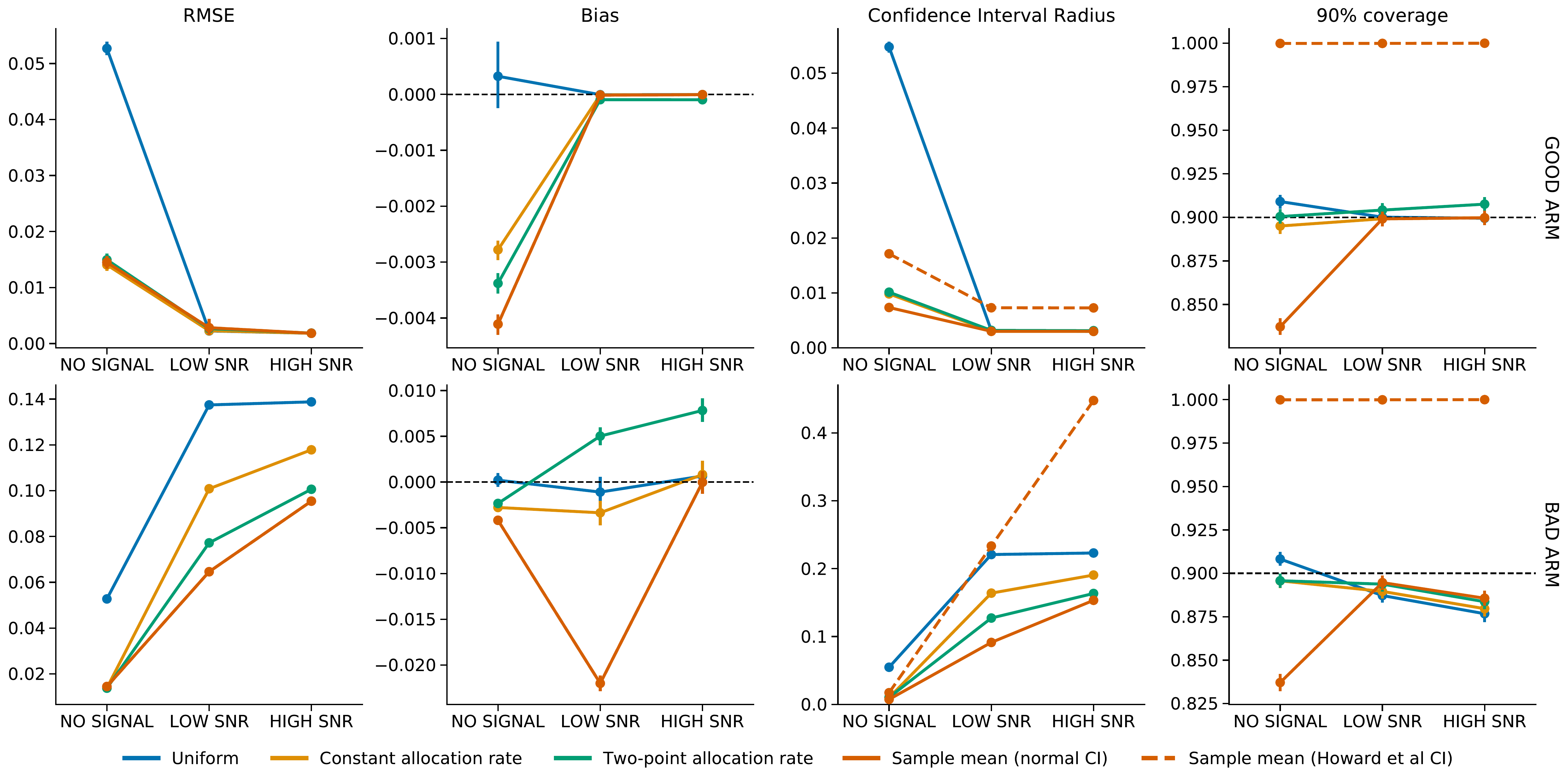}
  \caption{Estimates of the ``bad'' arm value $Q(1)$ and ``good'' arm value $Q(3)$ at $T = 10^5$. Error bars are 95\% confidence intervals around averages across $10^5$ replications.}
  \label{fig:arm_values}
\end{figure*}

\section{Related literature}
\label{sec:lit}

Much of the literature on policy evaluation with adaptively collected data focuses on learning or
estimating the value of an optimal policy. The classical literature \citep[e.g.,][]{rosenberger2015randomization} focuses on strategies for allocating treatment in clinical trials to optimize various criteria, such as determining whether a treatment is
helpful or harmful relative to control. \cite{van2008construction} generalizes this substantially, addressing
the problem of optimally allocating treatment to estimate or testing a hypothesis about a finite-dimensional parameter of the distribution of the data.
In optimal allocation problems, the undersampling issue we address by adaptive weighting does not arise,
as undersampling treatments relevant to the estimand or hypothesis of interest is suboptimal.

\citet{van2014online} considers the problem of policy evaluation when treatment is sequentially randomized but otherwise unrestricted. 
The estimator they propose in their Section~10.3, when specialized to the problem of estimating an arm value, reduces to the augmented inverse propensity weighting estimator \eqref{eq:aw}. They establish asymptotic
normality of their estimator under assumptions implying that a non-negligible proportion of participants is assigned the treatment of interest throughout the study. \cite{luedtke2016statistical} proposes a stabilized variant of this estimator which, similarly specialized, reduces to the adaptively-weighted estimator with constant allocation rates \eqref{eq:const}. The applicability of this approach to bandit problems is mentioned in \cite{luedtke2018parametric}. 
\cite{zhang2020inference} consider a similar refinement of an analogous weighted average estimator (see \eqref{eq:weighted-avg} below) for batched adaptive designs.

Drawing on the tradition of debiasing in high-dimensional statistics, \cite{deshpande2017accurate} propose a method that can be used to estimate policy values in non-contextual and linear-contextual bandits. Their approach, W-decorrelation, yields consistent and asymptotically normal estimates of linear regression coefficients when the covariates have strong serial correlation. For multi-armed bandits, where the arm indicators are used as the covariates, their estimates of arm values take the form
\begin{equation*}
  \begin{aligned}
    &\hQ_T^{WD}(w) := \bar{Y}_{w, T} + \sum_{t=1}^T a_{t, w}(Y_t - \bar{Y}_{w, T}) 
    \quad \text{where}\quad \\
    &a_{t, w} := \frac{1}{1 + \lambda}\p{\frac{\lambda}{1 + \lambda}}^{N_{W_t, t-1}}\mathbb{I}\{ W_t = w \},  
  \end{aligned}
\end{equation*}
where $N_{w, t}$ is the number of times arm $w$ was selected up to period $t$, $\bar{Y}_{w, T}$ is the sample average of its rewards at $T$, and $\lambda$ is a tuning parameter. We include a numerical evaluation of W-decorrelation in \appendixref[A]{sec:additional_figures}, finding it to produce arm value estimates with high variance.

\section{Discussion}

Adaptive experiments such as multi-armed bandits are often a more efficient way of collecting data than traditional randomized controlled trials. However, they bring about several new challenges for inference. Is it possible to use bandit-collected data to estimate parameters that were not targeted by the experiment? Will the resulting estimates have asymptotically normal distributions, allowing for our usual frequentist confidence intervals? This paper provided sufficient conditions for these questions to be answered in the affirmative and proposed an estimator that satisfies these assumptions by construction. Our approach relies on constructing averaging estimators where the weights are carefully adapted so that the resulting asymptotic distribution is normal with low variance. In empirical applications, we have shown that our method outperforms existing alternatives, in terms of both mean squared error and coverage.

We believe this work represents an important step towards a broader research agenda for policy learning and evaluation in adaptive experiments. Natural questions left open include the following.

{\it What other estimators can be used for normal inference with adaptively collected data?}
In this paper, we have focused on estimators derived via the adaptively-weighted AIPW construction
\eqref{eq:aw}. However, this is not the only way to obtain normal confidence intervals. For example,
in the setting of Theorem \ref{theo:arm_value_clt}, one could also consider the weighted average estimator
\begin{equation}
  \label{eq:weighted-avg}
  \begin{split}
  \hQ^{h\text{-avg}}_T(
  w) = \sum_{t=1}^T h_t(w) \frac{ \ind{W_t = w} }{e_t(w)} \, Y_t\, \bigg /\,   \sum_{t=1}^T h_t(w) \frac{\ind{W_t = w} }{e_t(w)}. 
\end{split}
\end{equation}
 Asymptotic normality of this estimator essentially follows from Theorem \ref{theo:arm_value_clt} (see \appendixref[C.4]{sec:alternative_estimators}). 
In our numerical experiments, we found its performance to be essentially indistinguishable from that of the adaptively-weighted estimator defined as in \eqref{eq:aw}. We've focused on \eqref{eq:aw} because it readily allows formal study in a more general setting; however, the simpler estimator \eqref{eq:weighted-avg} is appealing in the special case of evaluating a single arm. For a more general discussion of the relationship between augmented estimators like \smash{$\hQ^{h}$} and variants like \smash{$\hQ^{h\text{-avg}}$} in adaptive experiments, see \cite{van2014online}.

{\it What should an optimality theory look like?}
Our result in Theorem \ref{theo:arm_value_clt} provides one recipe for building confidence
intervals using an adaptive data collection algorithm like Thompson sampling for which we know
the treatment assignment probabilities. Here, however, we have no optimality guarantees on
the width of these confidence intervals. It would be of interest to characterize, e.g.,
the minimum worst-case expected width of normal confidence intervals that can be built using such data.

{\it How do our results generalize to more complex sampling designs?}
In many application areas, there's a need for methods for policy evaluation and inference that work with more general
designs such as contextual bandits, and in settings with non-stationarity or random stopping.

\section*{Acknowledgments}

We are grateful for the generous financial support provided by
the Sloan Foundation,
Office of Naval Research grant N00014-17-1-2131,
National Science Foundation grant DMS-1916163,
Schmidt Futures,
Golub Capital Social Impact Lab,
and the Stanford Institute for Human-Centered Artificial Intelligence. 
Ruohan Zhan acknowledges generous support from the Total Innovation graduate fellowship. In addition, we thank Steve Howard, Sylvia Klosin, Sanath Kumar Krishnamurthy and Aaditya Ramdas for helpful conversations.

\bibliographystyle{apalike}
\bibliography{references}

\ifpnas\else
  \newpage
  \appendix
  
\section{Additional Figures}
\label{sec:additional_figures}

Here we present additional numerical results. Sections \ref{sec:evolution_arm}-\ref{sec:lambdapath} show further results for the simulation setting described in Section \ref{sec:sims}. Section \ref{sec:evolution_arm} presents the evolution of estimates of arm values over time, and \ref{sec:normality_arm} presents their distribution at $T = 10^5$. Section \ref{sec:wdecorr} compares the adaptively-weighted AIPW estimator with two-point allocation with the W-decorrelation method of \cite{deshpande2017accurate}. Section \ref{sec:lambdapath} discusses the behavior of the two-point allocation over time. In addition, in Section \ref{sec:intro_example_revisited} we apply an adaptively-weighted AIPW estimator to the example described in the introduction of the main section of the paper.
 
As in the main section of the paper, all numbers below are aggregated over at least 10,000 simulations, and all error bars are 95\%-confidence intervals.

\subsection{Behavior of arm value estimators over time}
\label{sec:evolution_arm}

Figures \ref{fig:evolution_arm_values2}-\ref{fig:evolution_arm_values0} show the behavior of different estimators of arm values $Q(w)$ over time. Qualitatively, these figures are similar to Figure \ref{fig:contrast} discussed in the main section of the paper. Figure \ref{fig:evolution_arm_values2} demonstrates that the we are able to estimate the good arm value $Q(3)$ with considerable ease in high- and low-signal settings, in that all methods produce point estimates with negligible bias and small root mean-squared error, and moreover attain roughly correct coverage for large $T$. This suggests that the poor performance discussed in the main section of the paper when discussing estimates of the difference $\Delta(3, 1)$ are mostly due to the fact that estimating the bad arm value $Q(3)$ is hard. This is confirmed in Figure~\ref{fig:evolution_arm_values2}.

\begin{figure}[H]
  \centering
  \includegraphics[width=\textwidth]{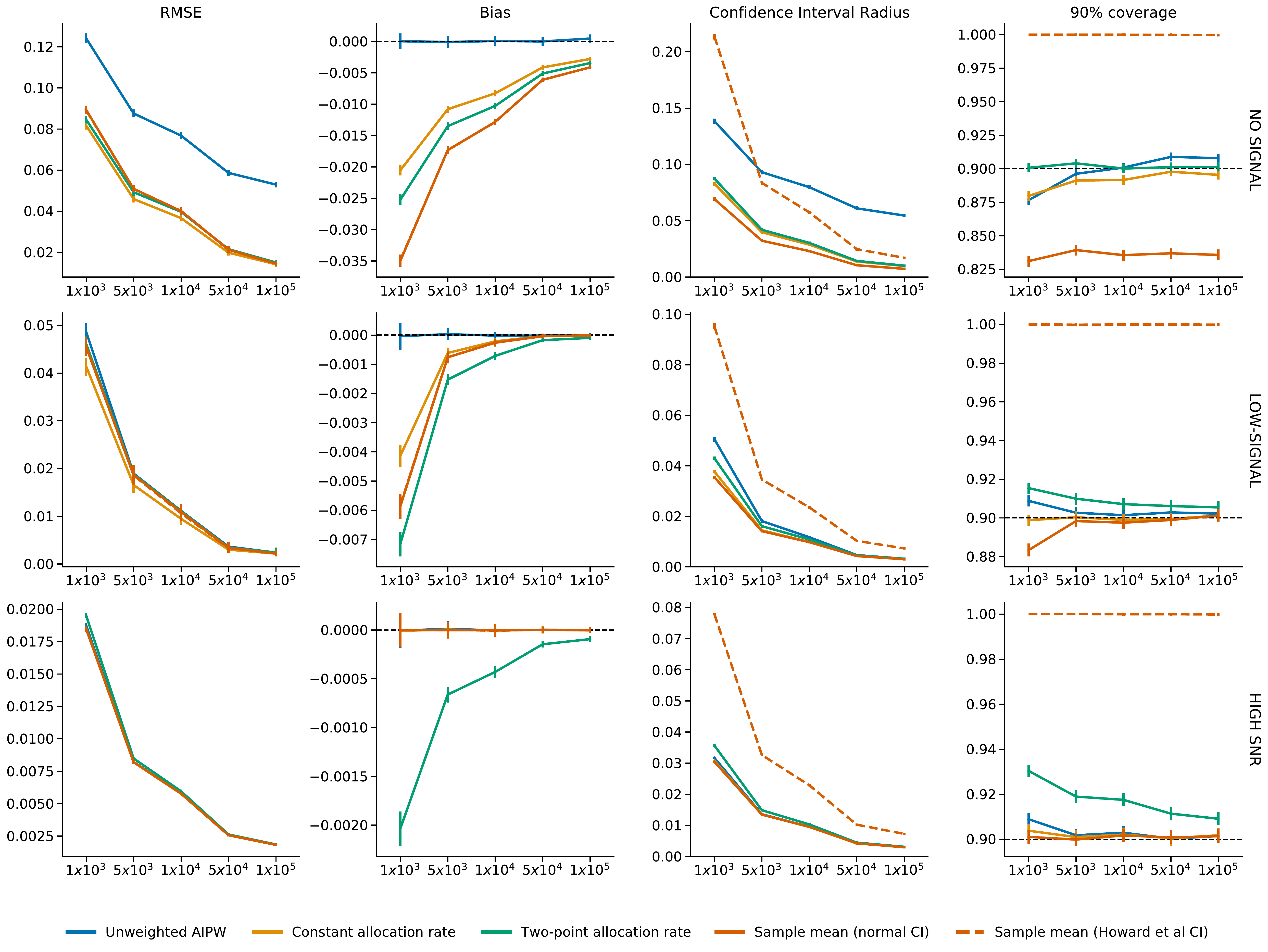}
  \caption{Estimates of the ``good'' arm value $Q(3)$ for different experiment lengths.}
  \label{fig:evolution_arm_values2}
\end{figure}

\begin{figure}[H]
  \centering
  \includegraphics[width=\textwidth]{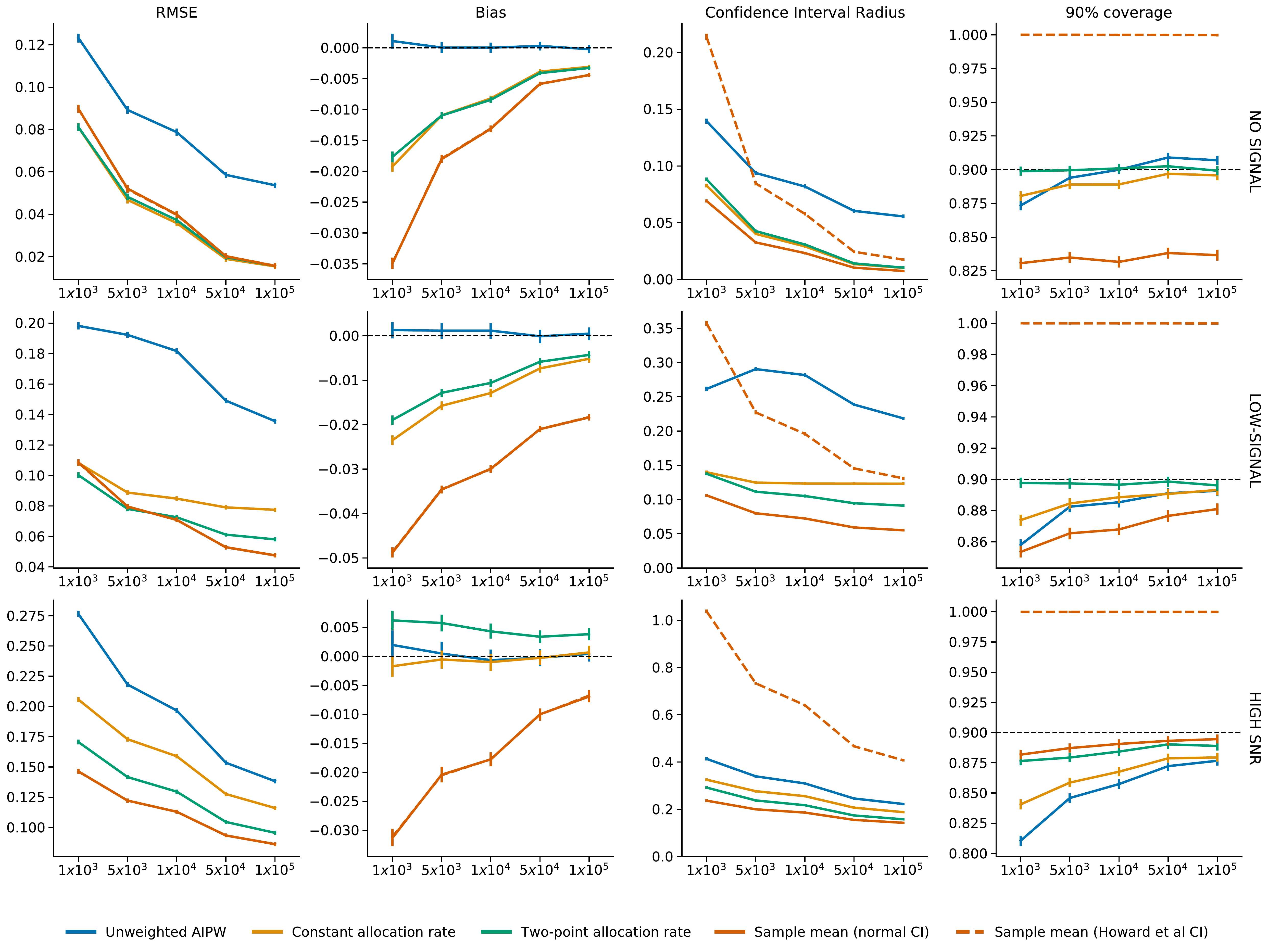}
  \caption{Estimates of the middle arm value $Q(2)$ across for different experiment lengths.}
  \label{fig:evolution_arm_values1}
\end{figure}

\begin{figure}[H]
  \centering
  \includegraphics[width=\textwidth]{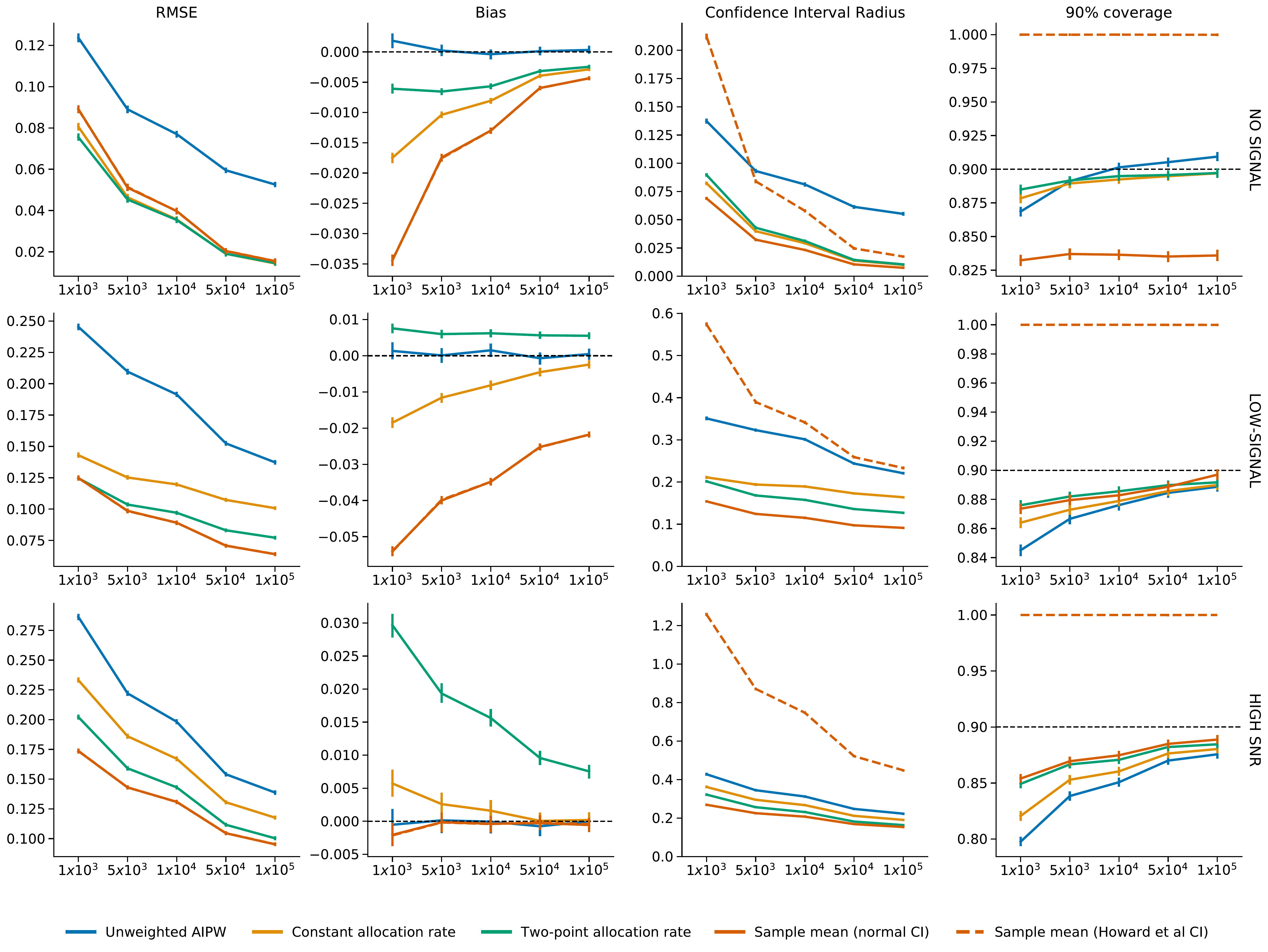}
  \caption{Estimates of the ``bad'' arm value $Q(1)$ across for different experiment lengths.}
  \label{fig:evolution_arm_values0}
\end{figure}

\subsection{Asymptotic normality of arm value estimators} 
\label{sec:normality_arm}

In Figures \ref{fig:histogram_2}-\ref{fig:histogram_0}, the first three columns show the asymptotic distribution of the studentized statistic \eqref{eq:clt} associated with the adaptively-weighted estimator for different weighting schemes; the last column shows the difference between sample mean estimate and true value divided by standard error of the mean. As discussed the main section of the paper, when there is no signal neither the unweighted AIPW estimator or the sample mean have an approximately normal distribution.

\begin{figure}[H]
  \centering
  \includegraphics[width=.8\textwidth]{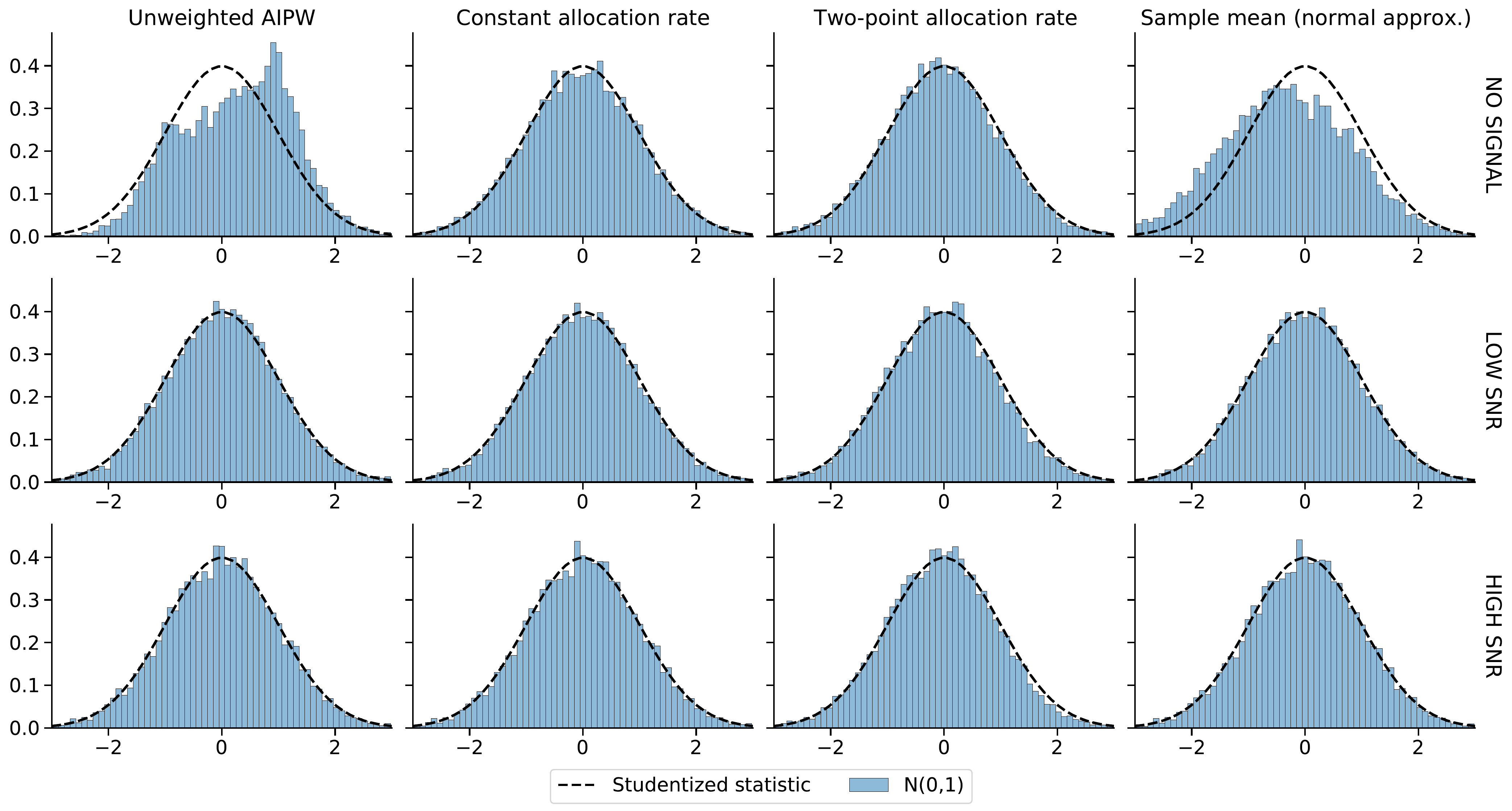}
  \caption{Distribution of studentized statistics \eqref{eq:clt} of the value of the ``good'' arm $Q(3)$ at~$T = 10^5$.}
  \label{fig:histogram_2}
\end{figure}

\begin{figure}[H]
  \centering
  \includegraphics[width=.8\textwidth]{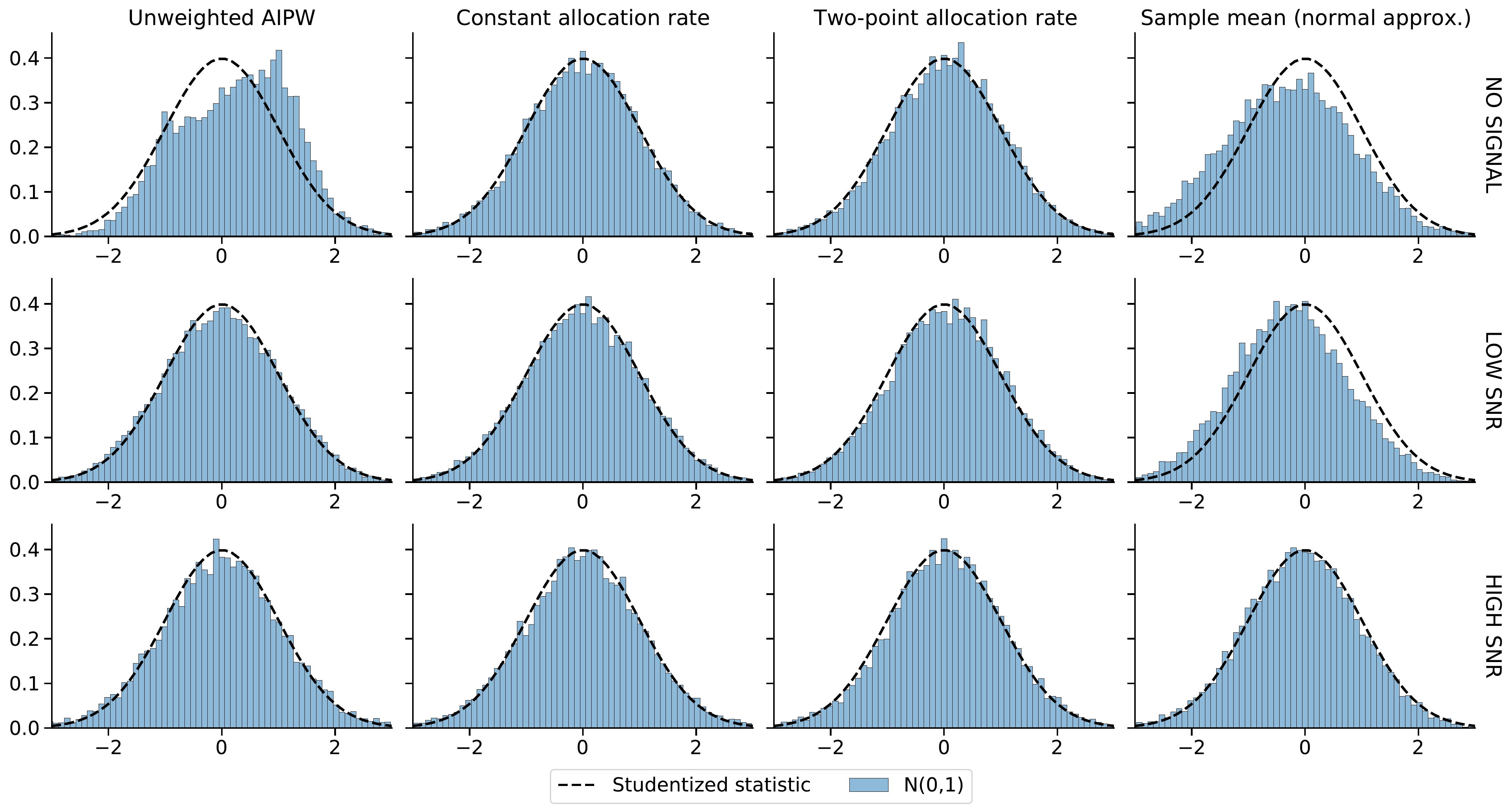}
  \caption{Distribution of studentized statistics \eqref{eq:clt} of $Q(2)$ at $T = 10^5$.}
  \label{fig:histogram_1}
\end{figure}

\begin{figure}[H]
  \centering
  \includegraphics[width=.8\textwidth]{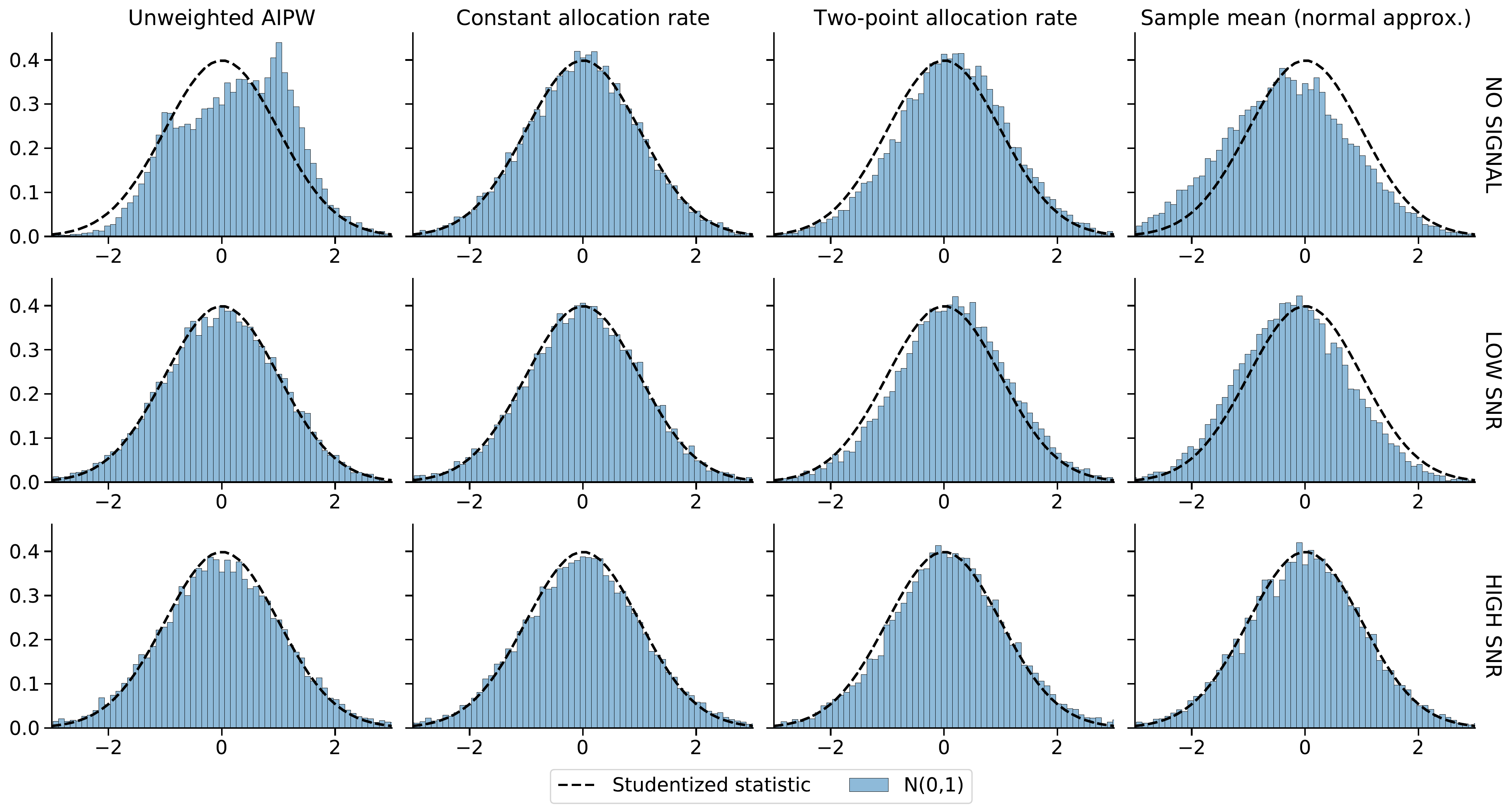}
  \caption{Distribution of studentized statistics \eqref{eq:clt} of value of the ``bad'' arm $Q(1)$ at $T = 10^5$.}
  \label{fig:histogram_0}
\end{figure}

\subsection{Comparison to W-decorrelation}
\label{sec:wdecorr}

Figure \ref{fig:compare_W} compares our ``two-point allocation method'' for variance stabilization to the W-decorrelation method of \cite{deshpande2017accurate}. We see that although both methods attain the correct coverage, W-decorrelation requires much more variance, resulting in an estimator that typically has much higher mean squared error.

\begin{figure}[H]
  \centering
  \includegraphics[width=\textwidth]{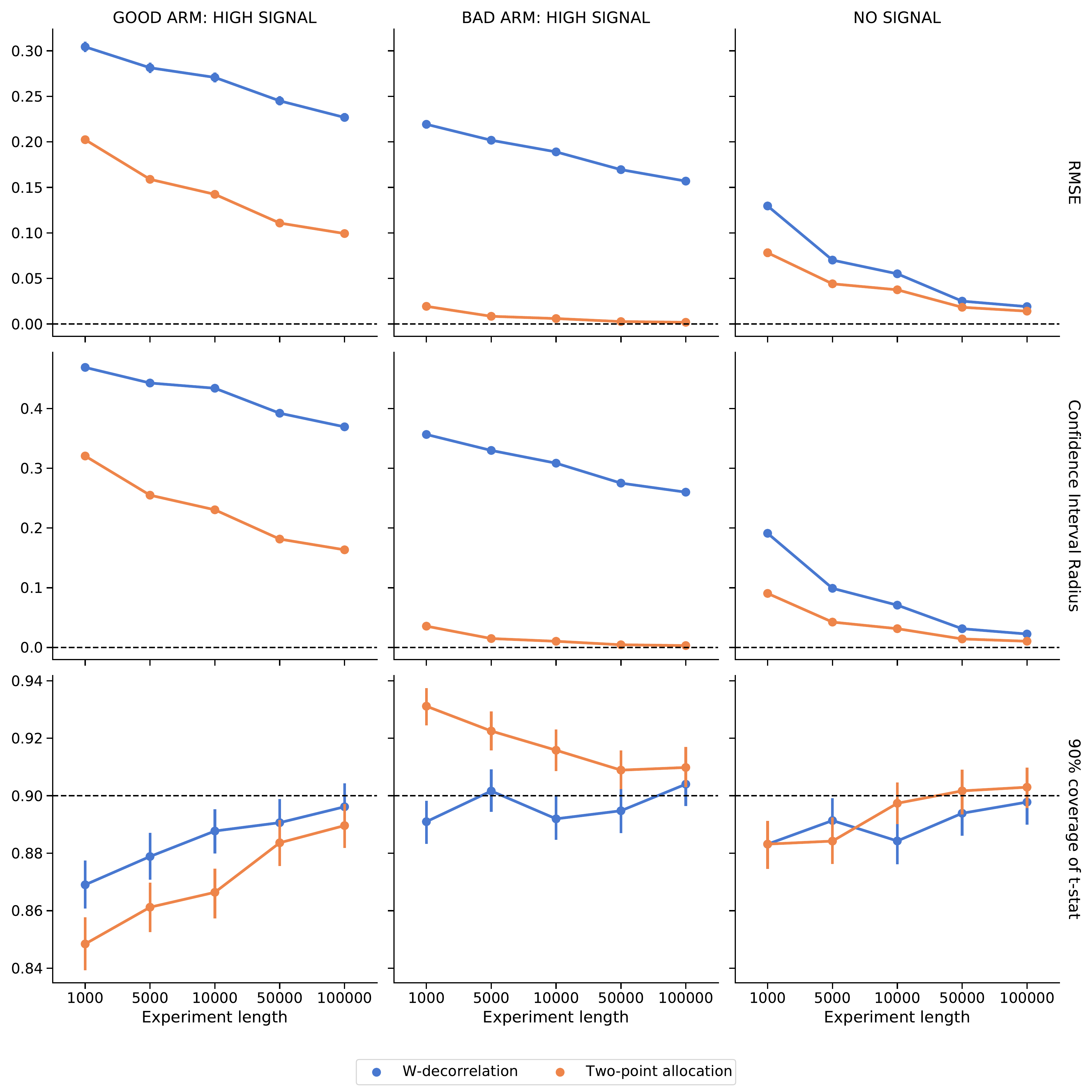}
  \caption{Comparison between the adaptively weighted estimator \eqref{eq:aw} with two-point allocation evaluation weights \eqref{eq:two_point_allocation} against W-decorrelation.}
  \label{fig:compare_W}
\end{figure}

\subsection{$\lambda_t$ path}
\label{sec:lambdapath}

Figure \ref{fig:lambda_20000} shows how the fraction of variance allocated to the current observation changes for the good and bad arms across simulation designs by our \emph{two-point allocation scheme}. On the top row, we see that the value of $\lambda_t$ decays over time, indicating the earlier observations receive more weight. On the bottom row, we see that once it becomes clearer that the arm is optimal, the allocation rate becomes roughly constant.

\begin{figure}[H]
  \centering
  \includegraphics[width=0.7\textwidth]{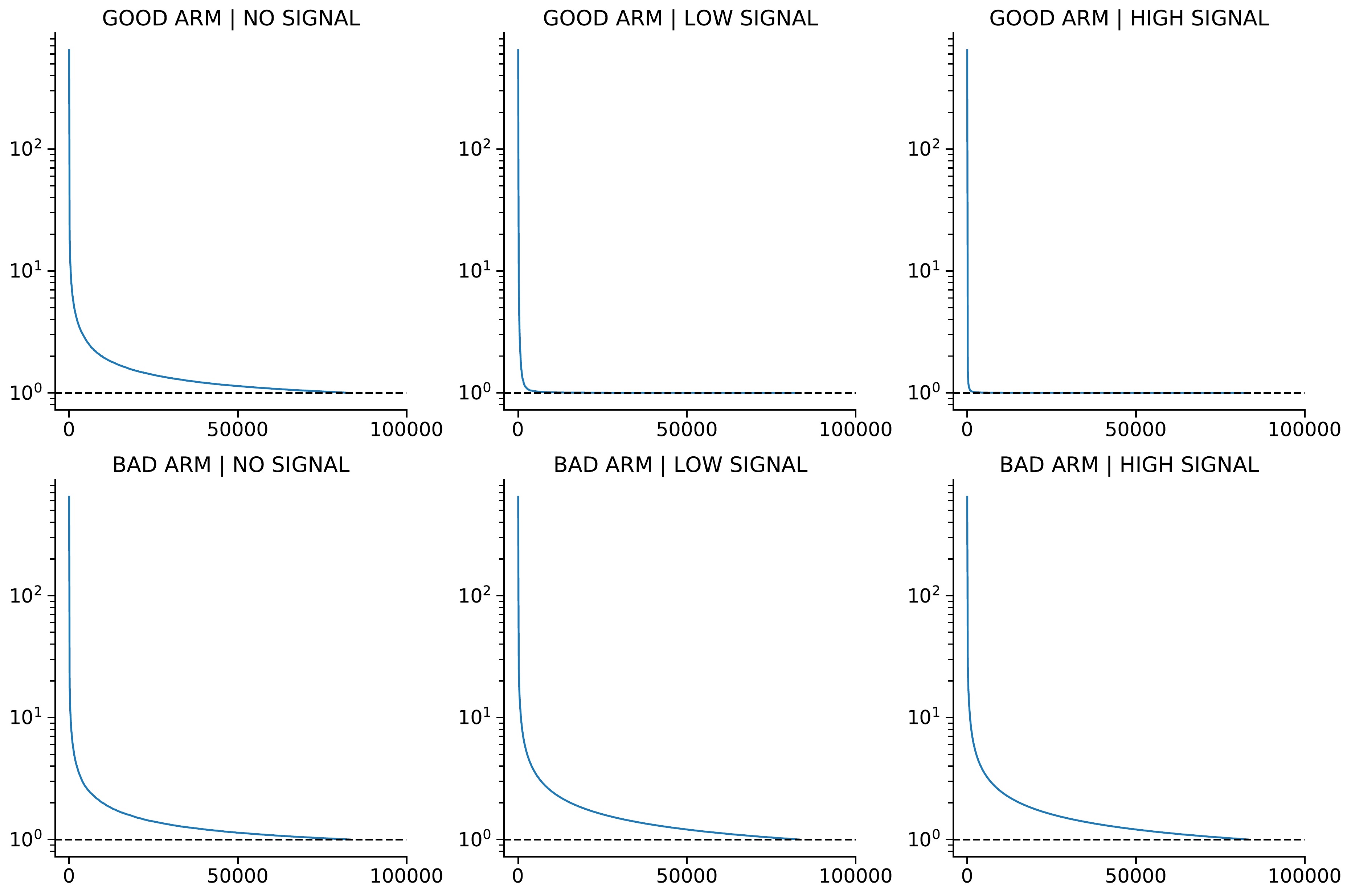}
  \caption{Values of $(T - t)\lambda_t$ under the two-point allocation specification \eqref{eq:two_point_allocation}.}
  \label{fig:lambda_20000}
\end{figure}

\subsection{Introduction example, revisited.}
\label{sec:intro_example_revisited} 

Here we return to the example described in the introduction to confirm our method yields an asymptotically normal test statistic in that setting. To recap, we have two identical arms with unit normal rewards. In the first half of the experiment, we pull both arms with equal probability; in the second half, we pull the arm with highest sample mean at $T/2$ with 90\% probability.

The top row of Figure \ref{fig:intro_example_complete} augments Figure \ref{fig:example} by including the adaptively-weighted estimator \eqref{eq:aw} with constant allocation weights \eqref{eq:const}. Overall, the adaptively-weighted estimator is more thin-tailed than the inverse propensity-weighted estimator and less biased than the sample average. 

The bottom row shows the studentized distribution of the same set of estimators. That is, each estimate is centered at the true value (zero) and divided by the square-root of its variance estimate. For the sample mean and adaptively-weighted estimator, the variance estimates are the same as used in Section \ref{sec:sims}; for the inverse propensity-weighted estimator, we use the estimate $\smash{\hV^{IPW}(w) := T^{-2} \sum_{t=1}^T (\hGamma_t^{IPW}(w) - \hQ^{IPW}_T(w))^2}$. As guaranteed by Theorem \ref{theo:arm_value_clt}, the adaptively-weighted estimator has a centered and approximately normal distribution. The sample mean and the inverse-propensity weighted estimator do not have normal asymptotic distributions, although the deviation from normality is relatively small. These deviations suggest that self-normalization (i.e., dividing by an estimate of the standard deviation) by itself is not sufficient to establish asymptotic normality, and our evaluation weights are indeed playing a role. 

\begin{figure}[H]
  \centering
  \includegraphics[width=\textwidth]{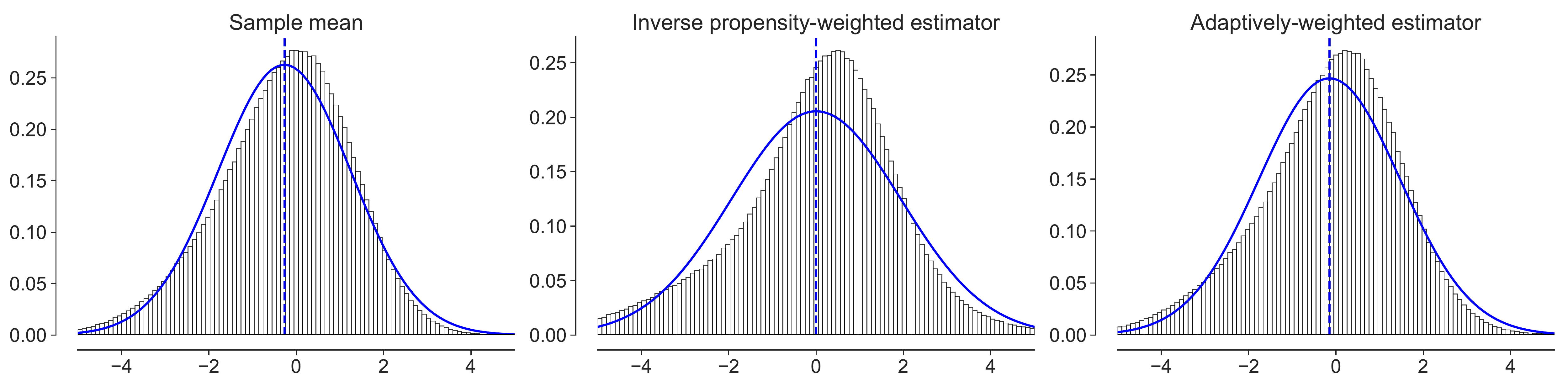}
  \includegraphics[width=\textwidth]{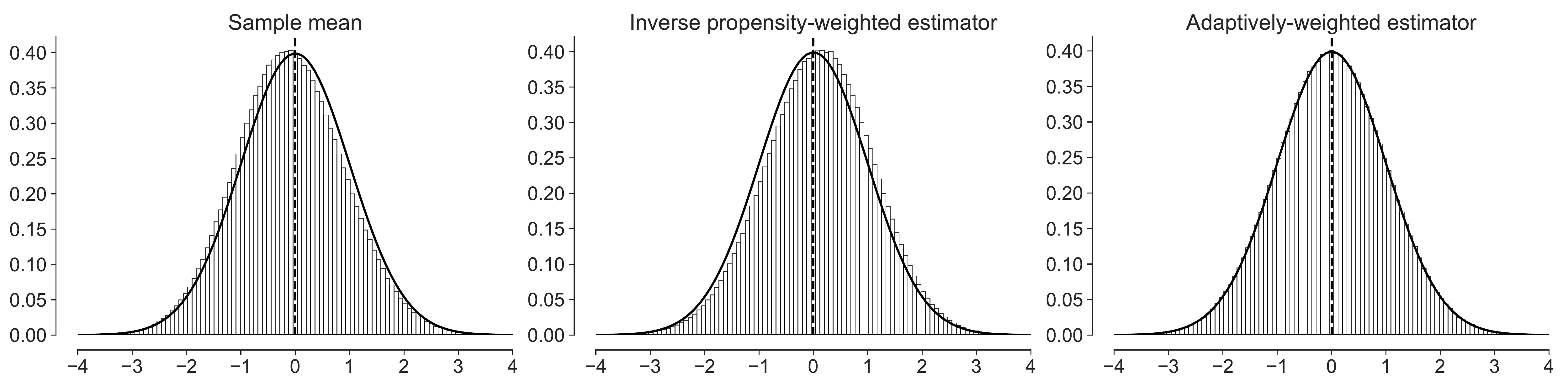}
  \caption{Distribution of the estimates of the sample mean \eqref{eq:avg}, inverse-propensity weighted mean \eqref{eq:ipw}, and adaptively-weighted estimator \eqref{eq:aw} with constant-allocation weights \eqref{eq:const}. Data-generating process was described in the introduction. All numbers are aggregated over 1,000,000 replications. \textit{Top}: Distribution of estimators for $T = 10^6$ across simulations, scaled by $\sqrt{T}$ for visualization. Solid blue line is the normal curve that matches the first two moments of the distribution. Dashed blue line denotes the empirical mean across simulations. \textit{Bottom}: Distribution of the ``studentized'' counterparts of each estimator in the top row. Solid line is the standard normal distribution. Dashed line is empirical average of each estimator (not of their studentized statistic) across simulations.}
  \label{fig:intro_example_complete}
\end{figure}

\clearpage

\section{Limit theorems for adaptively weighted unbiased scores}
\label{sec:general_clt}

In this section, we will prove more general versions of Theorems~\ref{theo:arm_value_clt}-\ref{theo:contrast_clt} 
from the body of the paper. 

\subsection{Setting} 

In Section \ref{sec:scoring}, we alluded to the fact that scoring rules can be used to construct estimators for a wide variety of estimands. 
We consider a setting in which there is a sequence of independent and identically distributed real-valued potential outcomes $(Y_t(w))_{w \in \mathcal W}$ 
for treatments in an arbitrary set $\mathcal W$. At each time-step we observe a realized treatment $W_t$ and the corresponding outcome $Y_t := Y_t(W_t)$, where the
distribution of $W_t$ is a known function of the history $H^{t-1} := \{ (Y_s,W_s) : s < t \}$. 
Letting $m(w) := \E[Y(w)]$ be the mean potential outcome at a given level of treatment, our goal is to estimate $\psi(m)$ where $\psi$ is a linear functional that is
continuous on the space $L_2(P_{t-1})$ of $H^{t-1}$-conditionally square integrable functions for all $t$.
This continuity property implies the existence of a unique \emph{Riesz representer} of $\psi$ on each space, 
i.e., a function $\gamma(\cdot; H^{t-1}) \in L_2(P_{t-1})$ satisfying
\begin{equation}
\label{eq:riesz}
\E[ \gamma(W_t; H^{t-1}) f(W_t) \mid H^{t-1}] = \psi(f) \ \text{ for all }\ f \in L_2(P_{t-1}).
\end{equation}
Using this Riesz representer, we define an unbiased scoring rule
\begin{equation}
 \label{eq:general_scoring_rule}
 \hGamma_t = \psi(\hat{m}) + \gamma(W_t; H^{t-1})\p{ Y_t - \hat{m}(W_t; H^{t-1})}
\end{equation}
in terms of an estimate $\hat m(\cdot; H^{t-1})$ of $m$ based on the history.
This framework generalizes the one used in \cite{chernozhukov2016locally}, \cite{hirshberg2017augmented}, and references therein
to the non-iid setting we consider here, in which \smash{$\norm{f}_{L_2(P_{t-1})}^2 = \E[f(W_t)^2 \mid H^{t-1}]$}, the space \smash{$L_2(P_{t-1})$}
is the set of functions with \smash{$\norm{f}_{L_2(P_{t-1})} < \infty$}, and continuity is the existence of a constant $c_{t-1}$
for which \smash{$\abs{\psi(f)} \le c_{t-1} \norm{f}_{L_2(P_{t-1})}$} for all functions $f \in L_2(P_{t-1})$.

Theorems \ref{theo:general_clt}, \ref{theo:general_variance_stabilizing_clt} and  \ref{theo:general_contrast} below concern this general setting. Theorems \ref{theo:arm_value_clt}, \ref{theo:variance_stabilizing_clt} and \ref{theo:contrast_clt} in the main section of the paper are special cases of these results.  In Section \ref{sec:cases-of-general} below, we explain in more detail how to map the general results back to the cases of interest in the paper.

As we did before, we economize on notation by writing subscripts instead of explicitly conditioning on the history so that $\gamma_t(\cdot) \equiv \gamma(\cdot; H^{t-1})$, $m_t \equiv m(\cdot; H^{t-1})$, and so on. Expectations of a random variable $X$ conditional on the history up to (and including) some period $t-1$ will
be denoted by $\EE[t-1]{X}$. Also,  whenever it does not lead to confusion we will write $\EE[t-1]{f}$ to mean $\EE[t-1]{f(W_t)}$ for functions of the treatment. Finally, we avoid some parens clutter by writing $\EE[t-1]{X}^p$ to mean $(\EE[t-1]{X})^p$.

\begin{assu}[Infinite sampling]
\label{assu:general_infinite_sampling}
The evaluation weights $h_t$ satisfy
\begin{align}
  \label{eq:general_infinite_sampling}
  \frac
    { \p{ \sum_{t=1}^T h_t }^2}
    { \EE{\sum_{t=1}^T h_t^2 \gamma_t^2}  } 
  &\xrightarrow[T \to \infty]{p} \infty.
\end{align}
\end{assu}

\begin{assu}[Variance convergence]
\label{assu:general_variance_convergence}
The evaluation weights $h_t$ satisfy, for some $p > 1$,
\begin{align}
  \label{eq:general_variance_convergence}
  \frac
    {\sum_{t=1}^T h_t^2 \EE[t-1]{\gamma_t^2} }   
    {\EE{ \sum_{t=1}^T h_t^2 \gamma_t^2 } }   
  &\xrightarrow[T \to \infty]{L_p} 1 \text{ for } p > 1.
\end{align}
\end{assu}

\begin{assu}[Bounded moments]
\label{assu:general_lyapunov}
The evaluation weights $h_t$ satisfy, for some $\delta > 0$,
\begin{align}
\label{eq:general_lyapunov}
\frac
  {\sum_{t=1}^T h_t^{2 + \delta} \EE[t-1]{|\gamma_t|^{2+\delta}} }   
  {\EE{ \sum_{t=1}^T h_t^2 \gamma_t^2}^{1 + \delta/2} }
&\xrightarrow[T \to \infty]{p} 0.
\end{align}
\end{assu}

\begin{theo}[General CLT for adaptively weighted scores]
\label{theo:general_clt}
In the setting described above, suppose that the variance of $Y_t(w)$ is bounded from above and away from zero and that $\EE{|Y_t(w)|^{2+\delta}}$ 
is bounded from above for some $\delta > 0$, with all bounds uniform in $w$; that $\gamma_t$ satisfies $\E_{t-1}\gamma_t^2 > b$ for every $t$ for some $b > 0$;
and that $\hm_t$ is a uniformly bounded sequence of estimators of the conditional mean function $m$ 
that converges to some function $m_\infty$ in the sense that \smash{$||\hm_t - m_\infty||_{L_\infty(P_{t-1})} \xrightarrow[t \to \infty]{a.s} 0$.}
Let $h_t$ denote \smash{$H^{t-1}$}-measurable non-negative weights satisfying Assumptions \ref{assu:general_infinite_sampling}, \ref{assu:general_variance_convergence} and \ref{assu:general_lyapunov}, and suppose that either
\begin{equation}
  \label{eq:general_mu_alternative}
  ||\hm - m||_{L_\infty(P_{t-1})} \xrightarrow[t \to \infty]{a.s} 0 
   \quad \text{or}  \quad \EE[t-1]{\gamma_t^2}  
  \xrightarrow[t \to \infty]{a.s.} \bar{\gamma}^2_{\infty} \in (0, \infty].
\end{equation}

Then for $\hGamma_t$ defined in \eqref{eq:general_scoring_rule}, the adaptively weighted estimator 
\begin{align}
  \label{eq:psihat}
  \hpsi_T = \frac{\sum_{t=1}^T h_t (\hGamma_t - \psi(m)) }{ \sum_{t=1}^T h_t }
\end{align}
converges to $\psi(m)$ in probability and the following studentized statistic is asymptotically normal:
  \begin{align}
    \label{eq:studentized_psihat}
    \frac{\hpsi_T - \psi(m)}{\hV_T^{\frac{1}{2}}} \xrightarrow[T \to \infty]{d} \mathcal{N}(0, 1),
    \quad \text{where} \quad
    \hV_T := \frac{\sum_{t=1}^T h_t^2 \left( \hGamma_t - \hpsi_T \right)^2}{\left( \sum_{t=1}^T h_t \right)^2}.
  \end{align}
\end{theo}

\begin{theo}[Construction of variance-stabilizing weights]
\label{theo:general_variance_stabilizing_clt}
In the setting of Theorem \ref{theo:general_clt}, suppose that the sequence of Riesz representers $\gamma_t$ satisfy
\begin{equation}
  \label{eq:variance_growth_rate}
  \frac{\EE[t-1]{|\gamma_t|^{2+\delta}}}
       {\EE[t-1]{\gamma_t^2}^{2+\delta}} \leq C 
  \qquad\text{and}\qquad
    \EE[t-1]{\gamma_t^2} \leq C't^{\alpha} \,\quad \text{for all }t
\end{equation}
for some positive $\delta$, nonnegative $\alpha \in [0, \delta/(2 + \delta))$, and positive constants $C, C'$. Then the variance-stabilizing weights defined by the recursion
\begin{equation}
  \label{eq:gamma}
  h_t^2 \EE[t-1]{\gamma_t^2} = \p{1 - \sum_{s=1}^{t-1} h_s^2 \EE[s-1]{\gamma_s^2} } \lambda_t
\end{equation}
will satisfy the infinite sampling \eqref{eq:general_infinite_sampling}, variance convergence \eqref{eq:general_variance_convergence} and Lyapunov \eqref{eq:general_lyapunov} conditions if the allocation rate satisfies $\lambda_t < 1$ for $t < T$, $\lambda_T = 1$, and for some positive constant~$C''$,
\begin{align}
\label{eq:general_allocation_rate_bounds}
  \frac{1}{1 + T - t} 
  \leq \lambda_t \leq C'' \, \frac{\EE[t-1]{\gamma_t^2}^{-1}}{t^{-\alpha} + T^{1-\alpha} - t^{1-\alpha}}. 
\end{align}
\end{theo}

\begin{theo}[CLT for adaptively-weighted unbiased score differences] \label{theo:general_contrast}
In the setting above, let $\psi_1, \psi_2$ be two linear functionals satisfying the continuity property described above for every $t$. Let $\gamma_{t, 1}, \gamma_{t, 2}$ be sequences of Riesz representers for $\psi_1$ and $\psi_2$, with 
  $\EE[t-1]{\gamma_{t, j}^2} > b > 0$ for every $t$ and $j$. For $j \in \{1,2\}$, let $\hm_{t, j}$ be a sequence of uniformly bounded estimators of $m$ such that $||\hm_{t, j} - m_{\infty, j}||_{L_\infty(P_{t-1})} \xrightarrow[t \to \infty]{a.s} 0$ for a fixed function $m_{\infty, j}$. In addition, let $\hGamma_{t, j}$ be a sequence of scores of the form
  \begin{align}
    \hGamma_{t, j} = \psi_j(\hat{m}_{t, j}) + \gamma_{t, j} (Y_t - \hm_{t, j}(W_t))
  \end{align}
  and let $h_{t, j}$ denote $H^{t-1}$-measurable non-negative weights. Suppose that each pair of sequences $h_{t, j}, \gamma_{t, j}$ satisfies Assumptions \ref{assu:general_infinite_sampling}, \ref{assu:general_variance_convergence} and \ref{assu:general_lyapunov}. And suppose that, in terms of these sequences and $\hV_{T,j}$ defined below, the following conditions are satisfied:
  \begin{align}
   \label{eq:general_variance_estimate_convergence}
   \hV_{T, 1} \,\big/\, \hV_{T, 2} 
   &\xrightarrow[T \to \infty]{p} r \in [0, \infty] 
   \\ 
  \label{eq:general_covariance_convergence}
   \frac{ \sum_{t=1}^T h_{t,1} h_{t, 2} \EE[t-1]{\gamma_{t,1} \gamma_{t,2}}}
        { \EE{ \sum_{t=1}^T h_{t,1}^2\gamma_{t, 1}^2 }^{1/2} \EE{ \sum_{t=1}^T h_{t,2}^2\gamma_{t, 2}^2 }^{1/2} }
   &\xrightarrow[T \to \infty]{p} 0
   \\ 
    \label{eq:general_contrast_alternative}
   \text{Either }\  
    ||\hm_{t, j} - m||_{L_\infty(P_{t-1})} &\xrightarrow[t \to \infty]{a.s} 0
   \quad \text{or} \quad
      \EE[t-1]{\gamma_{t, j}^2} \xrightarrow[t \to \infty]{a.s.} \infty 
   \qquad \text{for at least one }j.
 \end{align}
  Then the vector of studentized statistics defined by \eqref{eq:studentized_psihat} for each parameter is asymptotically jointly normal with identity covariance matrix.  Moreover, the estimator $\hpsi_{T, 1} - \hpsi_{T, 2}$ defined below is a consistent
  estimator of $\psi_1(m) - \psi_2(m)$, 
    \begin{align}
      \label{eq:generalized_delta}
      \hpsi_{T,1} - \hpsi_{T,2}
        := 
         \frac{ \sum_{t=1}^T h_{t, 1} \hGamma_{t, 1} }
              { \sum_{t=1}^T h_{t, 1} } 
              -
        \frac{ \sum_{t=1}^T h_{t, 2} \hGamma_{t, 2} }
             { \sum_{t=1}^T h_{t, 2} },
    \end{align}
    and the following studentized statistic is asymptotically standard normal:
    \begin{align}
      \label{eq:generalized_studentized_contrast}
      \frac{ \hpsi_{T, 1} - \hpsi_{T, 2} - (\psi_1(m) - \psi_2(m))}
            { \p{ \hV_{T, 1} + \hV_{T, 2} }^{1/2}  }
            \xrightarrow[T \to \infty]{d} \mathcal{N}(0, 1)
            \quad \text{where} \quad
            \hV_{T, j} := \frac{\sum_{t=1}^T h_{t, j}^2 \left( \hGamma_{t, j} - \hpsi_{T, j} \right)^2}{\left( \sum_{t=1}^T h_{t, j} \right)^2}.
    \end{align}
\end{theo}

\subsection{Specializing general results} 
\label{sec:cases-of-general}

Let's see how to apply the theorems above to the estimands discussed in the main section.

\paragraph{Arm-specific value} For the arm-value estimand considered in Theorem~\ref{theo:arm_value_clt} the functional is $\psi(f) = f(w)$, which evaluates the conditional mean function of the outcome at a specific treatment level, i.e., $\psi(m) = m(w)$. Conditional on the history up to period $t-1$, the Riesz representer of this functional is $\gamma_t = \ind{\cdot = w}/e_t(w)$. To confirm this, note that for any square-integrable function $f$,
$$\EE[t-1]{\gamma_t(W_{t})f(W_{t})} = \EE[t-1]{ \frac{\ind{W_t = w}}{e_t(w)} f(w)} = \frac{f(w)}{e_t(w)}\EE[t-1]{\ind{W_t =w}} = f(w).$$ 
Plugging this functional and Riesz representer into the general scoring rule \eqref{eq:general_scoring_rule}, we get $\hGamma_t^{AIPW}$ that was defined in \eqref{eq:aipw}. 

Theorem \ref{theo:arm_value_clt} in the main section established sufficient conditions for consistency and asymptotic normality of the adaptively weighted estimator $\hQ_T^h$ that uses this scoring rule. This result is a special case of Theorem \ref{theo:general_clt}, since conditions \eqref{eq:general_infinite_sampling}, \eqref{eq:general_variance_convergence}, \eqref{eq:general_lyapunov}, \eqref{eq:general_lyapunov} specialize to conditions \eqref{eq:infinite_sampling}, \eqref{eq:variance_convergence}, \eqref{eq:lyapunov} and \eqref{eq:e_mu_alternative} by substituting $\EE[t-1]{\gamma_t^2} = 1/e_t(w)$ and $\bar{\gamma}_\infty^2 = 1/e_{\infty}$.

Theorem \ref{theo:variance_stabilizing_clt} in the main section defined a class of variance-stabilizing schemes $h_t$ that satisfy conditions for consistency and asymptotic normality of $\hQ_T^h$. This result was a special version of Theorem \ref{theo:general_variance_stabilizing_clt}. To see why, note that if we specialize $\EE[t-1]{\gamma_t^2} = 1/e_{t}(w)$ the definition of a general variance-stabilizing weight \eqref{eq:gamma} reduces to variance-stabilizing weights for the arm value \eqref{eq:variance_stabilizing}. Moreover, substituting $\EE[t-1]{|\gamma_t|^{2+\delta}} = 1/e_t^{1+\delta}(w)$, conditions \eqref{eq:prop_LB} and \eqref{eq:allocation_rate_bounds} correspond to 
the first half of \eqref{eq:variance_growth_rate} and condition \eqref{eq:general_allocation_rate_bounds}. 
The sole remaining assumption of the general theorem is the second part of condition \eqref{eq:variance_growth_rate},
which is satisfied for arbitrary $\delta > 1$:
  \begin{align*}
    \frac{\EE[t-1]{|\gamma_t|^{2+\delta}}}{\EE[t-1]{\gamma_t^2}^{2+\delta}} 
      = \frac{1/e_t^{1 + \delta}}{1/e_{t}^{2+\delta}} \leq 1.
  \end{align*}

\paragraph{Treatment effects} As we mentioned in Section \ref{sec:contrasts}, a straightforward approach for constructing an estimator of the average difference in values $\psi(m) := m(w_1) - m(w_2)$ is to use the score $\smash{\hGamma_{t}^{DIFF} := \hGamma_t(w_1) - \hGamma_t(w_2)}$. Naturally, since each scoring rule $\hGamma_t(w_j)$ is unbiased for its target arm value, their difference is unbiased for the difference in arm values. Theorem \ref{theo:general_clt} covers the conditions for consistency and asymptotic normality of this estimator. In this case, the Riesz representer of the functional is $\gamma_t = \ind{\cdot = w_1}/e_t(w_1) - \ind{\cdot = w_2}/e_t(w_2)$, and we note that $\EE[t-1]{\gamma_t^2} = 1/e_t(w_1) + 1/e_t(w_2)$.

We also proposed an alternative estimator \eqref{eq:delta} in which each unbiased scoring rule received its own sequence of evaluation weights. Asymptotic normality of this estimator was established in Theorem \ref{theo:contrast_clt}, which is a special instance of Theorem \ref{theo:general_contrast}. In this case, the functionals are $\psi_{1}(m) = \EE{Y_t(w_1)}$ and $\psi_{2}(m) = \EE{Y_t(w_2)}$, and their Riesz representers are $\gamma_{t,1} = \ind{\cdot = w_1}/e_t(w_1)$ and $\gamma_{t,2} = \ind{\cdot = w_2}/e_t(w_2)$. Substituting these specializations, conditions
\eqref{eq:general_variance_estimate_convergence} and \eqref{eq:general_contrast_alternative} in Theorem \ref{theo:general_contrast} simplify to their counterparts \eqref{eq:evaluation_weight_convergence} and \eqref{eq:e_mu_alternative_constrast} in Theorem \ref{theo:contrast_clt}. And condition \eqref{eq:general_covariance_convergence} is always satisfied:
$$\gamma_{t,1}\gamma_{t,2} = \frac{\ind{W_t = w_1}}{e_t(w_1)}\frac{\ind{W_t = w_2}}{e_t(w_2)} = 0.$$

\paragraph{Average Derivative} The average derivative example in footnote \ref{foot:average_derivative} is also covered by Theorem \ref{theo:general_clt}. 
In this case the linear functional is $\psi(m) = \int m'(w)f(w)dw$. To confirm the Riesz representer $\gamma_t(w) = -f'(w)/f_t(w)$, note that for any differentiable function $g$, 
\begin{equation}
\begin{aligned}
    \EE[t-1]{\gamma_t(W_t)g(W_t)} 
    &= \int \gamma_t(w)g(w) f_t(w) dw 
    &= -\int \frac{f'(w)}{f_t(w)} g(w) f_t(w) dw 
    &= \int g'(w) f(w) dw,
\end{aligned}
\end{equation}
where the last equality holds via by integration by parts assuming that the treatment assignment distribution $f(w)$ is zero at $\pm\infty$.

\section{Proofs}
\label{sec:proofs}

%
%
\subsection{Proof of Theorem \ref{theo:general_clt}}

\paragraph{Overview} We will follow a familiar proof technique that involves three steps. We will first prove a CLT for the following \emph{auxiliary martingale difference sequence} 
  \begin{align}
    \label{eq:auxiliary_mds}
    \xi_{T,t} := \frac{ h_t(\hGamma_t- \psi(m)) }{ \left( \sum_{t=1}^T \EE{ h_t^2(\hGamma_t- \psi(m))^2 } \right)^{\frac{1}{2}} }
  \end{align}

by proving that it satisfies the conditions in Proposition \ref{prop:martingale_clt} below. Next, we will show consistency of our estimator $\hpsi_T$ for the true value $\psi(m)$. Finally, we will show asymptotic normality of the studentized version of $\hpsi_T$.

\begin{prop}[Martingale CLT]
  \label{prop:martingale_clt}
  \cite[e.g.,][]{helland1982central}
  Let $\{\xi_{T,t}, \ff_{T,t} \}_{t=1}^T$ be a square-integrable martingale difference sequence triangular array. Then the sum $\sum_{t=1}^T \xi_{T,t}$ will be asymptotically normally distributed, $\sum_{t=1}^T \xi_{T,t} \xrightarrow[]{d} \nn\p{0, \, 1}$ if the following conditions hold.\footnote{The relevant reference within \cite{helland1982central} is Theorem 2.5(a), with condition conditional Lindeberg condition (2.5) replaced with the stronger Lyapunov condition (2.9).}
  \begin{align}
      \label{eq:theorem_lyapunov}
      \sum_{t=1}^T \EE[t-1]{ |\xi_{T,t}|^{2 + \delta} } &\xrightarrow[T \to \infty]{p} 0 \quad \text{for some $\delta > 0$} \qquad &&\text{[Conditional Lyapunov condition]}  \\
      \label{eq:theorem_variance_convergence}
      \sum_{t=1}^T \EE[t-1]{ \xi_{T,t}^2 } &\xrightarrow[T \to \infty]{p} 1 \qquad &&\text{[Variance convergence]}
  \end{align}
\end{prop}

\begin{prop}[Convergence of quadratic variation]
  \label{prop:convergence_quadratic_variation}
  \cite[e.g][Theorem 2.23]{hall2014martingale}
  Let $\{\xi_{T,t}, \ff_{T,t} \}_{t=1}^T$ be a square-integrable martingale difference sequence triangular array. Define
  \begin{align}
      V_{T,t}^2 := \sum_{s=1}^t \EE[s-1]{\xi_{T,s}^2} \qquad
      U_{T,t}^2 := \sum_{s=1}^t \xi_{T,s}^2.
  \end{align}
  Suppose that the conditional Lyapunov condition \eqref{eq:theorem_lyapunov} holds, and also that the conditional variances are tight, that is $\sup_T \PP{ V_{T,T}^2 > \lambda } \to 0$ \text{as} $\lambda \to \infty$.
  Then
  \begin{align}
      \sup_t \abs{ V_{T,t}^2 - U_{T,t}^2 } \xrightarrow[T \to \infty]{p} 0.
  \end{align}
\end{prop}

Our goal is to prove that the auxiliary martingale difference sequence defined in \eqref{eq:auxiliary_mds} satisfies the two martingale central limit conditions stated in Proposition \ref{prop:martingale_clt}. We begin by proving a few lemmas.

\begin{lemm}[Convergence of weighted arrays]
  \label{lemm:weighted_average}
  Let $a_{T,t}$ be a triangular sequence of nonnegative weight vectors satisfying $\plim_{T \to \infty} \max_{1 \leq t \leq T} a_{T,t} = 0$ and $\plim_{T \to \infty} \sum_{t=1}^T a_{T,t} \leq C$ for some constant $C$. Also, let $x_t$ be a sequence of random variables satisfying $x_t \xrightarrow[t \to \infty]{a.s} 0$ bounded by $B$. Then $\sum_{t=1}^T a_{T,t}x_{t} \xrightarrow[T \to \infty]{p} 0$.
\end{lemm}
\begin{proof}[Proof (Lemma \ref{lemm:weighted_average})]
  Without loss of generality we will prove this claim for $C = 1$. Fix some positive $\epsilon$ and $\delta$. Almost-sure convergence of $x_t$ to zero implies that we can find a large enough $t_{\epsilon}$ such that $|x_t| < \epsilon$ with probability at least $1 - \delta$. Let $A(\epsilon)$ denote the event in which this happens. Under that event,
  \begin{equation}
    \begin{aligned}
      \sum_{t=1}^T a_{T,t}x_t 
        &\leq B \sum_{t \leq t_\epsilon}a_{T, t} + \sum_{t > t_{\epsilon}} a_{T, t}\epsilon \\
        &\leq Bt_{\epsilon} \max_{1 \leq t \leq T} a_{T, t} + \epsilon.
    \end{aligned}    
  \end{equation}
  Therefore,
  \begin{equation}  
      \begin{aligned}
         \PP{ \sum_{t=1}^T a_{T,t} x_t > 2\epsilon}
         &=  \PP{ \left\{ \sum_{t=1}^T a_{T,t} x_t > 2\epsilon \right\} \cap A(\epsilon) }
          + \PP{ \left\{ \sum_{t=1}^T a_{T,t} x_t > 2\epsilon \right\} \cap A^{c}(\epsilon) } \\
         &\leq \PP{ Bt_{\epsilon} \max_{1 \leq t \leq T} a_{T, t} + \epsilon > 2\epsilon } 
              + \PP{ A^{c}(\epsilon)} \\
         &\leq \PP{ B t_{\epsilon} \max_{1 \leq t \leq T} a_{T, t} > \epsilon } + \delta 
      \end{aligned}
  \end{equation}
  
  Since by assumption $\max_{1 \leq t \leq T} a_{T,t} \toprob 0$ as $T \to \infty$ and $\delta$ is arbitrarily small, the claim follows.
\end{proof}

\begin{lemm}[Negligible weights] In the setting of Theorem \eqref{theo:general_clt}, if the Lyapunov \eqref{eq:general_lyapunov} assumption is satisfied,
  \label{lemm:negligible_weights}
  \begin{equation}
    \label{eq:negligible_weights}
    \frac{\max_{1 \leq t \leq T} h_t^2 \EE[t-1]{\gamma_t^2} }
         {\EE{\sum_{t=1}^T h_t^2 \gamma_t^2 } } \xrightarrow[T \to \infty]{p} 0
  \end{equation}
\end{lemm}

\begin{proof}[Proof (Lemma \ref{lemm:negligible_weights})] We will show that the $2+\delta$ absolute power of \eqref{eq:negligible_weights} converges to zero. 
  \begin{equation}
      \abs{
      \frac
        { \max_{1 \leq t \leq T} h_t^2 \EE[t-1]{ \gamma_t^2 }}
        {  \EE{  \sum_{t=1}^T h_t^2\gamma_t^2 } }
      }^{1 + \delta/2}
      \leq
      \frac
        { \max_{1 \leq t \leq T} h_t^{2+\delta} |\EE[t-1]{ \gamma_t^2 }|^{1 + \delta/2} }
        { \EE{  \sum_{t=1}^T h_t^2\gamma_t^2 }^{1 + \delta/2} }
        \leq 
        \frac
          { \sum_{t=1}^T h_t^{2+\delta} \EE[t-1]{ |\gamma_t|^{2 + \delta} }}
          { \EE{ \sum_{t=1}^T  h_t^2\gamma_t^2 }^{1 + \delta/2} }.
  \end{equation}
  Exchanging the maximum with the mapping $f(x) = |x|^{1+\delta/2}$ in the first inequality is justified because the latter is an increasing function of $x$, 
  and the second inequality follows from Jensen's inequality because $f$ is a convex function. The right side converges to zero by \eqref{eq:general_lyapunov}. Because the absolute power function $f$ is continuous, convergence of the absolute power implies that the original sequence in \eqref{eq:negligible_weights} also converges to zero, completing the proof. 
\end{proof}

\begin{lemm}{Behavior of score variance}. \label{lemm:score_variance} The variance of the unbiased score $\hGamma_t$ conditional on past data simplifies to
  \begin{equation} 
    \EE[t-1]{(\hGamma_t - \psi(m))^2} = \Var[t-1]{\gamma_t (\hm_t - m)} + \EE[t-1]{\gamma_t^2 \sigma^2} \ \text{ where }\ \sigma^2(w) = \Var{Y_t(w)}.
  \end{equation}
  Moreover, there exists a positive constant $c$ such that
  \begin{equation}
       \frac{1}{c} \EE[t-1]{\gamma_t^2} \leq \EE[t-1]{(\hGamma_t - \psi(m))^2} \leq c \EE[t-1]{\gamma_t^2}.
  \end{equation}
\end{lemm}
\begin{proof}[Proof (Lemma \ref{lemm:score_variance})]
  Recalling that $\psi(m) = \E_{t-1} \hGamma_t$ and adding and subtracting $\gamma_t m$ from \eqref{eq:general_scoring_rule},
  \begin{equation}
    \label{eq:score_var}
    \begin{aligned}
      \EE[t-1]{ \p{ \hGamma_t - \E_{t-1} \hGamma_t}^2} &= \EE[t-1]{ \p{ \hGamma_t - \psi(m)} ^2} \\
        &= \EE[t-1]{ \p{ \psi(\hm_t - m) - \gamma_t(\hm_t - m) + \gamma_t(Y_t - m) }^2 } \\
        &= \EE[t-1]{ \psi(\hm_t - m)^2} + \EE[t-1]{ \gamma_t^2(\hm_t - m)^2 } + \EE[t-1]{ \gamma_t^2(Y_t - m)^2 } \\
        &- 2\EE[t-1]{ \psi(\hm_t - m) \gamma_t (\hm_t - m) } 
         + 2\EE[t-1]{ \psi(\hm_t - m) \gamma_t(Y_t - m) } \\
        &- 2\EE[t-1]{  \gamma_t^2 (\hm_t - m) (Y_t - m) }.
    \end{aligned}
  \end{equation}

  We can simplify this further. The first term is simply $\psi(\hm_t - m)$ as it is $H^{t-1}$-measurable. The third term equals $\EE[t-1]{\gamma_t^2 (Y_t - m)^2} = \EE[t-1]{\gamma_t^2 \sigma^2}$. In the fourth term, pull out the $H^{t-1}$-measurable factor $\psi(\hm_t - m)$, and apply the Riesz representation property \eqref{eq:riesz} to get 
 $\psi(\hm_t - m)\EE[t-1]{\gamma_t (\hm_t - m)} = \psi(\hm_t - m)^2$. Apply the law of iterated expectations on the fifth and sixth terms conditioning on both $H^{t-1}$ and $W_t$ to show that those terms vanish. We are left with
  \begin{equation}
    \begin{aligned}
      &= \psi(\hm_t - m)^2 + \EE[t-1]{\gamma_t^2 (\hm_t - m)^2} + \EE[t-1]{\gamma_t \sigma^2} - 2\psi(\hm_t - m)^2.
    \end{aligned}
  \end{equation}

  Using the Riesz representation property \eqref{eq:riesz} rewrite $\psi(\hm_t - m)^2$ as $\EE[t-1]{\gamma_t(\hm_t - m)}^2$, 
  \begin{equation}
    \begin{aligned}
      &= -\EE[t-1]{\gamma_t(\hm_t - m)}^2 + \EE[t-1]{\gamma_t^2 (\hm_t - m)^2} + \EE[t-1]{\gamma_t \sigma^2} \\
      &=  \Var[t-1]{\gamma_t (\hm_t - m)} + \EE[t-1]{\gamma_t^2 \sigma^2}. 
    \end{aligned}
  \end{equation}
  This completes the first part of the proof. The lower bound in second part is an obvious consequence, as by assumption $\sigma^2$ is bounded below.
  The upper bound follows similarly because $\sigma^2$, $\hm_t$, and $m$ are uniformly bounded.
\end{proof}

In possession of these results, we are ready to start the proof of Theorem \ref{theo:general_clt}.

\begin{proof}[Proof (Theorem \ref{theo:general_clt})] We must show that the Lyapunov and variance convergence conditions in \ref{prop:martingale_clt} are satisfied for the auxiliary martingale \eqref{eq:auxiliary_mds}.
  
\paragraph{Lyapunov} We would like to show that for some $\delta > 0$,
\begin{align}
  \label{eq:lyap}
  \sum_{t=1}^T \EE{ |\xi_{T,t}|^{2 + \delta} }
    &= \frac{ \sum_{t=1}^T  h_t^{2 + \delta}  \EE[t-1]{ |\hGamma_t - \psi(m)|^{2 + \delta}} }
            { \EE{ \sum_{t=1}^T h_t^2 (\hGamma_t - \psi(m))^2}^{1 + \delta/2} }  \xrightarrow[T \to \infty]{p} 0.
\end{align}

Let's begin by studying the numerator. By adding and subtracting $\gamma_t m$,
\begin{equation}
  \label{eq:lyap_numerator}
  \begin{aligned}
    || \hGamma_t - \psi(m) ||_{L_{2+\delta}(P_{t-1})}
      &=   \Norm{ \psi(\hm_t - m) - \gamma_t (\hm_t - m) + \gamma_t (Y_t - m) }_{L_{2+\delta}(P_{t-1})} \\
      &\leq   \Norm{ \psi(\hm_t - m) }_{L_{2+\delta}(P_{t-1})}   +
         \Norm{  \gamma_t (\hm_t - m) }_{L_{2+\delta}(P_{t-1})}   +
         \Norm{ \gamma_t (Y_t - m) }_{L_{2+\delta}(P_{t-1})},
  \end{aligned}
\end{equation}
where we used Minkowski's inequality in the second line. Let's show that all three terms in \eqref{eq:lyap_numerator} are bounded by a multiple of $\norm{\gamma_t}_{L_{2+\delta}(P_{t-1})}$. The first term is bounded by the second: via the Riesz representation property \eqref{eq:riesz}, it is 
equal to $\norm{\E_{t-1}[\gamma_t (\hm_t - m) ]}_{L_{2+\delta}(P_{t-1})}$, and by Jensen's inequality, this is bounded by $\norm{\gamma_t (\hm_t - m) }_{L_{2+\delta}(P_{t-1})}$.
We use H\"older's inequality to bound the second term: $\norm{ \gamma_t ( \hm_t - m) }_{L_{2+\delta}(P_{t-1})} \le \norm{\hm_t - m}_{L_{\infty}(P_{t-1})}\norm{ \gamma_t}_{L_{2+\delta}(P_{t-1})}$ and $\hm_t - m$ is uniformly bounded. Using the law of iterated expectation on the third term, $\E_{t-1}[|\gamma_t|^{2+\delta} \E_{t-1}[ |Y_t(W_t) - m(W_t)|^{2+\delta} | W_t]] \lesssim \E_{t-1} |\gamma_t|^{2+\delta}$ because $\E[\abs{Y_t(w)}^{2+\delta}]$ is uniformly bounded. Collecting these results we have that $|| \hGamma_t - \psi(m) ||_{L_{2+\delta}(P_{t-1})} \lesssim \norm{\gamma_t}_{L_{2+\delta}(P_{t-1})}$, or equivalently that $\E_{t-1}[|\hGamma_{t} - \psi(m)|^{2+\delta}] \lesssim \E_{t-1}[|\gamma_t|^{2+\delta}]$. Substituting this bound for the terms in the numerator of \eqref{eq:lyap} and bounding the terms in the denominator from below using Lemma \ref{lemm:score_variance},
\begin{align}
  \sum_{t=1}^T \EE{ |\xi_{T,t}|^{2 + \delta} }
    &\lesssim
        \frac{ \sum_{t=1}^T  h_t^{2 + \delta}  \EE[t-1]{ |\gamma_t|^{2 + \delta}} }
             { \EE{ \sum_{t=1}^T  h_t^2 \gamma_t^2}^{1 + \delta/2} },  
\end{align}
which goes to zero in probability by the Lyapunov assumption \eqref{eq:general_lyapunov}.

\paragraph{Variance convergence} 
Our goal is to show that the following sum of conditional variances converges to one.
\begin{align}
  \label{eq:varconv}
  \sum_{t=1}^T \EE{ \xi_{T,t}^2 }
    &= \frac{ \sum_{t=1}^T \EE[t-1]{ h_t^2 (\hGamma_t - \psi(m))^2} }
            { \EE{ \sum_{t=1}^T h_t^2 (\hGamma_t - \psi(m))^2} }  \toprob 1
\end{align}

The ratio in \eqref{eq:varconv} is $Z_T/\E Z_T$ for $Z_T = \sum_{t=1}^T \E_{t-1}{ h_t^2 (\hGamma_t - \psi(m))^2} $. We characterize the asymptotic behavior of $Z_T$ by decomposing it into $Z_T = A_T + R_T$. The dominant term is $A_T = \sum_{t=1}^T h_t^2 \EE[t-1]{\gamma_t^2 f}$ for a time-invariant function $f$ that is bounded from above and also bounded away from zero. The magnitude of the remainder term $R_T$ is bounded by $\sum_{t=1}^T h_t^2 \E_{t-1} [\gamma_t^2]x_{t}$, where $x_{t}$ is a bounded and almost surely vanishing sequence. Both $f$ and $x_t$ will be defined later. For now, let's assume that we have completed this characterization of $Z_T$, and see how this fact allows us to conclude the proof. 

First, let's show that $R_T = o_p(A_T)$. Because $f$ is bounded away from zero, $A_T \gtrsim \sum_t h_t^2 \E_{t-1}[\gamma_t^2]$. Therefore, 
the magnitude of the ratio $R_T/A_T$ is bounded by 
\begin{equation}
  \label{eq:vanishing_term}
  \frac{\sum_{t=1}^T h_t^2 \EE[t-1]{\gamma_t^{2}} x_t}
           {\sum_{t=1}^T h_t^2 \EE[t-1]{\gamma_t^2} }.
\end{equation}
This ratio \eqref{eq:vanishing_term} is a weighted average of a bounded sequence $x_{t}$ converging almost-surely to zero in the sense of Lemma~\ref{lemm:weighted_average}, with weights that sum to one deterministically and are individually negligible via Lemma \ref{lemm:negligible_weights} and Assumption~\ref{assu:general_variance_convergence}. Hence by Lemma \ref{lemm:weighted_average} it is $o_p(1)$.

Next, we will show that $\smash{\E R_T = o(\E A_T)}$. By H\"older's inequality,
\begin{equation}
  \label{eq:denominator_remainder}
   \frac{\E R_T}{\E A_T} = \EE{ \frac{R_T}{A_T} \frac{A_T}{\E A_T} } \le \norm*{\frac{R_T}{A_T}}_{L_{q}(P)} \norm*{\frac{A_T}{\E A_T}}_{L_{p}(P)}
\end{equation}
for some H\"older conjugates $p,q$. The second factor is bounded for some $p>1$ as shown in \eqref{eq:ratio_convergence} below, and the first converges to zero for any finite $q$ and therefore the conjugate of $p$, because $R_T/A_T$ converges to zero in probability and is uniformly bounded.

It follows from \eqref{eq:vanishing_term} and \eqref{eq:denominator_remainder} that
\begin{equation}
  \frac{Z_T}{\E Z_T} = \frac{(1+o_p(1))A_T}{(1+o(1))\E A_T} = \frac{ A_T }{ \E A_T } + o_p\p{\frac{A_T}{\E A_T}}.
\end{equation}
The main term $A_T/\E A_T$ converges to one in $p$-norm for some $p > 1$ and hence in probability, as 
\begin{align}
  \label{eq:ratio_convergence}
  \EE{ \abs{ \frac{\sum_{t=1}^T h_t^2 \EE[t-1]{\gamma_t^2 f}}{ \EE{\sum_{t=1}^T h_t^2 \gamma_t^2 f} }  - 1  }^p }
  \leq \abs{\frac{\sup_w f(w)}{\inf_w f(w)}}^{p} \EE{ \abs{
        \frac{\sum_{t=1}^T h_t^2 \EE[t-1]{\gamma_t^2} - \EE{\sum_{t=1}^T h_t^2 \gamma_t^2}}
             { \EE{\sum_{t=1}^T h_t^2 \gamma_t^2} }  
        }^p },
\end{align}
recalling that $f$ is bounded by above and away from zero, and by the variance condition \eqref{eq:general_variance_convergence} there exists some number $p > 1$ for which the second factor in \eqref{eq:ratio_convergence} converges to zero. We conclude that $Z_T/\E Z_T = 1 + o_p(1)$.

All that remains is to show that indeed $Z_T$ decomposes into $A_T$ and $R_T$ as claimed. Expanding the expression for the conditional variance of $\hGamma_t$ from Lemma \ref{lemm:score_variance},
\begin{equation}
  \begin{aligned}
    \label{eq:score_variance_expanded}
    &\EE[t-1]{(\hGamma_t - \psi(m))^2} \\
    &= \EE[t-1]{\gamma_t^2 \sigma^2} + \Var[t-1]{\gamma_t (\hm_t - m)} \\
    &= \EE[t-1]{\gamma_t^2 \sigma^2} \\
    &+ \EE[t-1]{\gamma_t^2 (m_\infty - m)^2}
    + 2 \EE[t-1]{\gamma_t^2 (\hm_t - m_\infty)(m_\infty - m)}
    + \EE[t-1]{\gamma_t^2 (\hm_t - m_\infty)^2} \\
    &- \EE[t-1]{\gamma_t (m_\infty - m)}^2 
    - 2 \EE[t-1]{\gamma_t (\hm_t - m_\infty)}\EE[t-1]{\gamma_t (m_\infty - m)}
    - \EE[t-1]{\gamma_t (\hm_t- m_\infty)}^2
    \end{aligned}  
\end{equation}

Consider the three terms in the left column of \eqref{eq:score_variance_expanded}, i.e., those that do not involve $\hm_t$. If $\hm_t$ is consistent, the sum of these three terms is $\EE[t-1]{\gamma_t f}$ for $f(w) = \sigma^2(w)$, as the terms involving $m_\infty - m$ in \eqref{eq:score_variance_expanded} are zero.
Thus, it suffices to show that the other terms make up the remainder $R_T$,  i.e., that the magnitude of their sum after multiplying by $h_t^2$ and summing across time periods is bounded by $\sum_{t=1}^T h_t^2 \E_{t-1} [\gamma_t^2]x_{t}$, where $x_{t}$ is a bounded and almost surely vanishing sequence. It suffices to show that each term has this property individually, as by the triangle inequality $x_t = \sum_{j=1}^k x_{t,j}$  will be such a sequence if $x_{t,1} \ldots x_{t,k}$ are.

Irrespective of whether $\hm_t$ is consistent, these other terms are remainder terms. 
By H\"older's inequality, the magnitude of the third and fourth terms are bounded 
by $\E_{t-1}[\gamma_t^2] \norm{\hm_t - m_\infty}_{L_\infty(P_{t-1})} \norm{m_\infty - m}_{L_\infty(P_{t-1})}$ 
and $\E_{t-1}[\gamma_t^2] \norm{\hm_t - m_\infty}_{L_\infty(P_{t-1})}^2$ respectively,
and $x_{t,1} := \norm{\hm_t - m_\infty}_{L_\infty(P_{t-1})} \norm{m_\infty - m}_{L_\infty(P_{t-1})}$
and $x_{t,2} := \norm{\hm_t - m_\infty}_{L_\infty(P_{t-1})}^2$ are bounded and go to zero almost surely given
our assumptions that $\hm_t$ converges and $\hm_t$ and $m$ are bounded.
Furthermore, the sum of the third term and the corresponding term on the line below it, and likewise
the fourth and the one below it, are smaller in magnitude than the third and fourth terms themselves --- these sums are conditional variances and covariances where the third and fourth terms themselves are conditional mean squares --- so they satisfy the same bounds.

Now suppose $\hm_t$ is not consistent. In this case, we take $f(w) = \sigma^2(w) +  (m_\infty(w) - m(w))^2 + \psi(m_{\infty} - m)^2(\bar{\gamma}^2_{\infty})^{-1}$, recalling that $\bar{\gamma}_{\infty}^2 > 0$ by assumption. The difference between the sum of the three terms in the left column of \eqref{eq:score_variance_expanded} and $\E_{t-1}[\gamma_t^2 f]$ is 
$\EE[t-1]{\gamma_t (m_\infty- m)}^2 -  \psi(m_{\infty} - m)^2(\bar{\gamma}^2_{\infty})^{-1}\E_{t-1}[\gamma_t^2]$. We must show that the contribution of this difference is another remainder term. Applying the Riesz representation \eqref{eq:riesz} to the first term, the difference simplifies to $\psi(m_{\infty} - m)^2 (1 -  (\bar{\gamma}^2_{\infty})^{-1} \E_{t-1}{\gamma_t^2})$. By assumption \eqref{eq:general_mu_alternative}, the ratio $(\bar{\gamma}^2_{\infty})^{-1} \E_{t-1}{\gamma_t^2}$ converges to one almost surely, so the difference converges to zero almost surely. To confirm this is a remainder term note that $1 = \EE[t-1]{\gamma_t^2} (\EE[t-1]{\gamma_t^2})^{-1} \lesssim \EE[t-1]{\gamma_t^2}$, since $\E_{t-1}[\gamma_t^2]$ is bounded away from zero, so our difference is bounded by $\E_{t-1}[\gamma_t^2] x_{t,3}$, for $x_{t,3} = \psi(m_\infty - m) (1 - (\bar{\gamma}^2_{\infty})^{-1} \E_{t-1}{\gamma_t^2})$.

This concludes the proof of asymptotic normality for the auxiliary martingale \eqref{eq:auxiliary_mds}.

\paragraph{Consistency of $\hpsi_T$}

Starting from the definition of $\hpsi_T$ in \eqref{eq:aw} and rearranging terms,
\begin{align}
  \label{eq:consistency_prod}
  \hpsi_T - \psi(m)
  =  \frac{ \sum_{t=1}^T h_t (\hGamma_t - \psi(m))}
          { \sum_{t=1}^T h_t } 
  =  \frac{ \sum_{t=1}^T h_t (\hGamma_t - \psi(m))}
          { \sum_{t=1}^T \EE{ h_t^2 (\hGamma_t - \psi(m))^2 }^{1/2} } 
     \frac{ \EE{ \sum_{t=1}^T  h_t^2 (\hGamma_t - \psi(m))^2 }^{1/2}}{ \sum_{t=1}^T h_t  } 
\end{align}

The first factor is simply $\sum_{t=1}^T \xi_{T, t}$, which have just shown to be asymptotically normal and hence bounded in probability. By Lemma \ref{lemm:score_variance}, the second factor is upper bounded by an expression that goes to zero by the infinite sampling assumption \eqref{eq:general_infinite_sampling},
\begin{align}
  \frac{ \EE{  \sum_{t=1}^T h_t^2 (\hGamma_t - \psi(m))^2 }^{1/2}}
       { \sum_{t=1}^T h_t  }
    \lesssim  \frac{  \EE{ \sum_{t=1}^T h_t^2 \gamma_t^2 }^{1/2}}
                   { \sum_{t=1}^T h_t  }.
\end{align}

 This implies that \eqref{eq:consistency_prod} is the product of a factor that is bounded in probability times one that converges to zero in probability, and therefore it must also vanish in probability.

\paragraph{Asymptotic normality}

Rewrite the studentized statistic \eqref{eq:clt} as
\begin{equation}
    \begin{aligned}
        \frac{\hpsi_T - \psi(m)}{ \hV_T^{\frac{1}{2}} }
        &=
            { \frac{\sum_{t=1}^T h_t(\hGamma_t - \psi(m))}{\sum_{t=1}^T h_t} } \,\bigg/\,
            { \left( \frac{\sum_{t=1}^T h_t^2(\hGamma_t - \hpsi_T)^2}{ (\sum_{t=1}^T h_t)^2 } \right)^{1/2} }  \\
        &= \frac
            { \sum_{t=1}^T h_t(\hGamma_t - \psi(m))}
            { \left(  \sum_{t=1}^T h_t^2(\hGamma_t - \hpsi_T)^2 \right)^{1/2} } \\
        &=
          \left( \sum_{t=1}^T  \xi_{T,t} \right)
          \times
          \left(
            \frac
              { \EE{  \sum_{t=1}^T h_t^2 (\hGamma_t - \psi(m))^2 } }
              { \sum_{t=1}^T h_t^2 \left(\hGamma_t- \hpsi_T \right)^2 }
        \right)^{1/2}.
    \end{aligned}
\end{equation}

Since we have previously proved that the factor on the left is asymptotically normal, the remainder of the proof will show that the factor on the right converges in probability to $1$. 
Inverting the ratio and linearizing it around $\psi(m)$,
\begin{equation}
    \label{eq:asymp_linear}
    \frac
        { \sum_{t=1}^T h_t^2 \left(\hGamma_t- \hpsi_T \right)^2 }
        { \EE{  \sum_{t=1}^T h_t^2 (\hGamma_t - \psi(m))^2 } }
    = \frac
        { \sum_{t=1}^T h_t^2 (\hGamma_t - \psi(m))^2 }
        {  \EE{  \sum_{t=1}^Th_t^2 (\hGamma_t - \psi(m))^2 } }
    - 2(\hpsi_T - \psi(m))  \frac
        { \sum_{t=1}^T h_t^2 \left(\hGamma_t- \tilde{\psi}_T \right) }
        { \EE{  \sum_{t=1}^T h_t^2 (\hGamma_t - \psi(m))^2 } },
\end{equation}
with $\tilde{\psi}_T$ between $\psi(m)$ and $\hpsi_T$ by the mean value theorem. As we'll show next, the first term converges to one in probability, and the remaining term on the right converges to zero  in probability.

Let's show that the first term in \eqref{eq:asymp_linear} converges to one. It is the \emph{quadratic variation} $\sum_{t=1} \xi_{T, t}^2$ of the auxiliary martingale difference sequence defined in \eqref{eq:auxiliary_mds}. We showed before that the conditional variance $\sum_{t=1}^T \E_{t-1}\xi_{T, t}^2$ converges to one in probability, and by Proposition \ref{prop:convergence_quadratic_variation} the conditional variance and the quadratic variation converge in probability, so it follows that the quadratic variation converges to one in probability.

To complete the proof let's show that the second term in \eqref{eq:asymp_linear} converges to zero. Because $\hpsi_T$ converges to $\psi(m)$ in probability, so does $\tilde{\psi}_T$. Hence we just need to show that the rightmost factor in \eqref{eq:asymp_linear} is bounded in probability. Adding and subtracting $\sum_{t=1}^T h_t^2 \psi(m)$ from its numerator, 
\begin{equation}
  \label{eq:asymp_linear2}  
  \begin{aligned}
    \frac
        { \sum_{t=1}^T h_t^2 \left(\hGamma_t- \tilde{\psi}_T \right) }
        {  \EE{ \sum_{t=1}^T h_t^2 (\hGamma_t - \psi(m))^2 } }
    &= \frac
        { \sum_{t=1}^T h_t^2 (\hGamma_t - \psi(m)) }
        { \EE{  \sum_{t=1}^T h_t^2 (\hGamma_t - \psi(m))^2 } }
    + \left(\psi(m)- \tilde{\psi}_T \right)
      \frac
        { \sum_{t=1}^T h_t^2 }
        { \EE{  \sum_{t=1}^T h_t^2 (\hGamma_t - \psi(m))^2 } }.
  \end{aligned}
\end{equation}
Via H\"older's inequality, the first term on the right-hand side of \eqref{eq:asymp_linear2} is bounded as
\begin{equation}
  \label{eq:asymp_linear3}
  \begin{aligned}
    \frac
        { \sum_{t=1}^T h_t^2 (\hGamma_t - \psi(m)) }
        {  \EE{ \sum_{t=1}^Th_t^2 (\hGamma_t - \psi(m))^2 } }
    &\leq
      \frac
        { \max_{1 \leq t \leq T} h_t}
        { \EE{  \sum_{t=1}^Th_t^2 (\hGamma_t - \psi(m))^2 }^{1``/2} }
      \abs{\frac
        { \sum_{t=1}^T h_t (\hGamma_t - \psi(m)) }
        { \EE{  \sum_{t=1}^T h_t^2 (\hGamma_t - \psi(m))^2 }^{1/2} }},
  \end{aligned}
\end{equation}
which is the product of two factors. By Lemma \ref{lemm:score_variance}, the square of the first factor is upper bounded~as
\begin{equation}
    \frac
      { \max_{1 \leq t \leq T} h_t^2 }
      {  \EE{ \sum_{t=1}^Th_t^2 (\hGamma_t - \psi(m))^2 } }
    \lesssim
    \frac
      { \max_{1 \leq t \leq T} h_t^2 \EE[t-1]{ \gamma_t^2 }}
      { \EE{  \sum_{t=1}^T h_t^2\gamma_t^2 }^{1 + \delta/2} }
\end{equation}
which goes to zero by Lemma \ref{lemm:negligible_weights}. The second factor in \eqref{eq:asymp_linear3} is the absolute value of something that we have proven to converge to a standard normal random variable and is therefore bounded in probability. Therefore, the first term on the right-hand side of \eqref{eq:asymp_linear2} converges to zero in probability.

The second term on the right-hand side of \eqref{eq:asymp_linear2} is also vanishing, as it is the product of
the vanishing factor $\tilde{\psi}_T-\psi(m)$ and a factor we can show to be no larger than constant order using our variance convergence assumption \eqref{eq:general_variance_convergence} and a uniform lower bound on $\EE[t-1]{\gamma_t^2}$,
\begin{align}
  \frac
    { \sum_{t=1}^T h_t^2 \EE[t-1]{\gamma_t^2}/\EE[t-1]{\gamma_t^2} }
    { \EE{ \sum_{t=1}^T  h_t^2 (\hGamma_t - \psi(m))^2  } }
  \lesssim
  \frac
    { \sum_{t=1}^T h_t^2 \EE[t-1]{\gamma_t^2} }
    { \EE{ \sum_{t=1}^T  h_t^2\gamma_t^2} } \xrightarrow[T \to \infty]{L_{p}} 1.
\end{align}

 This concludes the proof of Theorem \ref{theo:general_clt}.
\end{proof}


\subsection{General CLT for adaptively-weighted unbiased score differences}

\begin{proof}[Proof (Theorem \ref{theo:general_contrast})]

We will first show that for any $(t_{1}, t_{2}) \in \mathbb{R}^2$, the linear combination of auxiliary martingales $t_{1}\xi_{t, 1} + t_{2}\xi_{t, 2}$ defined in \eqref{eq:auxiliary_mds} is asymptotically normal. By the Cram\'{e}r-Wold theorem, this will imply that the pair $(\xi_{t, 1}, \xi_{t, 2})$ is asymptotically normal as well.

\paragraph{Variance convergence} We must show that the following converges to a constant,
\begin{equation}
  \label{eq:linear_combo_variance_square}
  \sum_{t=1}^T \EE[t-1]{(t_{1}\xi_{t, 1} + t_{2}\xi_{t, 2})^2} = 
    t_{1}^2 \sum_{t=1}^T \EE[t-1]{\xi_{t, 1}^2}
    + t_{2}^2  \sum_{t=1}^T \EE[t-1]{\xi_{t, 2}^2} 
    + 2 t_{1}t_{2} \sum_{t=1}^T \EE[t-1]{\xi_{t, 1}\xi_{t, 2}}.
\end{equation}

By the results proved in Theorem \ref{theo:arm_value_clt}, the asymptotic variance of the auxiliary martingales is one, so first two terms in \eqref{eq:linear_combo_variance_square} converge to $t_{1}^2$ and $t_{2}^2$.  We will prove that the cross terms converge to zero in probability, that is,
\begin{equation}
  \label{eq:contrast_cross_term}
  \abs{ \sum_{t=1}^T \EE[t-1]{\xi_{t, 1}\xi_{t, 2}} }
      = \frac{\sum_{t=1}^T h_{t, 1}h_{t, 2} \abs{ \EE[t-1]{ (\hGamma_{t, 1} - \psi_1(m))(\hGamma_{t, 2} - \psi_2(m)) }}}
          { \EE{ \sum_{t=1}^T h_{t, 1}^2 (\hGamma_{t, 1} - \psi_1(m))^2 }^{1/2} 
            \EE{ \sum_{t=1}^T h_{t, 2}^2 (\hGamma_{t, 2} - \psi_2(m))^2 }^{1/2}
          } 
      \xrightarrow[T \to \infty]{p} 0.
\end{equation}

Begin by expanding the covariance of scores,
\begin{equation}  
  \label{eq:score_covariance_first}
  \begin{aligned}
    &\EE[t-1]{(\hGamma_{t, 1} - \psi_1(m))(\hGamma_{t, 2} - \psi_2(m))} \\
    &= \EE[t-1]{ (\psi_1(\hm_{t,1} - m) + \gamma_{t,1}(Y_t - \hm_{t,1})) 
                 (\psi_2(\hm_{t,2} - m) + \gamma_{t,2}(Y_t - \hm_{t,2})} \\
    &= \EE[t-1]{\psi_1(\hm_{t,1} - m)\psi_2(\hm_{t,2} - m)}
      + \EE[t-1]{\psi_1(\hm_{t,1} - m)\gamma_{t,2}(Y_t - \hm_{t,2}) } \\
    &+ \EE[t-1]{\psi_2(\hm_{t,2} - m)\gamma_{t,1}(Y_t - \hm_{t,1}) }
      +  \EE[t-1]{\gamma_{t,1}\gamma_{t,2} (Y_t - \hm_{t,1})(Y_t - \hm_{t,2})  }.
\end{aligned}
\end{equation}

Since the $\psi_j(\hm_{t,j} - m)$ are $H^{t-1}$-measurable, they can be pulled out of the expectation; this is
\begin{equation}
  \label{eq:score_covariance_second}
  \begin{aligned}
    &= \psi_1(\hm_{t,1} - m)\psi_2(\hm_{t,2} - m)
      + \psi_1(\hm_{t,1} - m)\EE[t-1]{\gamma_{t,2}(Y_t - \hm_{t,2}) } \\
    &+ \psi_2(\hm_{t,2} - m)\EE[t-1]{\gamma_{t,1}(Y_t - \hm_{t,1}) }
      +  \EE[t-1]{\gamma_{t,1}\gamma_{t,2} (Y_t - \hm_{t,1})(Y_t - \hm_{t,2})  },
  \end{aligned}
\end{equation}

Noting that $\gamma_{t,j}(W_t)$ and $\hm_{t,j}(W_t)$ are measurable with respect to $(H^{t-1},W_t)$, by the law of iterated expectation $\EE[t-1]{\gamma_{t,j}(Y_t - \hm_{t,j}) } = \EE[t-1]{\gamma_{t,j}(m - \hm_{t,j}) }$. Moreover, by the Riesz representation property $\EE[t-1]{\gamma_{t,j}(m - \hm_{t,j}))} = \psi(m - \hm_{t,j})$. If we make these substitutions in \eqref{eq:score_covariance_second}, the first and second terms cancel out, leaving 
\begin{equation}
  \label{eq:score_covariance_third}
  \begin{aligned}
  &= -\psi_1(\hm_{t,1} - m)\psi_2(\hm_{t,2} - m)
    +  \EE[t-1]{\gamma_{t,1}\gamma_{t,2} (Y_t - \hm_{t,1})(Y_t - \hm_{t,2})  }.
  \end{aligned}
\end{equation}

Let's focus on the last term. Adding and subtracting $m$ from each $(Y_{t,j} - m)$ term, 
\begin{equation}
  \label{eq:score_covariance_cross_term}
  \begin{aligned}
  &\EE[t-1]{\gamma_{t,1}\gamma_{t,2} (Y_t - \hm_{t,1})(Y_t - \hm_{t,2})  } \\
  &= \EE[t-1]{\gamma_{t,1}\gamma_{t,2} (Y_t - m)^2 }
  + \EE[t-1]{\gamma_{t,1}\gamma_{t,2} (Y_t - m)(m - \hm_{t,1})  } \\
  &+ \EE[t-1]{\gamma_{t,1}\gamma_{t,2} (Y_t - m)(m - \hm_{t,2})  }
  + \EE[t-1]{\gamma_{t,1}\gamma_{t,2} (m - \hm_{t,1})(m - \hm_{t,2})  },
  \end{aligned}
\end{equation}
Applying the law of iterated expectations conditioning on $H^{t-1}$ and $W_{t}$ as above, the two middle terms vanish, and the first simplifies to $\E_{t-1}\gamma_{t,1}\gamma_{t,2}\sigma^2$. Plugging \eqref{eq:score_covariance_cross_term} back into \eqref{eq:score_covariance_third}, the covariance of scores \eqref{eq:score_covariance_first} becomes
\begin{equation}
  \label{eq:score_covariance_simplified}
  \begin{aligned}
    &\EE[t-1]{(\hGamma_{t, 1} - \psi_1(m))(\hGamma_{t, 2} - \psi_2(m))} \\
    &= \EE[t-1]{\gamma_{t,1}\gamma_{t,2}\sigma^2} 
     + \EE[t-1]{\gamma_{t,1}\gamma_{t,2} (\hm_{t,1} - m)(\hm_{t,2} - m )} 
     - \psi_1(\hm_{t,1} - m)\psi_2(\hm_{t,2} - m).
  \end{aligned}
\end{equation}

Now substitute \eqref{eq:score_covariance_simplified} into the numerator of \eqref{eq:contrast_cross_term} and bound it by three sums, each corresponding to a term in \eqref{eq:score_covariance_simplified}, via the triangle inequality. Furthermore, substitute the lower bound from Lemma \ref{lemm:score_variance}, $\EE[t-1]{h_{t, j}^2(\hGamma_{t, j} - \psi(m))^2} \gtrsim \EE[t-1]{h_{t, j}^2\gamma_{t, j}^2}$, for each such term in the denominator.
This gives a three-term upper bound on \eqref{eq:contrast_cross_term}, 
\begin{equation}
  \begin{aligned}
    \abs{ \sum_{t=1}^T \EE[t-1]{\xi_{t, 1}\xi_{t, 2}} }
        &\lesssim  
        \frac{\sum_{t=1}^T h_{t, 1}h_{t, 2} \EE[t-1]{\gamma_{t,1}\gamma_{t,2}\sigma^2} }{
          \EE{ \sum_{t=1}^Th_{t, 1}^2 \gamma_{t, 1}^2}^{1/2} 
          \EE{ \sum_{t=1}^Th_{t, 1}^2 \gamma_{t, 2}^2}^{1/2} } \\
        &+  \abs{\frac{\sum_{t=1}^T h_{t, 1}h_{t, 2} \EE[t-1]{\gamma_{t,1}\gamma_{t,2} (\hm_{t,1} - m)(\hm_{t,2} - m)} }{
            \EE{ \sum_{t=1}^Th_{t, 1}^2 \gamma_{t, 1}^2}^{1/2} 
            \EE{ \sum_{t=1}^Th_{t, 1}^2 \gamma_{t, 2}^2}^{1/2} }} \\
        &+ \abs{\frac{\sum_{t=1}^T h_{t, 1}h_{t, 2} \psi_1(\hm_{t,1} - m)\psi_2(\hm_{t,2} - m) }{
            \EE{ \sum_{t=1}^Th_{t, 1}^2 \gamma_{t, 1}^2}^{1/2} 
            \EE{ \sum_{t=1}^Th_{t, 1}^2 \gamma_{t, 2}^2}^{1/2} }}.
  \end{aligned}
\end{equation}
The first two terms in this bound go to zero in probability by assumption \eqref{eq:general_covariance_convergence} and H\"{o}lder's inequality, recalling that $\sigma^2$, $\hm_{t, j}$ and $m$ are bounded. By the Cauchy-Schwarz inequality, the last term is bounded by
\begin{equation}
  \label{eq:linear_combo_variance_cs}
  \begin{aligned}
      \p{\frac{ \sum_{t=1}^T h_{t, 1}^2 \psi_1(\hm_{t, 1} - m)^2 }
              {   \EE{ \sum_{t=1}^Th_{t, 1}^2 \gamma_{t, 1}^2} } }^{1/2}
              \p{\frac{ \sum_{t=1}^T h_{t, 2}^2 \psi_2(\hm_{t, 2} - m)^2 }
              {   \EE{ \sum_{t=1}^Th_{t, 2}^2 \gamma_{t, 2}^2} } }^{1/2}
  \end{aligned}
\end{equation}

Let's show that this product converges to zero in probability. Each factor can be written as
\begin{equation}
  \label{eq:linear_combo_factor1}
  \frac{ \sum_{t=1}^T h_{t, j}^2 \psi_j(\hm_{t, j} - m)^2 }
          {   \EE{ \sum_{t=1}^T h_{t, j}^2 \gamma_{t, j}^2} }
  = \frac{ \sum_{t=1}^T h_{t, j}^2 \EE[t-1]{\gamma_{t, j}^2} \p{\EE[t-1]{\gamma_{t, j}^2}^{-1}\psi_j(\hm_{t, j} - m)^2}}
          {   \EE{ \sum_{t=1}^T h_{t, j}^2 \gamma_{t, j}^2} },
\end{equation}

implying \eqref{eq:linear_combo_factor1} is a weighted average of $x_{t,j} := (\E_{t-1}\gamma_{t, j}^2)^{-1} \psi_j(\hm_{t, j} - m)^2$ with weights $a_{t,T,j} := (h_{t, j}^2 \E_{t-1}\gamma_{t, j}^2) / \sum_{t=1}^T \EE{h_{t, j}^2 \gamma_{t, j}^2}$. The variance convergence assumption \eqref{eq:general_variance_convergence} implies that the sum of these weights $a_{t,T,j}$ is bounded in probability, and Lemma \ref{lemm:negligible_weights} implies that the largest weight converges to zero in probability. So if in addition $x_{t,j}$ were to go to zero almost surely, Lemma \ref{lemm:weighted_average} would apply and \eqref{eq:linear_combo_factor1} would be $o_p(1)$. Let's show that this is true for at least one arm.

Suppose that the consistency-convergence assumption \eqref{eq:general_mu_alternative} is satisfied for arm $1$. Then 
\begin{equation}
  \label{eq:x_bound}
  \begin{aligned}
  x_{t,1} 
    &\lesssim 
      \frac{\psi_{1}(\hm_{t,1} - m_{\infty,1})^2}
             {\EE[t-1]{\gamma_{t, 1}^2}} 
    + \frac{\psi_{1}(m_\infty - m)^2}
           {\EE[t-1]{\gamma_{t, 1}^2}} \\
    &= \frac{\EE[t-1]{\gamma_{t, 1}(\hm_{t,1} - m_{\infty,1})}^2}
             {\EE[t-1]{\gamma_{t, 1}^2}} 
      + \frac{\psi_{1}(m_{\infty,1} - m)^2}
             {\EE[t-1]{\gamma_{t, 1}^2}} \\
    &\leq
        \EE[t-1]{(\hm_{t,1} - m_{\infty,1})^2}
      + \frac{\psi_{1}(m_{\infty,1} - m)^2}
             {\EE[t-1]{\gamma_{t, 1}^2}},
  \end{aligned}
\end{equation}
where in the first line we applied the quadratic inequality $(a + b)^2 \leq 2(a^2 + b^2)$ to the numerator, and in the second and third lines we applied to the left term first the Riesz representation property \eqref{eq:riesz} and then the Cauchy-Schwarz inequality. The resulting first term is bounded by $\norm{\hm_{t,1} - m_{\infty,1}}_{L_{\infty}(P_{t-1})}^2$, which goes to zero almost surely by assumption. Let's consider the second term. If the $\hm_{t}$ is consistent then its numerator is zero, and since its denominator is bounded away from zero we are done; otherwise, by \eqref{eq:general_mu_alternative} its denominator goes to infinity almost surely,
and its denominator is a nonzero constant. In either case the bound on $x_{t,1}$ goes to zero almost surely. 
Therefore, as we discussed in the previous paragraph, Lemma \ref{lemm:weighted_average} indeed applies for arm $1$ and \eqref{eq:linear_combo_factor1} is $o_p(1)$.

To deal with the remaining factor in \eqref{eq:linear_combo_variance_cs}, first apply the Riesz representation property to replace $\psi_{2}(\hm_{t,2} - m)$ by $\EE[t-1]{\gamma_{t,2}(\hm_{t,2} -m )}$, then apply Cauchy-Schwarz and use the fact that $\EE[t-1]{(\hm_{t,2} - m)^2}$ is bounded since  $\hm_t$ and $m$ are bounded:
\begin{equation}
  \label{eq:linear_combo_factor2}
  \frac{ \sum_{t=1}^T h_{t, 2}^2 \psi_2(\hm_{t, 2} - m)^2 }
          {   \EE{ \sum_{t=1}^T h_{t, j}^2 \gamma_{t, j}^2} }
  = \frac{ \sum_{t=1}^T h_{t, 2}^2 \EE[t-1]{\gamma_{t, 2}(\hm_{t,2} - m)}^2}
          {   \EE{ \sum_{t=1}^T h_{t, j}^2 \gamma_{t, j}^2} }
  \lesssim \frac{ \sum_{t=1}^T h_{t, 2}^2 \EE[t-1]{\gamma_{t, 2}^2}}
          {   \EE{ \sum_{t=1}^T h_{t, j}^2 \gamma_{t, j}^2} }.
\end{equation}
By the variance convergence condition \eqref{eq:general_variance_convergence}, \eqref{eq:linear_combo_factor2} is bounded in probability. Therefore, we have just shown that the product in \eqref{eq:linear_combo_variance_cs} has one $O_p(1)$ and one $o_p(1)$ factor, completing our argument that all terms in  \eqref{eq:x_bound} are negligible. This completes our proof of variance convergence.

\paragraph{Lyapunov condition} This follows immediately from previous results, since
\begin{equation}
  \sum_{t=1}^T \EE[t-1]{|t_{1}\xi_{t, 1} + t_{2}\xi_{t, 2}|^{2+\delta}} \lesssim
    |t_1|^{2+\delta} \sum_{t=1}^T \EE[t-1]{|\xi_{t, 1}|^{2+\delta}} 
   + |t_2|^{2+\delta} \sum_{t=1}^T \EE[t-1]{|\xi_{t, 2}|^{2+\delta}} 
\end{equation}
and we have shown that both sums on the right-hand side converge to zero asymptotically.

\paragraph{Joint normality} The previous two paragraphs have shown that any linear combination of auxiliary martingales $t_1 \xi_{T, 1} + t_2 \xi_{T, 2}$ satisfies a martingale central limit theorem with asymptotic variance $t_1^2 + t_2^2$. By the Cram\'{e}r-Wold theorem, this implies that the vector 
\begin{align}
  Z_T := 
    \p{
      \frac
        { \sum_{t=1}^T h_{t, 1}(\hGamma_{t, 1} - \psi_1(m))}
        {  \EE{ \sum_{t=1}^T h_{t, 1}^2(\hGamma_{t, 1} - \psi_1(m))^2}^{1/2} }, \quad
      \frac
        { \sum_{t=1}^T h_{t, 2}(\hGamma_{t, 2} - \psi_2(m))}
        { \EE{ \sum_{t=1}^T h_{t, 2}^2 (\hGamma_{t, 2} - \psi_2(m))^2}^{1/2} }
      }
\end{align}
converges in distribution to a jointly normal random variable $Z \sim \mathcal{N}(0, I)$ as $T \to \infty$.

Previous results have showed that the ratio between empirical and population variances converges to one. Therefore, by Slutsky's theorem, the vector of feasible studentized statistics
\begin{align}
    \label{eq:jointly_normal_empirical}
    \p{
      \frac
        { \sum_{t=1}^T h_{t, 1}(\hGamma_{t, 1} - \psi_1(m))}
        { \p{ \sum_{t=1}^T h_{t, 1}^2(\hGamma_{t, 1} - \psi_1(m))^2 }^{1/2} }, \quad
      \frac
        { \sum_{t=1}^T h_{t, 2}(\hGamma_{t, 2} - \psi_2(m))}
        { \p{ \sum_{t=1}^T h_{t, 2}^2 (\hGamma_{t, 2} - \psi_2(m))^2}^{1/2} }
      },
\end{align}
also converges in distribution to the same jointly normal random variable $Z \sim \mathcal{N}(0, I)$ as $T \to \infty$.

To get our estimator \eqref{eq:studentized_contrast}, premultiply the elements of the sequence in \eqref{eq:jointly_normal_empirical} by random vectors
\begin{align}
  \label{eq:nut}
  \nu_T := 
    \p{
      \frac{\hV_{T, 1}}{ \p{\hV_{T, 1} + \hV_{T, 2}}^{1/2}},
      \frac{ -\hV_{T, 2}}{ \p{\hV_{T, 1} + \hV_{T, 2`}}^{1/2}} 
      }
    = \p{
        \frac{1}{ \p{1 + \hV_{T, 2}/\hV_{T, 1}}^{1/2}},
        \frac{-1}{ \p{1 + \hV_{T, 1}/\hV_{T, 2}}^{1/2}},
        },
\end{align}
where $\hV_T$ was defined in \eqref{eq:studentized_psihat}. By assumption \eqref{eq:general_variance_estimate_convergence}, $\nu_T$ converges in probability to the vector 
\begin{align}
  \nu := 
    \p{
    \frac{1}{\p{1 + r^{-1}}^{1/2}},
    \frac{1}{\p{1 + r}^{1/2}}
    }.
\end{align}

To analyze the asymptotic distribution of $\nu_T'\xi_{T,t}$, first 
note that since $\nu_T$ converges to a constant, the pair $(\nu_T, Z_T)$ converges jointly in distribution. By Skorokhod's representation theorem there exists a sequence $(\tilde{\nu}_T, \tZ_T)$ whose elements are distributed as their corresponding elements on the sequence in $(\nu_T, Z_T)$, and in addition there exists a random variable $\tZ \sim \mathcal{N}(0, I)$ such that $\tZ_T \xrightarrow[T \to \infty]{a.s.} \tZ$, and $\tnu \xrightarrow[T \to \infty]{a.s.} \nu$. Therefore, $\langle \nu_T, Z_T \rangle$ has the same distribution as $\langle \tnu_T, \tZ_T \rangle$, which we decompose as
\begin{equation}
  \label{eq:z_nu}
  \langle \tnu_T, \tZ_T \rangle 
  = \langle \tnu_T - \nu, \tZ_T - \tZ \rangle 
  + \langle \nu, \tZ_T - \tZ \rangle 
  + \langle \tnu_T - \nu, \tZ \rangle 
  + \langle \nu, \tZ \rangle.
\end{equation}
By Slutsky's theorem, the first three terms in \eqref{eq:z_nu} vanish as $T \to \infty$ since $\tnu_T - \nu $ and $\tZ_T - Z$ go to zero almost-surely. The last term is normally distributed with unit variance since $\nu'\nu = 1$. 

\end{proof}

%
%

\subsection{Proof of Theorem \ref{theo:general_variance_stabilizing_clt}}

\begin{proof}[Proof (Theorem \ref{theo:general_variance_stabilizing_clt})]  

Let's show that an adaptively-weighted estimator with variance-stabilizing weights \eqref{eq:gamma} satisfies the infinite sampling \eqref{eq:general_infinite_sampling}, variance convergence \eqref{eq:general_variance_convergence} and Lyapunov \eqref{eq:general_lyapunov} conditions.

\paragraph{Variance convergence} For variance-stabilizing weights, the numerator of the variance converge condition \eqref{eq:general_variance_convergence}  is constant since, by the recursion \eqref{eq:gamma} and the fact that $\lambda_T = 1$,
\begin{equation}
  \label{eq:variances_sum_to_one}
  \begin{aligned}
    \sum_{t=1}^{T} h_t^2 \EE[t-1]{\gamma_t^2} 
      &= \sum_{t=1}^{T-1} h_t^2 \EE[t-1]{\gamma_t^2} +  h_T^2 \EE[T-1]{\gamma_T^2} \\
      &= \sum_{t=1}^{T-1} h_t^2 \EE[t-1]{\gamma_t^2} + \p{1 - \sum_{t=1}^{T-1} h_t^2 \EE[t-1]{\gamma_t^2} } \lambda_T\\
      &= 1.
  \end{aligned}
\end{equation}
Moreover, since the denominator of \eqref{eq:general_variance_convergence} is the unconditional expectation of \eqref{eq:variances_sum_to_one} the result follows.

For the remainder it will be convenient to write $h_t^2 \EE[t-1]{\gamma_t^2}$ in \eqref{eq:gamma} in terms of the allocation rates:
  \begin{equation}
    \label{eq:lambda_representation}
        h_t^2 \EE[t-1]{\gamma_t^2} = \lambda_t p_{t-1} \qquad \text{where} \qquad p_{t-1} := \prod_{s=1}^{t-1}(1-\lambda_s).
  \end{equation}  
  
We can show \eqref{eq:lambda_representation} in two steps. First, let's establish the auxiliary identity $\smash{p_{t} = 1 - \sum_{s=1}^{t-1} \lambda_s p_{s-1}}$. It is evidently true for the first period since then both sides of this equality are one. Supposing that it is true for all periods up to $k-1$, at the $k^{th}$ period we have
\begin{equation}
  \label{eq:lambda_aux}
  \begin{aligned}
    1 - \sum_{s=1}^{k} \lambda_s p_{s-1}
    = 1 - \sum_{s=1}^{k-1} \lambda_s p_{s-1} - \lambda_k p_{k-1} 
    = p_{k-1} - \lambda_k p_{k-1}
    = (1 - \lambda_k) p_{k-1}
    = p_{k},
  \end{aligned}
\end{equation}
where in the second equality we used the inductive hypothesis, and in the last equality we used the definition of $p_{k}$ from \eqref{eq:lambda_representation}. This proves our auxiliary identity. Now let's show \eqref{eq:lambda_representation}. It is true for the first period, since $h_1^2 \EE[0]{\gamma_1^2} = \lambda_1$ by \eqref{eq:gamma}. Supposing that is true for the first $k-1$ periods,
\begin{equation}
  \begin{aligned}
    h_k^2 \EE[k-1]{\gamma_k^2}
    = \lambda_{k} \p{1 - \sum_{s=1}^{k-1} h_s^2 \EE[s-1]{\gamma_s^2}} 
    = \lambda_k \p{1 - \sum_{s=1}^{k-1} \lambda_s p_{s-1} } 
    = \lambda_k p_{k-1},
  \end{aligned}
\end{equation}
where the first inequality is due to \eqref{eq:gamma}, the second inequality was due to the inductive hypothesis, and the third equality was due to the auxiliary identity \eqref{eq:lambda_aux}. This establishes \eqref{eq:lambda_representation}.

  \paragraph{Lyapunov} As we noted in \eqref{eq:variances_sum_to_one}, our variance-stabilizing weights ensure $\smash{\E[ \sum_{t=1}^T h_t^2 \gamma_t^2] = 1}$, so the denominator of the Lyapunov condition \eqref{eq:general_lyapunov} is constant, and we only need to show that its numerator converges to zero. Writing this condition in terms of the allocation rates as in \eqref{eq:lambda_representation},
  \begin{align}
    \label{eq:lyapunov_lambda}
    \sum_{t=1}^T |h_t|^{2+\delta} \EE[t-1]{|\gamma_t|^{2+\delta}} 
      = \sum_{t=1}^T (\lambda_t p_{t-1})^{1 + \delta/2} 
          \frac{\EE[t-1]{|\gamma_t|^{2+\delta}}}{\EE[t-1]{\gamma_t^2}^{1 + \delta/2} }.
  \end{align}

  Let's upper bound the product $p_{t-1}$. It is largest when the allocation rates $\lambda_t$ are smallest. Substituting the lower bound for $\lambda_t$ from \eqref{eq:general_allocation_rate_bounds} and writing out the product explicitly,
  \begin{equation}
    \label{eq:product_upper_bound}
    \begin{aligned}
        p_{t-1}
        &\leq \prod_{s=1}^{t-1}\p{ 1 - \frac{1}{1 + T - s}} \\
        &= \p{ \frac{T - 1}{T} }  \p{ \frac{T - 2}{T - 1}} \cdots  \p{ \frac{T - t + 1}{T - t + 2} }   \\
        &= 1 - \frac{t}{T} + \frac{1}{T}.
    \end{aligned}
  \end{equation}

Next, using \eqref{eq:product_upper_bound} and the upper bound assumed in \eqref{eq:general_allocation_rate_bounds}, the product $\lambda_t p_{t-1}$ is bounded as
\begin{equation}
  \begin{aligned}
    \label{eq:lambda_p}
    \lambda_t p_{t-1}
      &\lesssim \frac{1}{\EE[t-1]{\gamma_t^2}}\frac{ 1 - (t/T) + (1/T) }{t^{-\alpha} + T^{1-\alpha} - t^{1-\alpha}} \\
      &= \frac{1}{\EE[t-1]{\gamma_t^2}}\frac{1}{T^{1-\alpha}} 
          \frac{  1 - (t/T) + (1/T)}
               { 1 - (t/T)^{1-\alpha} + (1/T)(t/T)^{-\alpha}}.
  \end{aligned}  
\end{equation}

In a moment we will show that the last factor in \eqref{eq:lambda_p} is uniformly bounded by a finite constant. If we take this as fact for now, we can upper bound the summand in \eqref{eq:lyapunov_lambda} via \eqref{eq:lambda_p} to get
\begin{equation}
  \label{eq:lambda_p2}
  (\lambda_t p_{t-1})^{1+\delta/2}  
  \frac{\EE[t-1]{|\gamma_t|^{2+\delta}}}{\EE[t-1]{\gamma_t^2}^{1 + \delta/2} }
  \lesssim
      \frac{1}{T^{(1-\alpha)(1+\delta/2)}}
       \frac{\EE[t-1]{|\gamma_t|^{2+\delta}}}
         {\EE[t-1]{\gamma_t^2}^{2+\delta}}.
\end{equation}
By \eqref{eq:variance_growth_rate} the second factor in \eqref{eq:lambda_p2} also bounded. Summing up \eqref{eq:lambda_p2} over $t$,
\begin{equation}
  \label{eq:lyapunov_conclusion}
  \sum_{t=1}^T
  (\lambda_t p_{t-1})^{1+\delta/2}  
  \frac{\EE[t-1]{|\gamma_t|^{2+\delta}}}{\EE[t-1]{\gamma_t^2}^{1 + \delta/2} } 
  \lesssim \frac{T}{T^{(1-\alpha)(1+\delta/2)}},
\end{equation}
which converges to zero as $T \to \infty$ provided that $1 - (1-\alpha)(1+\delta/2) < 0$, or $\alpha < \delta/(2+\delta)$, proving that the Lyapunov condition is satisfied. 

To complete the proof, we prove that the last factor in \eqref{eq:lambda_p} is bounded. This factor is $f_{\alpha}(t/T, 1/T)$ for 
\begin{align}
  f_{\alpha}(x, c) := \frac{1 - x + c}{ 1 - x^{1-\alpha} + cx^{-\alpha}} \quad \text{ defined on the domain } \quad x \in [0, 1), \text{ and } c > 0. 
\end{align}

Let's show that $f_{\alpha}$ is uniformly bounded in its domain. Its partial derivative with respect to $c$ is
\begin{equation}
    \frac{\partial}{\partial c} f_{\alpha}(x, c) 
      = \frac{1 - x^{\alpha}}
             { (1 - x^{1-\alpha} + cx^{-\alpha})^2},
\end{equation}

which is always nonpositive for any $x, c$ in the domain of $f_{\alpha}$, implying that for any $c > 0$
\begin{equation}
  f_{\alpha}(x,c) = 
    \frac{1 - x + c}
         { 1 - x^{1-\alpha} + cx^{-\alpha}} 
    \leq \frac{1 - x}
              {1 - x^{1-\alpha}} 
    = f_{\alpha}(x, 0).
\end{equation}

The function $f_{\alpha}(x, 0) = (1-x)/(1-x^{1-\alpha})$ for $x \in \mathbb{R}$ has derivative equal to
\begin{equation}
  \label{eq:fzero_dx}
  \frac{\partial}{\partial x} f_{\alpha}(x, 0) 
    =  -\frac{x^\alpha ( x^\alpha -1 - \alpha(x-1))}{(x - x^\alpha)^2}.
\end{equation}

The sign of \eqref{eq:fzero_dx} depends only on the function $g(x, \alpha) := x^\alpha - 1 - \alpha (x-1)$, which is the remainder of the second order taylor expansion of $x^{\alpha}$ around $x=1$. By the mean value theorem, $g(x, \alpha) = (1/2)\alpha(\alpha - 1)z^{\alpha-2}(x - 1)^2$ for some $z$ between $x$ and $1$, which is nonpositive for $\alpha < 1$, implying that \eqref{eq:fzero_dx} is nonnegative for $\alpha < 1$. Therefore, $f_{\alpha}(x,0)$ is increasing at $x \in [0,1)$ for $\alpha < 1$. 
It follows that its value on the interval $[0, 1)$ is bounded by its limit as $x$ approaches one. By L'H\^{o}pital's rule, $\lim_{x \to 1}f_{\alpha}(x, 0) = (1-\alpha)^{-1}$, which is finite for any $\alpha < 1$.
 
\paragraph{Infinite sampling} Again recalling the argument in \eqref{eq:variances_sum_to_one}, when variance-stabilizing weights are used the denominator of the infinite sampling condition \eqref{eq:general_infinite_sampling} is constant. 
 In the notation of \eqref{eq:lambda_representation},  we only need to show that its numerator goes to zero, or equivalently,
\begin{equation}
  \begin{aligned}
    \label{eq:h_sum}
    \sum_{t=1}^T h_t 
      &= \sum_{t=1}^T \sqrt{\lambda_t p_{t-1} \EE[t-1]{\gamma_t^2}^{-1} } 
      \xrightarrow[T \to \infty]{p} \infty.
    \end{aligned}
\end{equation}

Let $T_{0} := \min \{ t : \sum_{s=1}^t \lambda_s > 1/2 \}$. Note that since $\lambda_T = 1$, there exists at least one period $t \leq T$ that satisfies this condition. Now consider the constrained minimization problem,
\begin{equation}
  \label{eq:constrained_minimization}
    p_{T_0 - 1} :=
    \min_{\lambda} \left\{ \prod_{s=1}^{T_{0}-1}(1- \lambda_s)
      \quad \text{s.t} \quad  
    \sum_{s=1}^{T_{0}-1} \lambda_s \leq \frac{1}{2} \text{ and } \lambda_s \geq 0 \right\}.
\end{equation}
The objective function in \eqref{eq:constrained_minimization} is log-concave and strictly decreasing in $\lambda_s$, therefore the problem has corner solutions of the form $\lambda_s = 1/2$ for some period $s$ and $\lambda_{s'} = 0$ for all $s' \neq s$. Moreover, note that $p_t$ is decreasing in $t$, so that
\begin{equation}
  \label{eq:product_lower_bound}
    p_{1} \geq \cdots \geq p_{T_{0}-1} \geq \tilde{p} = \frac{1}{2}.
\end{equation}

Now consider two cases. First, if $T_{0} \geq T/2$, then consider the sum only up to $T/2$. In this range, \eqref{eq:h_sum} can be lower bounded by replacing the allocation rate $\lambda_t$ with its lower bound from \eqref{eq:general_allocation_rate_bounds}, replacing
 $p_{t-1}$ by the lower bound $1/2$ from \eqref{eq:product_lower_bound}, 
and replacing \smash{$\EE[t-1]{\gamma_t^2}^{-1}$} with the lower bound $t^{-\alpha}/C'$ implied by \eqref{eq:variance_growth_rate},
\begin{equation}
    \label{eq:first_branch}
    \sum_{t=1}^T h_t 
      \gtrsim \sum_{t=1}^{T/2} \frac{t^{-\alpha/2}}{\sqrt{1 + T - t}}
      \geq \int_{0}^{T/2} \frac{t^{-\alpha/2}}{\sqrt{T}} 
      = T^{\frac{1-\alpha}{2}},
\end{equation}
where in the second inequality we lower bounded the numerator sum by an integral. Since \eqref{eq:first_branch} goes to infinity goes to infinity for $\alpha < 1$, \eqref{eq:h_sum} is satisfied.

Next, consider the case $T_{0} \leq T/2$. In this case we will sum only up to $T_{0}$. Again we replace $p_{t-1}$ by a constant following \eqref{eq:product_lower_bound}, and then use the inverse of the upper bound for $\lambda_t$ in \eqref{eq:general_allocation_rate_bounds},
\begin{equation}
  \begin{aligned}
    \label{eq:second_branch}
    \sum_{t=1}^{T/2} h_t 
      &\gtrsim \sum_{t=1}^{T_{0}} \lambda_t \sqrt{\frac{\EE[t-1]{\gamma_t^2}^{-1}}{\lambda_t}}  \\
      &\geq \sum_{t=1}^{T_{0}} \lambda_t \sqrt{T^{1-\alpha}- t^{1-\alpha}} \\
      &\geq \sqrt{T^{1-\alpha}- (T/2)^{1-\alpha}} \sum_{t=1}^{T_{0}} \lambda_t  \\
      &\gtrsim T^{\frac{1-\alpha}{2}} \sum_{t=1}^{T_{0}} \lambda_t.
  \end{aligned}
\end{equation}

Now, recalling the definition of $T_{0}$, $\sum_{t=1}^{T_{0}} \lambda_t$ is greater than $1/2$, so again the sum \eqref{eq:second_branch} goes to infinity for $\alpha < 1$. 

\end{proof}

\subsection{Asymptotic normality of alternative estimators}
\label{sec:alternative_estimators}

\newcommand{\Qhavg}{\hQ^{h\text{-avg}}_T}
\newcommand{\Vhavg}{\hV^{h\text{-avg}}_T}
\newcommand{\tGamma}{\tilde{\Gamma}}
\newcommand{\tV}{\tilde{V}}

In the Discussion section, we claimed that the alternative estimator $\Qhavg$ defined as
\begin{equation}
  \label{eq:weighted-avg-restated}
  \Qhavg = 
    \frac{\sum_{t=1}^T h_t \frac{ \ind{W_t = w} }{e_t} \p{Y_t - Q}}
          {\sum_{t=1}^T h_t \frac{\ind{W_t = w} }{e_t}},
\end{equation}

when appropriately studentized, has an asymptotically normal distribution under the same conditions as the adaptively-weighted estimator \eqref{eq:aw}. That is, 
\begin{equation}
  \label{eq:weighted-avg-clt}
  \frac{\Qhavg - Q}
        {(\hV_T^{\text{h-avg}})^{1/2}}  
       \xrightarrow[T \to \infty]{d} N(0, 1) 
  \qquad \text{where} \qquad
  \Vhavg :=  
  \frac{\sum_{t=1}^T h_t^2 \frac{ \ind{W_t = w} }{e_t^2} \, \p{Y_t - \Qhavg}^2}
       { \p{ \sum_{t=1}^T h_t \frac{\ind{W_t = w} }{e_t}}^2 }.
\end{equation}

Let's sketch a proof of this claim. We will show that the studentized statistic in \eqref{eq:weighted-avg-clt} asymptotically behaves like the studentized statistic of particular adaptively-weighted estimator with ``oracle'' plugin estimator $\hm_t(w) = Q(w)$:
\begin{align}
  \label{eq:qoracle}
  \tQ_T := 
    \frac{\sum_{t=1}^T h_t (\tGamma_t - Q)}
         {\sum_{t=1}^T h_t}
  \qquad \tGamma_t := Q + \frac{\ind{W_t = w}}{e_t}(Y_t - Q)
\end{align}
Since this choice of plugin is obviously consistent, under the conditions of Theorem \ref{theo:arm_value_clt},
\begin{equation}
  \label{eq:qoracleclt}
  \begin{split}
  \frac{\tQ_T - Q}
       {\tV_T^{1/2}}
       \xrightarrow[T \to \infty]{d} N(0, 1) 
  \qquad \text{where} \qquad
  \tV_T :=  
  \frac{\sum_{t=1}^T h_t^2 \p{\tGamma_t - \tQ_T}^2}
       { \p{ \sum_{t=1}^T h_t }^2 }.
\end{split}
\end{equation}

We begin by showing that $\smash{(\tQ_T - Q)}$ is asymptotically equivalent to $(\smash{\Qhavg - Q})$. Substituting the definition of $\tGamma$ from \eqref{eq:qoracleclt},
\begin{equation}
  \label{eq:numerator-decomp}
  \begin{aligned}
  (\tQ_T - Q) 
    = 
    \frac{\sum_{t=1}^T h_t \frac{\ind{W_t = w}}{e_t}(Y_t - Q)}
         {\sum_{t=1}^T h_t}
    = 
    (\Qhavg - Q)
    \frac{ \sum_{t=1}^T h_t \frac{\ind{W_t = w} }{e_t} }
          {\sum_{t=1}^T h_t },
\end{aligned}
\end{equation}
where in the last equality we multiplied and divided by $\smash{\sum_{t=1}^T h_t \frac{\ind{W_t = w}}{e_t}}$. We can show that the ratio of weights multiplying $(\Qhavg -  Q)$ in the last expression in \eqref{eq:numerator-decomp} converges to one, since its deviations from one can be decompose into two factors,
\begin{equation}
  \label{eq:hterm-decomp}
  \begin{aligned}
    \frac{\sum_{t=1}^T h_t \p{\frac{\ind{W_t = w} }{e_t} - 1}}
         { \sum_{t=1}^T h_t }
    = \frac{\sum_{t=1}^T h_t \p{\frac{\ind{W_t = w}}{e_t} - 1}}
            {\EE{\sum_{t=1}^T h_t^2 \p{\frac{\ind{W_t = w}}{e_t} - 1}^2 }^{1/2}}
      \frac{\EE{\sum_{t=1}^T h_t^2 \p{\frac{\ind{W_t = w}}{e_t} - 1}^2 }^{1/2}}
            {\sum_{t=1}^T h_t}.
\end{aligned}
\end{equation}
The mean square of the first factor is one, so it is $O_p(1)$;
the square of the second factor is less than $(\sum_{t=1}^T \E[h_t^2/e_t])/(\sum_{t=1}^T h_t)^2$, which is $o_p(1)$ 
under our infinite sampling assumption~\eqref{eq:infinite_sampling}. Therefore,  the numerators of the studentized statistics in \eqref{eq:weighted-avg-clt} and \eqref{eq:qoracleclt} are asymptotically equivalent.

Next we will show that the two variances $\Vhavg$ and $\tV_T$ are also asymptotically equivalent. Again substituting the definition of $\tGamma$ from \eqref{eq:qoracle} into the definition of the $\tV_T$ from \eqref{eq:qoracleclt} and expanding the squares in the numerator,
\begin{equation}
  \label{eq:weighted-variance-decomp}
  \begin{aligned}
    \tV_T 
    &= 
      \frac{\sum_{t=1}^T h_t^2\frac{\ind{W_t = w}}{e_t^2} \p{Y_t - Q}^2 }
           {\p{ \sum_{t=1}^T h_t}^2 }
      +
      \p{Q - \tQ_T}^2
      \frac{\sum_{t=1}^T h_t^2}
           {\p{\sum_{t=1}^T h_t}^2}
      +
      2\p{Q - \tQ_T}
      \frac{\sum_{t=1}^T h_t^2\frac{\ind{W_t = w}}{e_t} \p{Y_t - Q} }
           {\p{ \sum_{t=1}^T h_t}^2 }.
  \end{aligned}
\end{equation}

The second and third terms in \eqref{eq:weighted-variance-decomp} are $o_p(1)$. Since  $\tQ_T$ is consistent, we only need to show that the factors multiplying $(Q - \tQ_T)$ in each term are at most bounded in probability. The factor multiplying $\smash{(Q - \tQ_T)^2}$ in the second term is always smaller than one since the weights are nonnegative. 
 The factor multiplying $\smash{(Q - \tQ_T)}$ in the third term can be decomposed as
\begin{equation}
  \label{eq:decomp-tmp}
  \frac{\sum_{t=1}^T h_t^2\frac{\ind{W_t = w}}{e_t} \p{Y_t - Q} }
       { \EE{ \sum_{t=1}^T h_t^2/e_t } }
  \frac{ \EE{ \sum_{t=1}^T h_t^2/e_t } }
       {\p{ \sum_{t=1}^T h_t}^2 }.
\end{equation}

The expectation of the first factor in \eqref{eq:decomp-tmp} is bounded, since by the triangle inequality 
\begin{equation}
  \label{eq:decomp-tmp-bound}
  \frac{ \EE{\sum_{t=1}^T h_t^2/e_t \EE[t-1]{\ind{W_t = w}|Y_t - Q|}} }
        { \EE{ \sum_{t=1}^T h_t^2/e_t } }
  \lesssim
  \frac{\EE{\sum_{t=1}^{T} h_t^2}}
        { \EE{ \sum_{t=1}^T h_t^2/e_t } }
\end{equation}
where in the inequality we used the fact that $\EE[t-1]{\ind{W_t = w}|Y_t - Q|} = \EE{|Y_t(w) - Q|}e_t$ and $Y_t(w)$ has bounded absolute moments. By Markov's inequality this implies that the first factor in \eqref{eq:decomp-tmp} is $O_p(1)$. By the infinite sampling condition \eqref{eq:infinite_sampling} the second factor is $o_p(1)$. Therefore, \eqref{eq:decomp-tmp} is $o_p(1)$ and thus so is the the third factor in \eqref{eq:weighted-variance-decomp}.

The previous paragraph confirmed that only the first term in \eqref{eq:weighted-variance-decomp} is asymptotically non-negligible. Multiplying and dividing this term by $\p{\sum_{t=1}^T h_t \frac{\ind{W_t = w}}{e_t}}^2$ and rearranging,
\begin{equation}
  \label{eq:weighted-variance-decomp2}
  \begin{aligned}
    \tV_T 
     &= \frac{\sum_{t=1}^T h_t^2\frac{\ind{W_t = w}}{e_t^2} \p{Y_t - Q}^2 }
           {\p{ \sum_{t=1}^T h_t \frac{\ind{W_t = w} }{e_t}}^2  }
      \p{ \frac{ \sum_{t=1}^T h_t \frac{\ind{W_t = w} }{e_t}  }
            { \sum_{t=1}^T h_t } }^2  + o_p(1)
   \\
    &=
      \Vhavg\cdot
      \p{ \frac{ \sum_{t=1}^T h_t \frac{\ind{W_t = w} }{e_t}  }
            { \sum_{t=1}^T h_t } }^2
      +
      o_p(1).
  \end{aligned}
\end{equation}

As argued in \eqref{eq:hterm-decomp}, the ratio of weights multiplying $\Vhavg$ in \eqref{eq:weighted-variance-decomp2} converges to one in probability, so that $\tV_T = \Vhavg(1 + o_p(1))$.

The argument above implies that
\begin{equation}
  \frac{\tQ_T - Q}
       {\tV_T^{1/2}}
  =
  \frac{\Qhavg - Q}
        {(\hV_T^{\text{h-avg}})^{1/2}} 
  \p{1 + o_p(1)},
\end{equation}
therefore our central limit theorem claim \eqref{eq:weighted-avg-clt} follows from Slutsky's theorem and \eqref{eq:qoracleclt}.

\subsection{Two-point allocation bounds}
\label{sec:twopoint_allocation_bounds}

Here we show that the two-point allocation $\lambda_t^{\text{twopoint}}$ defined in \eqref{eq:two_point_allocation} indeed lies between the bounds stated in condition \eqref{eq:allocation_rate_bounds} of Theorem \ref{theo:variance_stabilizing_clt}. Recall that $\lambda_t^{\text{twopoint}}$ is a convex combination between two values, one associated with a scenario in which propensity scores will be always high, and one associated with the scenario in which they decay. We'll begin by showing that the two values satisfy the following inequality, 
\begin{equation}
  \label{eq:lambda_inequality}
    \frac{1}{1 + T - t} \leq
    \frac{t^{-\alpha}}{t^{-\alpha} + \frac{T^{1-\alpha} - t^{1-\alpha}}{1 - \alpha}}
     \qquad \text{for all} \quad \alpha, \, t, \text{ and } T.
\end{equation}

This immediately confirms the lower bound stated in \eqref{eq:allocation_rate_bounds}, since it coincides with the smaller value in \eqref{eq:lambda_inequality}. For upper bound, use the fact that $e_t \geq Ct^{-\alpha}$ and that $(1-\alpha)^{-1} \geq 1$ to confirm that the higher value is smaller than the upper bound required by the theorem,
\begin{equation}
  \label{eq:lambda_bad_upper_bound}
    \frac{t^{-a}}{t^{-\alpha} + \frac{T^{1-\alpha} - t^{1-\alpha}}{1 - \alpha}} 
    \leq \frac{e_t/C}{t^{-\alpha} + \frac{T^{1-\alpha} - t^{1-\alpha}}{1 - \alpha}} 
    \lesssim \frac{e_t}{t^{-\alpha} + T^{1-\alpha} - t^{1-\alpha}}.
\end{equation}

All that is left is to establish inequality \eqref{eq:lambda_inequality}. Rearranging, we see that \eqref{eq:lambda_inequality} is true if an only if 
\begin{equation}
  \label{eq:lambda_inequality_rearranged}
    (t/T)^{\alpha} - (1-\alpha) - \alpha(t/T) \leq 0  \qquad \text{for all} \quad \alpha, \, t, \text{ and } T.
\end{equation}

To show that inequality \eqref{eq:lambda_inequality_rearranged} is true, consider the function
\begin{equation*}
  f(x, \alpha)= x^{\alpha} - (1-\alpha) - \alpha x \qquad \text{for } x \in (0, 1] \text{ and } \alpha \in [0, 1).
\end{equation*}
This function is increasing in $x$ for all $\alpha$ on its domain, since $f'_x(x, \alpha) = \alpha x^{\alpha - 1} - \alpha > 0$. Therefore, it attains a maximum at $f(1, \alpha) = 0$ for any $\alpha$. This proves \eqref{eq:lambda_inequality_rearranged} since it equals $f(t/T, \alpha)$, and concludes the proof.

%
%

\section{Simulation details}
\label{sec:simdetails}

\subsection{Modified thompson sampling}
\label{sec:thompson_sampling}

We used the following modified Thompson Sampling algorithm. Set the prior mean and variance of each arm to $m_0(w) = 0$ and $\sigma_0(w) = 1$. Then, for each period $t$,
\begin{enumerate}
  \item Update the posterior probability distribution of the estimate of the mean assuming model with a normal prior and normal likelihood, get $\hat{m}_t(w), \hsigma_t^2(w)$.
  \item Draw $L$ times for each arm $w \in \{1, \cdots, K\}$ from the posterior distribution.
    \begin{align}
      \tilde{y}_{w}^{(\ell)} \sim \mathcal{N}(\hat{m}_t(w), \hsigma_t^2(w))
    \end{align}
  \item Compute ``raw'' Thompson Sampling probabilities.
    \begin{align}
        \bar{e}_t(w) = \frac{1}{L} \sum_{\ell=1}^{L} \ind{ w = \arg\max\{ y_{1}^{(\ell)}, \cdots, y_{K}^{(\ell)} \} }
    \end{align}
  \item Assign a probability floor of $x_t\%$ as follows. If an arm has $\bar{e}_t(w) < x_t$, let $e_t(w) = x_t$. Then, shrink all other arms by letting $e_t(w) = x_t + c(\bar{e}_t(w) - x_t)$, where $c$ is some constant that makes the sum of all assignment probabilities be one.
  \item Draw from Thompson Sampling probabilities with floor for every $t$ in this batch.
    \begin{align}
        W_t \sim \text{Multinomial}\{e_{t}(1), \cdots, e_t(K) \}
    \end{align}
  \item Store the vector of probabilities $e_t(\cdot)$, the selected arm $W_t$ and the observed reward $Y_t$.
\end{enumerate}

\subsection{Non-asymptotic confidence sequences}
\label{sec:nonasymptic_ci}

The confidence sequences around arm values denoted by ``Howard et al CI'' in Section \ref{sec:sims} were constructed via the empirical-Bernstein confidence sequence described in Theorem 4 of \cite{howard2021uniform} with the Gamma-Exponential mixture uniform boundary. In their notation, the chosen parameters are as follows. The scaling parameter is set to $c = 2$, the width of the support of the outcome variables $Y_t(w)$. The variance process is the sum of empirical variances, computed as 
$$\widehat{V}_{T, w} = \sum_{\substack{1 \leq t \leq T \\ t: W_t = w}} (Y_t - \widehat{Y}_{t-1, w})^2,$$

where $\widehat{Y}_{t-1, w}$ is the sample mean using past data, 
$$\widehat{Y}_{t-1, w} = \frac{1}{t-1} \sum_{\substack{1 \leq s \leq t-1 \\ s: W_s = w}} Y_{s}.$$

We optimize the bounds to be tighter for the ``bad" arm. Since in expectation this arm will be pulled $t_{opt} = \lfloor (1/K) \sum_{t=1}^{T} t^{-.7} \rfloor$ times, the optimized intrinsic time is set to $v_{opt, w} = Var(Y_{t}(w)) \cdot t_{opt}$. Both the outcome variance and its support are assumed to be known.

To construct 90\%-confidence intervals for the difference between two arms, we construct two-sided 95\% confidence intervals for each arm and then take a union bound. That is, if the 95\%-intervals for arms $1$ and $3$ are $[\ell_1, u_1]$ and $[\ell_3, u_3]$, then our 90\%-confidence interval for their difference is $[\ell_3 - u_1, u_3 - \ell_1]$.

\fi

\end{document}